\documentclass{article}

\usepackage[final]{neurips_2019}

\usepackage[utf8]{inputenc} 
\usepackage[T1]{fontenc}    
\usepackage{hyperref}       
\usepackage{url}            
\usepackage{booktabs}       
\usepackage{amsfonts}       
\usepackage{nicefrac}       
\usepackage{microtype}      
\usepackage[table, dvipsnames]{xcolor}

\usepackage{amsmath}
\usepackage{amsthm}
\usepackage{amssymb}

\usepackage{algorithm}
\usepackage{algorithmic}

\usepackage{thmtools}
\usepackage{thm-restate}

\usepackage{subfigure}

\def\eqdef{:=}
\def\VAR{\mathrm{Var}}
\def\Regret{\mathrm{Regret}}

\newtheorem{proposition}{Proposition}

\newtheorem{lemma}[proposition]{Lemma}
\newtheorem{theorem}[proposition]{Theorem}

\newtheorem{defn}{Definition}
\newtheorem{remark}{Remark}

\newtheorem{definition}{Definition}
\newenvironment{proofsketch}{%
  \proof}{\endproof}
\usepackage{bbm}
\newcommand{\Olog}{\tilde{\mathcal{O}}}

\usepackage{mathtools}
\DeclarePairedDelimiter\br{(}{)}
\DeclarePairedDelimiter\brs{[}{]}
\DeclarePairedDelimiter\brc{\{}{\}}

\DeclarePairedDelimiter\abs{\lvert}{\rvert}
\DeclarePairedDelimiter\norm{\lVert}{\rVert}

\DeclarePairedDelimiter\innorm{\langle}{\rangle}

\newcommand{\E}{\mathbb{E}}
\newcommand{\R}{\mathbb{R}}

\newcommand{\F}{\mathcal{F}}
\newcommand{\N}{\mathcal{N}}

\newcommand{\ind}{\mathbbm{1}}

\usepackage[colorinlistoftodos]{todonotes}

\makeatletter
\newcommand{\printfnsymbol}[1]{%
  \textsuperscript{\@fnsymbol{#1}}%
}

\title{Tight Regret Bounds for Model-Based Reinforcement Learning with Greedy Policies}

\author{
  Yonathan Efroni\thanks{equal contribution}\\
  Technion, Israel\\
  \And
  Nadav Merlis \printfnsymbol{1}\\
  Technion, Israel\\
  \And
  Mohammad Ghavamzadeh\\
  Facebook AI Research\\
  \And
  Shie Mannor \\
  Technion, Israel\\
  }

\begin{document}

\maketitle

\begin{abstract}
    State-of-the-art efficient model-based Reinforcement Learning (RL) algorithms typically act by iteratively solving empirical models, i.e., by  performing \emph{full-planning} on Markov Decision Processes (MDPs) built by the gathered experience. 
    In this paper, we focus on model-based RL in the finite-state finite-horizon undiscounted MDP setting and establish that exploring with \emph{greedy policies} -- act by \emph{1-step planning} -- can achieve  tight minimax performance in terms of regret,  $\Olog(\sqrt{HSAT})$. Thus, full-planning in model-based RL can be avoided altogether without any performance degradation, and, by doing so, the computational complexity decreases by a factor of $S$. The results are based on a novel analysis of real-time dynamic programming, then extended to model-based RL. Specifically, we generalize existing algorithms that perform full-planning to act by 1-step planning. For these generalizations, we prove regret bounds with the same rate as their full-planning counterparts. 
\end{abstract}

\section{Introduction}

Reinforcement learning (RL) \citep{sutton2018reinforcement} is a field of machine learning that tackles the problem of learning how to act in an {\em unknown} dynamic environment. An agent interacts with the environment, and receives feedback on its actions in the form of a state-dependent reward signal. Using this experience, the agent's goal is then to find a policy that maximizes the long-term reward. 

There are two main approaches for learning such a policy: model-based and model-free. The model-based approach estimates the system's model and uses it to assess the long-term effects of actions via \emph{full-planning} (e.g., \citealt{jaksch2010near}). Model-based RL algorithms usually enjoy good performance guarantees in terms of the regret -- the difference between the sum of rewards gained by playing an optimal policy and the sum of rewards that the agent accumulates  \citep{jaksch2010near,bartlett2009regal}. Nevertheless, model-based algorithms suffer from high space and computation complexity. The former is caused by the need for storing a model. The latter is due to the frequent full-planning, which requires a full solution of the estimated model. Alternatively, model-free RL algorithms directly estimate quantities that take into account the long-term effect of an action, thus, avoiding model estimation and planning operations altogether \citep{jin2018q}. These algorithms usually enjoy better computational and space complexity, but seem to have worse performance guarantees.

In many applications, the high computational complexity of model-based RL makes them infeasible. Thus, practical model-based approaches alleviate this computational burden by using \emph{short-term planning}  e.g., Dyna \citep{sutton1991dyna}, instead of full-planning. 
To the best of our knowledge, there are no regret guarantees for such algorithms, even in the tabular setting. This raises the following  question: \emph{Can a model-based approach coupled with short-term planning enjoy the favorable performance of model-based RL?}

In this work, we show that model-based algorithms that use 1-step planning can achieve the same performance as algorithms that perform full-planning, thus, answering affirmatively to the above question. To this end, we study Real-Time Dynamic-Programming (RTDP) \citep{barto1995learning} that finds the optimal policy of a {\em known} model by acting greedily based on 1-step planning, and establish new and sharper finite sample guarantees. We demonstrate how the new analysis of RTDP can be incorporated into two model-based RL algorithms, and prove that the regret of the resulting algorithms remains unchanged, while their computational complexity drastically decreases. As Table \ref{tab:tabbounds} shows, this reduces the computational complexity of model-based RL methods by a factor of $S$. 

The contributions of our paper are as follows: we first prove regret bounds for RTDP when the model is known. To do so, we establish concentration results on Decreasing Bounded Processes, which are of independent interest. We then show that the regret bound translates into a Uniform Probably Approximately Correct (PAC) \citep{dann2017unifying} bound for RTDP that greatly improves existing PAC results \citep{strehl2006pac}. Next, we move to the learning problem, where the model is unknown. Based on the analysis developed for RTDP we adapt UCRL2 \citep{jaksch2010near} and EULER \citep{zanette2019tighter}, both act by full-planning, to UCRL2 with Greedy Policies (UCRL2-GP) and EULER with Greedy Policies (EULER-GP); model-based algorithms that act by 1-step planning. The adapted versions are shown to preserve the performance guarantees, while improve in terms of computational complexity.

\begin{table*}
\begin{center}
\begin{tabular}{|c| c | c | c| }\hline
 Algorithm & Regret & Time Complexity & Space Complexity \\ \hline \hline
  \texttt{UCRL2}\footnotemark {\tiny\citep{jaksch2010near}} & $\tilde O(\sqrt{H^2S^2AT})$ & $\Olog(\N SAH)$ & $\Olog(HS+\N SA)$ \\ 
  \hline
 \texttt{UCBVI} {\tiny\citep{azar2017minimax}} & $\tilde O(\sqrt{HSAT} + \sqrt{H^2 T} )$& $\Olog(\N SAH)$ & $\Olog(HS+\N SA)$ \\  
 \hline
 \texttt{EULER} {\tiny \citep{zanette2019tighter}} & $\Olog(\sqrt{HSAT})$& $\Olog(\N SAH)$  & $\Olog(HS+\N SA)$   \\
  \hline
 \rowcolor{lightgray}
 \texttt{UCRL2-GP}& $\tilde O(\sqrt{H^2S^2 A T})$ &
$\Olog(\N AH)$ & $\Olog(HS+\N SA)$\\
\hline
\rowcolor{lightgray}
 \texttt{EULER-GP}& $\tilde O(\sqrt{HS A T})$ &
$\Olog(\N AH)$ & $\Olog(HS+\N SA)$\\
\hline
     \texttt{Q-v2}  {\tiny\citep{jin2018q}}& $\tilde O(\sqrt{ H^3 S A T}$  & $\Olog(AH)$ & $\Olog(HSA )$ \\
\hline\hline

 Lower bounds & $\Omega\left(\sqrt{HSAT} \right)$ & -- & --\\\hline
\end{tabular}
\end{center}
\caption{Comparison of our bounds with several state-of-the-art bounds for RL in tabular finite-horizon MDPs. The time complexity of the algorithms is per episode; $S$ and $A$ are the sizes of the state and action sets, respectively; $H$ is the horizon of the MDP; $T$ is the total number of samples that the algorithm gathers; $\N\le S$ is the maximum number of non-zero transition probabilities across the entire state-action pairs. 
The algorithms proposed in this paper are highlighted in gray. 
}
\label{tab:tabbounds}
\end{table*}
\footnotetext{Similarly to previous work in the finite horizon setting, we state the regret in terms of the horizon $H$. The regret in the infinite horizon setting is $DS\sqrt{AT}$, where $D$ is the diameter of the MDP.}

\section{Notations and Definitions}\label{sec: setup}

We consider finite-horizon MDPs with time-independent dynamics \citep{bertsekas1996neuro}. A finite-horizon MDP is defined by the tuple $\mathcal{M} = \br*{\mathcal{S},\mathcal{A}, R, p, H}$, where $\mathcal{S}$ and $\mathcal{A}$ are the state and action spaces with cardinalities $S$ and $A$, respectively. The immediate reward for taking an action $a$ at state $s$ is a random variable $R(s,a)\in\brs*{0,1}$ with expectation $\E R(s,a)=r(s,a)$. The transition probability is $p(s'\mid s,a)$, the probability of transitioning to state $s'$ upon taking action $a$ at state $s$. 
Furthermore, $\N\eqdef \max_{s,a}|\brc*{s':p(s'\mid s,a)>0}|$ is the maximum number of non-zero transition probabilities across the entire state-action pairs. If this number is unknown to the designer of the algorithm in advanced, then we set $\N=S$. 
The initial state in each episode is arbitrarily chosen and $H\in \mathbb{N}$ is the {\em horizon}, i.e.,~the number of time-steps in each episode. We define $[N] \eqdef \brc*{1,\ldots,N},\;$ for all $N\in \mathbb{N}$, and throughout the paper use $t\in\brs*{H}$ and $k\in\brs*{K}$ to denote time-step inside an episode and the index of an episode, respectively.

A deterministic policy $\pi: \mathcal{S}\times[H]\rightarrow \mathcal{A}$ is a mapping from states and time-step indices to actions. We denote by $a_t \eqdef \pi(s_t,t)$, the action taken at time $t$ at state $s_t$ according to a policy $\pi$. The quality of a policy $\pi$ from state $s$ at time $t$ is measured by its value function, which is defined as
\begin{align*}
    V_t^\pi(s) \eqdef \E\brs*{\sum_{t'=t}^H r\br*{s_{t'},\pi(s_{t'},t')}\mid s_t=s},
\end{align*}
where the expectation is over the environment's randomness. An optimal policy maximizes this value for all states $s$ and time-steps $t$, and the corresponding optimal value is denoted by $V_t^*(s) \eqdef \max_{\pi} V_t^\pi(s),\;$ for all $t\in [H]$. The optimal value satisfies the optimal Bellman equation, i.e.,
\begin{align}
\label{eq:bellman}
    V_t^*(s) = T^*V_{t+1}^*(s) \eqdef \max_{a}\brc*{r(s,a) + p(\cdot\mid s,a)^TV_{t+1}^*}.
\end{align}
We consider an agent that repeatedly interacts with an MDP in a sequence of episodes $[K]$. The performance of the agent is measured by its \textit{regret}, defined as $\Regret(K)\eqdef \sum_{k=1}^K \br*{V_1^*(s_1^k) - V_1^{\pi_k}(s_1^k)}$. Throughout this work, the policy $\pi_k$ is computed by a 1-step planning operation with respect to the value function estimated by the algorithm at the end of episode $k-1$, denoted by~$\bar{V}^{k-1}$. We also call such policy a \emph{greedy policy}. Moreover, $s_t^k$ and $a_t^k$ stand, respectively, for the state and the action taken at the $t^{th}$ time-step of the $k^{th}$ episode.

Next, we define the filtration $\F_k$ that includes all events (states, actions, and rewards) until the end of the $k^{th}$ episode, as well as the initial state of the episode $k+1$. We denote by $T=KH$, the total number of time-steps (samples). Moreover, we denote by $n_{k}(s,a)$, the number of times that the agent has visited state-action pair $(s,a)$, and by $\hat{X}_k$, the empirical average of a random variable $X$. Both quantities are based on experience gathered until the end of the $k^{th}$ episode and are $\F_k$ measurable. We also define the probability to visit the state-action pair $(s,a)$ at the $k^{th}$ episode at time-step $t$ by $w_{tk}(s,a)=\Pr\br*{s_t^k=s,a_t^k=a \mid s_0^k,\pi_k}$. We note that $\pi_k$ is $\F_{k-1}$ measurable, and thus, $w_{tk}(s,a)=\Pr\br*{s_t^k=s,a_t^k=a \mid \F_{k-1}}$. Also denote $w_k(s,a)=\sum_{t=1}^H w_{tk}(s,a)$. 

We use $\Olog(X)$ to refer to a quantity that depends on $X$ up to poly-log expression of a quantity at most polynomial in $S$, $A$, $T$, $K$, $H$, and $\frac{1}{\delta}$. Similarly, $\lesssim$ represents $\leq$ up to numerical constants or poly-log factors. We define $\norm{X}_{2,p}\eqdef \sqrt{\E_p X^2}$, where $p$ is a probability distribution over the domain of $X$, and use $X \vee Y \eqdef \max\brc*{X,Y}$. Lastly, $\mathcal{P}(\mathcal{S})$ is the set of probability distributions over the state space $\mathcal S$. 

\section{Real-Time Dynamic Programming}\label{sec: RTDP}

\begin{algorithm}
\begin{algorithmic}
\caption{Real-Time Dynamic Programming}
\label{algo: RTDP}
    \STATE Initialize: $\forall s\in \mathcal S,\; \forall t\in [H],\; \bar{V}^0_t(s)=H-(t-1)$.
    \FOR{$k=1,2,\ldots$}
        \STATE Initialize $s^k_1$
        \FOR{$t=1,\ldots,H$}
            \STATE $a_t^k\in \arg\max_a  r(s_t^k,a)+ p(\cdot \mid s_t^k,a)^T \bar{V}^{k-1}_{t+1}$
            \STATE $\bar{V}^{k}_{t}(s_t^k) =  r(s_t^k,a_t^k)+ p(\cdot\mid s_t^k,a_t^k)^T \bar{V}^{k-1}_{t+1}$
            \STATE Act with $a_t^k$ and observe $s_{t+1}^k$.
        \ENDFOR
    \ENDFOR
\end{algorithmic}
\end{algorithm}

 RTDP \citep{barto1995learning} is a well-known algorithm that solves an MDP when a model of the environment is given. Unlike, e.g., Value Iteration (VI) \citep{bertsekas1996neuro} that solves an MDP by offline calculations, RTDP solves an MDP in a real-time manner. As mentioned in \citet{barto1995learning}, RTDP can be interpreted as an asynchronous VI adjusted to a real-time algorithm. 
 
Algorithm~\ref{algo: RTDP} contains the pseudocode of RTDP for finite-horizon MDPs. The value function is initialized with an optimistic value, i.e., an upper bound of the optimal value. At each time-step $t$ and episode $k$, the agent acts from the current state $s_t^k$ greedily with respect to the current value at the next time step, $\bar{V}^{k-1}_{t+1}$. It then updates the value of $s_t^k$ according to the optimal Bellman operator. We denote by $\bar{V}$, the value function, and as we show in the following, it always upper bounds $V^*$. Note that since the action at a fixed state is chosen according to $\bar{V}^{k-1}$, then $\pi_k$ is $\F_{k-1}$ measurable.
 
 Since RTDP is an online algorithm, i.e., it updates its value estimates through interactions with the environment, it is natural to measure its performance in terms of the regret. The rest of this section is devoted to supplying expected and high-probability bounds on the regret of RTDP, which will also lead to PAC bounds for this algorithm. In Section \ref{sec: model base RL with 1-step greedy}, based on the observations from this section, we will establish minimax regret bounds for 1-step greedy model-based RL.

We start by stating two basic properties of RTDP in the following lemma: the value is always optimistic and decreases in $k$ (see proof in Appendix \ref{sec:supp proof RTDP}). Although the first property is known \citep{barto1995learning}, to the best of our knowledge, the second one has not been proven in previous work.

\begin{restatable}{lemma}{rtdpProperties}
\label{lemma:rtdp properties}
For all $s$, $t$, and $k$, it holds that (i) $V^*_t(s) \leq \bar{V}^k_t(s)$ and (ii) $\bar{V}^{k}_t(s) \leq \bar{V}^{k-1}_t(s)$.
\end{restatable}

The following lemma, that we believe is new, relates the difference between the optimistic value $\bar{V}_1^{k-1}(s_1^k)$ and the real value $V_1^{\pi_k}(s_1^k)$ to the \emph{expected cumulative update}  of the value function at the end of the $k^{th}$ episode (see proof in Appendix \ref{sec:supp proof RTDP}). 

\begin{restatable}[Value Update for Exact Model]{lemma}{rdtpExpectedValueUpdate}
\label{lemma: RTDP expected value difference}
The expected cumulative value update of RTDP at the $k^{th}$ episode satisfies
\begin{align*}
    \bar{V}_1^{k-1}(s^k_1)-V_1^{\pi_k}(s^k_1) = \sum_{t=1}^{H} \E[ \bar{V}_t^{k-1}(s_t^k) - \bar{V}_t^{k}(s_t^k)\mid \F_{k-1}].
\end{align*}
\end{restatable}

The result relates the difference of the optimistic value $\bar{V}^{k-1}$ and the value of the greedy policy $V^{\pi_k}$ to the expected update along the trajectory, created by following~$\pi_k$. Thus, for example, if the optimistic value is overestimated, then the value update throughout this episode is expected to be large.

\subsection{Regret and PAC Analysis}

Using Lemma~\ref{lemma:rtdp properties}, we observe that the sequence of values is decreasing and bounded from below. Thus, intuitively, the decrements of the values cannot be indefinitely large. Importantly, Lemma~\ref{lemma: RTDP expected value difference} states that when the expected decrements of the values are small, then  $V_1^{\pi_k}(s_1^k)$ is close to $\bar{V}^{k-1}(s_1^k)$, and thus, to $V^*$, since $\bar{V}^{k-1}(s_1^k)\ge \bar{V}^*(s_1^k) \ge V_1^{\pi_k}(s_1^k)$. 

Building on this reasoning, we are led to establish a general result on a decreasing process. This result will allow us to formally justify the aforementioned reasoning and derive regret bounds for RTDP. 
The proof utilizes self-normalized concentration bounds \citep{pena2007pseudo}, applied on martingales,  and can be found in Appendix \ref{supp: proofs on DBP}.

\begin{restatable}[Decreasing Bounded Process]{defn}{DefnDecreasingProcess}\label{defn: DBP}
We call a random process $\brc*{X_k,\F_k}_{k\geq 0}$, where $\brc*{\F_k}_{k\geq 0}$ is a filtration and $\brc*{X_k}_{k\geq 0}$ is adapted to this filtration, a Decreasing Bounded Process, if it satisfies the following properties:
\begin{enumerate}
\item $\brc*{X_k}_{k\geq 0}$ decreases, i.e., $X_{k+1}\leq X_k$ a.s. .
\item $X_0=C\geq 0,$ and for all $k,\ X_k\geq 0$ a.s. .
\end{enumerate}
\end{restatable}

\begin{restatable}[Regret Bound of a Decreasing Bounded Process]{theorem}{TheoremDecreasingProcess}
\label{theorem: regret of decreasing process}
Let $\brc*{X_k,\F_k}_{k\geq 0}$ be a Decreasing Bounded Process 
and $R_K = \sum_{k=1}^K X_{k-1} - \E[X_{k}\mid \F_{k-1}]$ be its $K$-round regret. Then,
\begin{align*}
    \Pr\brc*{\exists K>0: R_K \ge C\br*{1+2\sqrt{\ln\br*{2/\delta}}}^2} \le \delta.
\end{align*}
Specifically, it holds that $\Pr\brc*{\exists K>0: R_K \geq 9C\ln(3/\delta)} \le \delta.$
\end{restatable}

We are now ready to prove the central result of this section, the expected and high-probability regret bounds on RTDP (see full proof in Appendix \ref{sec:supp proof RTDP}).

\begin{restatable}[Regret Bounds for RTDP]{theorem}{TheoremRegretRTDP}
\label{theorem: regret rtdp}
The following regret bounds hold for RTDP:
\begin{enumerate}
    \item $\E[\Regret(K)] \leq SH^2.$
    \item For any $\delta>0$, with probability $1-\delta$, for all $K>0$, $\Regret(K) \leq 9 SH^2\ln(3SH/\delta).$
\end{enumerate}
\end{restatable}
\begin{proofsketch}
We give a sketch of the proof of the second claim. 
Applying Lemmas~\ref{lemma:rtdp properties} and then~\ref{lemma: RTDP expected value difference},
\begin{align}
     \Regret(K)& \eqdef \sum_{k=1}^K V^*_1(s^k_1)- V^{\pi_k}_1(s^k_1)
      \leq  \sum_{k=1}^K \bar{V}_1^{k-1}(s^k_1)- V^{\pi_k}(s^k_1)\nonumber \\
     &\leq \sum_{k=1}^K\sum_{t=1}^{H} \E[ \bar{V}_t^{k-1}(s_t^k) - \bar{V}_t^{k}(s_t^k)\mid \F_{k-1}]. \label{eq: rtdp general bound}
\end{align}

We then establish (see Lemma \ref{lemma: regret to SH decreasing processes}) that RHS of~\eqref{eq: rtdp general bound} is, in fact, a sum of $SH$ Decreasing Bounded Processes, i.e.,
\begin{align}
     \eqref{eq: rtdp general bound} & =\sum_{t=1}^{H}\;\sum_{s\in\mathcal S} \;\sum_{k=1}^K \bar{V}_t^{k-1}(s)- \E[\bar{V}_t^{k}(s)  \mid \F_{k-1}].\label{eq:bound-temp0}
\end{align}

Since for any fixed $s,t$, $\brc*{\bar{V}_t^{k}(s)}_{k\geq 0}$ is a decreasing process by Lemma \ref{lemma:rtdp properties}, we can use Theorem~\ref{theorem: regret of decreasing process}, for a fixed $s,t$, and conclude the proof by applying the union bound on all $SH$ terms in~\eqref{eq:bound-temp0}. 
\end{proofsketch}

Theorem \ref{theorem: regret rtdp} exhibits a regret bound that does not depend on $T=KH$. While it is expected that RTDP, that has access to the exact model, would achieve better performance than an RL algorithm with no such access, a regret bound independent of $T$ is a noteworthy result. Indeed, it leads to the following Uniform PAC (see \citealt{dann2017unifying} for the definition) and $(0,\delta)$ PAC guarantees for RTDP (see proofs in Appendix \ref{sec:supp proof RTDP}). To the best of our knowledge, both are the first PAC guarantees for RTDP.\footnote{Existing PAC results on RTDP analyze variations of RTDP in which $\epsilon$ is an input parameter of the algorithm.}

\begin{restatable}[RTDP is Uniform PAC]{corollary}{CorPACrtdp}
\label{corollary: pac RTDP}
Let $\delta>0$ and $N_\epsilon$ be the number of episodes in which RTDP outputs a policy with ${V^*_1(s_1^k)-V^{\pi_k}_1(s_1^k)>\epsilon}$. Then, 
\begin{align*}
    \Pr\brc*{\exists \epsilon>0: N_\epsilon \geq \frac{9SH^2\ln(3SH/\delta)}{\epsilon}}\leq \delta.
\end{align*}
\end{restatable}

\begin{restatable}[RTDP is $(0,\delta)$ PAC]{corollary}{CorPACrtdpTotalEpisodes}
\label{corollary: pac RTDP finite sample}
Let $\delta>0$ and $N$ be the number of episodes in which RTDP outputs a non optimal policy. Define the (unknown) gap of the MDP, $\Delta(\mathcal{M})= \min_s \min_{\pi: V^{\pi}_1(s)\neq V_1^*(s)} V^*_1(s) - V^{\pi}_1(s) > 0.$ Then,
 \begin{align*}
    \Pr\brc*{N \geq \frac{9SH^2\ln(3SH/\delta)}{\Delta(\mathcal{M})}}\leq \delta.
\end{align*}
\end{restatable}

\section{Exploration in Model-based RL: Greedy Policy Achieves Minimax Regret}\label{sec: model base RL with 1-step greedy}

We start this section by formulating a general optimistic RL scheme that acts by 1-step planning (see Algorithm \ref{alg: general model based RL algorithm}).
Then, we establish Lemma \ref{lemma: model base RL expected value difference}, which generalizes Lemma \ref{lemma: RTDP expected value difference} to the case where a non-exact model is used for the value updates. 
Using this lemma, we offer a novel regret decomposition for algorithms which follow Algorithm \ref{alg: general model based RL algorithm}.  Based on the  decomposition, we analyze generalizations of UCRL2 \citep{jaksch2010near} (for finite horizon MDPs) and EULER \citep{zanette2019tighter}, that use greedy policies instead of solving an MDP (full planning) at the beginning of each episode. Surprisingly, we find that both generalized algorithms do not suffer from performance degradation, up to numerical constants and logarithmic factors. Thus, we conclude that there exists an RL algorithm that achieves the minimax regret bound, while acting according to greedy policies. 

\begin{algorithm}
\begin{algorithmic}[1]
\caption{Model-based RL with Greedy Policies}
\label{alg: general model based RL algorithm}
    \STATE Initialize: $\forall s\in \mathcal{S}, \; \forall t\in [H],\ \bar{V}^0_t(s)=H-(t-1)$.
    \FOR{$k=1,2,\ldots$}
        \STATE Initialize $s_1^k$
        \FOR{$t=1,\ldots,H$}
            \STATE $\forall a,\ \bar{Q}(s_t^k,a) = \mathrm{ModelBasedOptimisticQ}\br*{\hat{r}_{k-1},\hat{p}_{k-1},n_{k-1},\bar{V}_{t+1}^{k-1}}$
            \STATE $a_t^k\in \arg\max_a \bar{Q}(s_t^k,a)$
            \STATE $\bar{V}^{k}_{t}(s_t^k) =\min\brc*{\bar{V}^{k-1}_{t}(s_t^k),\bar{Q}(s_t^k,a_t^k)}$ \label{alg line: v^k does not decrease}
            \STATE Act with $a_t^k$ and observe $s_{t+1}^k$.
        \ENDFOR
            \STATE Update $\hat{r}_k,\hat{p}_k,n_{k}$ with all experience gathered in episode.
    \ENDFOR
\end{algorithmic}
\end{algorithm}

Consider the general RL scheme that explores by greedy policies as depicted in Algorithm \ref{alg: general model based RL algorithm}. The value $\bar{V}$ is initialized optimistically and the algorithm interacts with the unknown environment in an episodic manner. At each time-step $t$, a greedy policy from the current state, $s_t^k$, is calculated optimistically based on the empirical model $(\hat{r}_{k-1},\hat{p}_{k-1},n_{k-1})$ and the current value at the next time-step $\bar{V}_{t+1}^{k-1}$. 
This is done in a subroutine called `ModelBasedOptimisticQ'.\footnote{We also allow the subroutine to use $\mathcal{O}(S)$ internal memory for auxiliary calculations, which does not change the overall space complexity.} 
We further assume the optimistic $Q$-function has the form $\bar{Q}(s_t^k,a)= \tilde{r}_{k-1}(s_t^k,a) +\tilde{p}_{k-1}(\cdot\mid s_t^k,a)^T\bar{V}_{t+1}^{k-1}$ and refer to  $(\tilde{r}_{k-1},\tilde{p}_{k-1})$ as the optimistic model. The agent interacts with the environment based on the greedy policy with respect to $\bar{Q}$ and uses the gathered experience to update the empirical model at the end of the episode.

By construction of the update rule (see Line \ref{alg line: v^k does not decrease}), the value is a decreasing function of $k$, for all $(s,t)\in\mathcal{S}\times [H]$. Thus, property (ii) in Lemma~\ref{lemma:rtdp properties} holds for Algorithm \ref{alg: general model based RL algorithm}. Furthermore, the  algorithms analyzed in this section will also be optimistic with high probability, i.e., property (i) in Lemma~\ref{lemma:rtdp properties} also holds. Finally, since the value update uses the empirical quantities $\hat{r}_{k-1}$, $\hat{p}_{k-1}$, $n_{k-1}$ and $\bar{V}_{t+1}^{k-1}$ from the previous episode, policy $\pi_k$ is still $\F_{k-1}$ measurable.

The following lemma generalizes Lemma \ref{lemma: RTDP expected value difference} to the case where, unlike in RTDP, the update rule does not use the exact model (see proof in Appendix \ref{supp proofs for general greedy policy RL}).

\begin{restatable}[Value Update for Optimistic Model]{lemma}{LemmaModelBaseDecomposition}
\label{lemma: model base RL expected value difference} 
The expected cumulative value update of Algorithm \ref{alg: general model based RL algorithm} in the $k^{th}$ episode is bounded by
\begin{align*}
    \bar{V}_1^{k-1}(s_1^k) &- V_1^{\pi_k}(s_1^k) \leq  \sum_{t=1}^{H}\E\brs*{ \bar{V}_t^{k-1}(s_t^k) - \bar{V}_t^{k}(s_t^k)\mid \F_{k-1}}\\
    & + \sum_{t=1}^{H}\E \brs*{(\tilde{r}_{k-1} - r)(s_t^k,a_t^k)+(\tilde{p}_{k-1} - p)(\cdot\mid s^k_t,a_t^k)^T \bar{V}_{t+1}^{k-1} \mid \F_{k-1}  }\enspace.
\end{align*}
\end{restatable}

In the rest of the section, we consider two instantiations of the subroutine `ModelBasedOptimisticQ' in Algorithm~\ref{alg: general model based RL algorithm}. We use the bonus terms of UCRL2 and of EULER to acquire an optimistic $Q$-function,~$\bar{Q}$. These two options then lead to UCRL2 with Greedy Policies (UCRL2-GP) and EULER with Greedy Policies (EULER-GP) algorithms.

\subsection{UCRL2 with Greedy Policies for Finite-Horizon MDPs}\label{sec: UCRL2 with greedy policies}

 \begin{algorithm}
    \caption{UCRL2 with Greedy Policies (UCRL2-GP)}
    \label{alg:UCRL2 Optimistic Model}
	\begin{algorithmic}[1]
        \STATE $\tilde{r}_{k-1}(s_t^k,a) = \hat{r}_{k-1}(s_t^k,a) + \sqrt{\frac{2\ln \frac{8SAT}{\delta}}{n_{k-1}(s_t^k,a)\vee 1}}$
        \STATE  $CI(s_t^k,a) = \brc*{P'\in \mathcal{P}(\mathcal{S}): \norm{P'(\cdot)-\hat{p}_{k-1}(\cdot\mid s_t^k,a)}_1\leq \sqrt{\frac{4S\ln\frac{12SAT}{\delta}}{n_{k-1}(s_t^k,a)\vee 1}}}$
        \STATE $\tilde{p}_{k-1}(\cdot\mid s_t^k,a) = \arg\max_{P'\in CI(s_t^k,a)} P'(\cdot\mid s_t^k,a )^T \bar{V}^{k-1}_{t+1}$ \label{eq alg: UCRL optimistic model}
        \STATE $\bar{Q}(s_t^k,a) = \tilde{r}_{k-1}(s_t^k,a)+\tilde{p}_{k-1}(\cdot\mid s_t^k,a)^T\bar{V}^{k-1}_{t+1}$
        \STATE {\bf Return} $\bar{Q}(s_t^k,a)$
    \end{algorithmic}
  \end{algorithm}
  
We form the optimistic local model based on the confidence set of UCRL2~\citep{jaksch2010near}. 
This amounts to use Algorithm~\ref{alg:UCRL2 Optimistic Model} as the subroutine `ModelBasedOptimisticQ' in Algorithm~\ref{alg: general model based RL algorithm}. The maximization problem on Line~\ref{eq alg: UCRL optimistic model} of Algorithm~\ref{alg:UCRL2 Optimistic Model} is common, when using bonus based on an optimistic model \citep{jaksch2010near}, and it can be solved efficiently in $\Olog(\mathcal{N})$ operations  
(e.g., \citealt{strehl2008analysis}, Section 3.1.5). A full version of the algorithm can be found in Appendix \ref{supp: UCRL2}.

Thus, Algorithm~\ref{alg:UCRL2 Optimistic Model} performs $\mathcal{N}AH$ operations per episode. This saves the need to perform Extended Value Iteration~\citep{jaksch2010near},  that costs $\mathcal{N}SAH$ operations per episode (an extra factor of $S$). Despite the significant improvement in terms of computational complexity, the regret of UCRL2-GP is similar to the one of UCRL2 \citep{jaksch2010near} as the following theorem formalizes (see proof in Appendix \ref{supp: UCRL2}).

\begin{restatable}[Regret Bound of UCRL2-GP]{theorem}{TheoremGreedyPoliciesUCRL}\label{theorem: UCRL2 with greedy policies}
    For any time $T\leq KH$, with probability at least $1-\delta$, the regret of UCRL2-GP is bounded by $\Olog\br*{HS\sqrt{AT} +H^2\sqrt{S}SA}$.
\end{restatable}
\begin{proofsketch}
Using the optimism of the value function (see Section \ref{supp: optimism UCRL2}) and by applying Lemma \ref{lemma: model base RL expected value difference}, we bound the regret as follows:
\begin{align}
    \mathrm{Regret}(K) &= \sum_{k=1}^K V_1^*(s_1^k) - V_1^{\pi_k}(s_1^k) \leq \sum_{k=1}^K \bar{V}_1^{k-1}(s_1^k) - V_1^{\pi_k}(s_1^k)\nonumber \\
    &\leq \sum_{k=1}^K\sum_{t=1}^{H} \E[ \bar{V}_t^{k-1}(s_t^k) - \bar{V}_t^{k}(s_t^k)\mid \F_{k-1}] \nonumber\\
    &\quad+\sum_{k=1}^K\sum_{t=1}^{H} \E \left[ (\tilde{r}_{k-1} - r)(s_t^k,a_t^k)+(\tilde{p}_{k-1} - p)(\cdot\mid s^k_t,a_t^k)^T \bar{V}_{t+1}^{k-1} \mid \F_{k-1}  \right]. \label{eq: regret decomposition}
\end{align}

Thus, the regret is upper bounded by two terms. As in Theorem~\ref{theorem: regret rtdp}, by applying Lemma \ref{lemma: sum of decreasing processes} (Appendix \ref{supp: proofs on DBP}), the first term in \eqref{eq: regret decomposition} is a sum of $SH$ Decreasing Bounded Processes, and can thus be bounded by $\Olog\br*{SH^2}$. The presence of the second term in \eqref{eq: regret decomposition} is common in recent regret analyses (e.g.,~\citealt{dann2017unifying}). Using standard techniques \citep{jaksch2010near,dann2017unifying,zanette2019tighter}, this term can be bounded (up to additive constant factors) with high probability by ${\lesssim H\sqrt{S}\sum_{k=1}^K\sum_{t=1}^{H} \E \left[ \sqrt{\frac{1}{n_{k-1}(s_t^k,a_t^k)}} \mid \F_{k-1}  \right]\leq \Olog(HS\sqrt{AT})}$.
\end{proofsketch}

\subsection{EULER with Greedy Policies}\label{sec: EULER with greedy policies}

In this section, we use bonus terms as in EULER ~\citep{zanette2019tighter}. Similar to the previous section, this amounts to replacing the subroutine `ModelBasedOptimisticQ' in Algorithm~\ref{alg: general model based RL algorithm} with a subroutine based on the bonus terms from~\citep{zanette2019tighter}. Algorithm \ref{alg supp: EULER with greedy policies} in Appendix \ref{supp: EULER} contains the pseudocode of the algorithm. The bonus terms in EULER are based on the empirical Bernstein inequality and tracking both an upper bound $\bar{V}_t$ and a lower-bound $\underline{V}_t$ on $V^*_t$.  Using these, EULER achieves both  minimax optimal and problem dependent regret bounds. 

EULER \citep{zanette2019tighter} performs $\mathcal{O}(\N S A H)$ computations per episode (same as the VI algorithm), while EULER-GP requires only $\mathcal{O}(\N A H)$. Despite this advantage in computational complexity, EULER-GP exhibits similar minimax regret bounds to EULER (see proof in Appendix~\ref{supp: EULER}), much like the equivalent performance of UCRL2 and UCRL2-GP proved in Section~\ref{sec: UCRL2 with greedy policies}.
  
\begin{restatable}[Regret Bound of EULER-GP]{theorem}{TheoremGreedyPoliciesEULER}\label{theorem: EULER with greedy policies}
    Let $\mathcal{G}$ be an upper bound on the total reward collected within an episode. Define ${\mathbb{Q}^*\eqdef \max_{s,a,t} \br*{\VAR{R(s,a)+\VAR_{s'\sim p(\cdot\mid s,a)}V^*_{t+1}(s)}}}$ and ${H_{\mathrm{eff}}\eqdef \min \brc*{\mathbb{Q}^*, \mathcal{G}^2/H}}$.
    With probability $1-\delta$, for any time $T\leq KH$ jointly on all episodes $k\in[K]$, the regret of EULER-GP is bounded  by $\Olog\br*{\sqrt{H_{\mathrm{eff}}SAT} + \sqrt{S}SAH^2(\sqrt{S}+\sqrt{H})}.$ Thus, it is also bounded by $\Olog\br*{\sqrt{HSAT} + \sqrt{S}SAH^2(\sqrt{S}+\sqrt{H})}$. 
\end{restatable}

Note that Theorem \ref{theorem: EULER with greedy policies} exhibits similar problem-dependent regret-bounds as in Theorem~1 of~\citep{zanette2019tighter}. Thus, the same corollaries derived in \citep{zanette2019tighter} for EULER can also be applied to EULER-GP.

\section{Experiments}
In this section, we present an empirical evaluation of both UCRL2 and EULER, and compare their performance to the proposed variants, which use greedy policy updates, UCRL2-GP and EULER-GP, respectively. We evaluated the algorithms on two environments. (i) \textbf{Chain environment} \citep{osband2017posterior}: In this MDP, there are $N$ states, which are connected in a chain. The agent starts at the left side of the chain and can move either to the left or try moving to the right, which succeeds w.p. $1-1/N$, and results with movement to the left otherwise. The agent goal is to reach the right side of the chain and try moving to the right, which results with a reward $r\sim\mathcal{N}(1,1)$. Moving backwards from the initials state also results with $r\sim\mathcal{N}(0,1)$, and otherwise, the reward is $r=0$. Furthermore, the horizon is set to $H=N$, so that the agent must always move to the right to have a chance to receive a reward. (ii) \textbf{2D chain}: A generalization of the chain environment, in which the agent starts at the upper-left corner of a $N\times N$ grid and aims to reach the lower-right corner and move towards this corner, in $H=2N-1$ steps. Similarly to the chain environment, there is a probability $1/H$ to move backwards (up or left), and the agent must always move toward the corner to observe a reward $r\sim\mathcal{N}(1,1)$. Moving into the starting corner results with $r\sim\mathcal{N}(0,1)$, and otherwise $r=0$. This environment is more challenging for greedy updates, since there are many possible trajectories that lead to reward.

\begin{figure}
\subfigure[Chain environment with $N=25$ states]{
\includegraphics[width=0.45\linewidth]{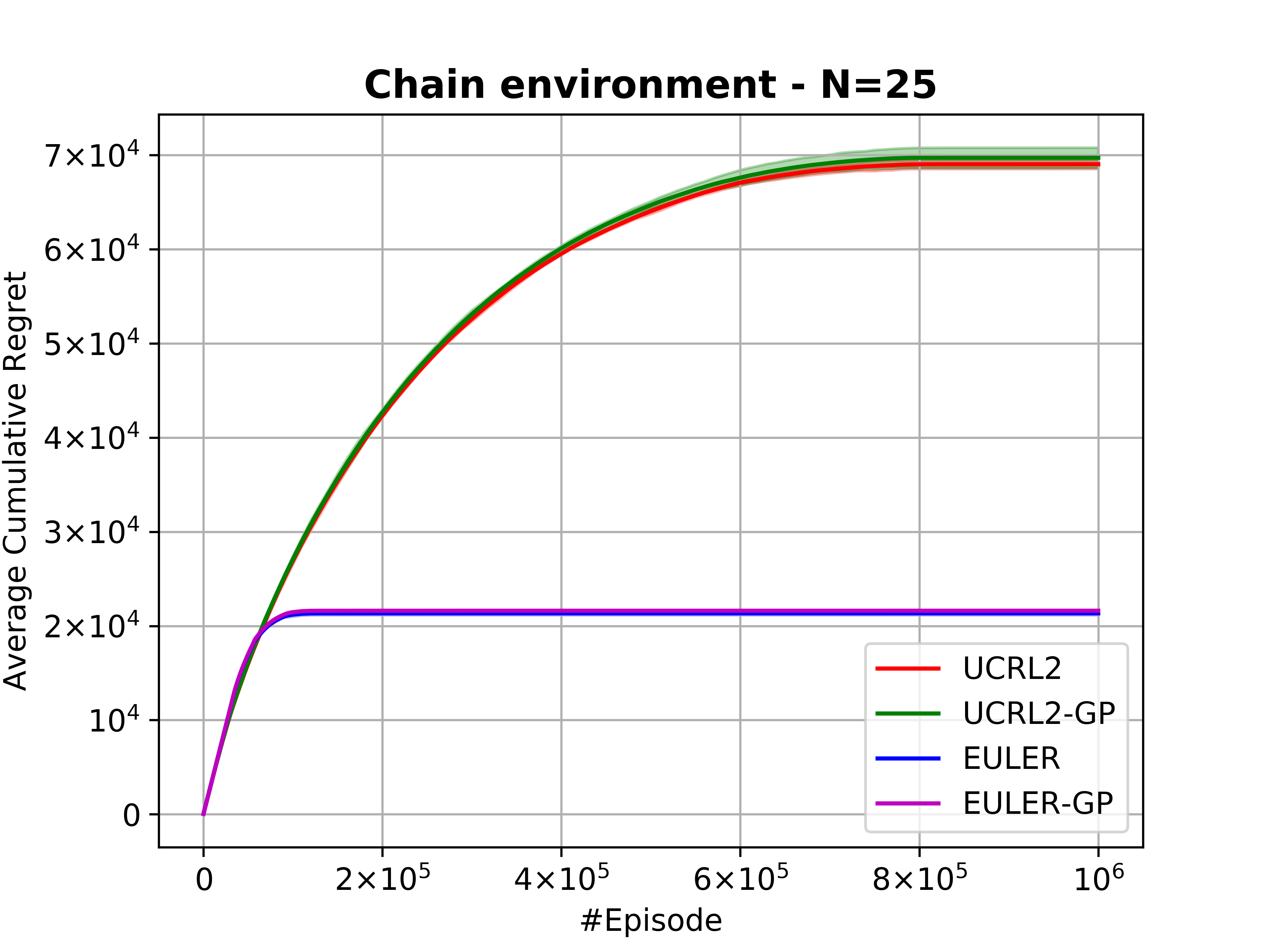}
\label{subfig:chain25}}
\subfigure[2D chain environment with $5\times5$ grid]{
\includegraphics[width=0.45\linewidth]{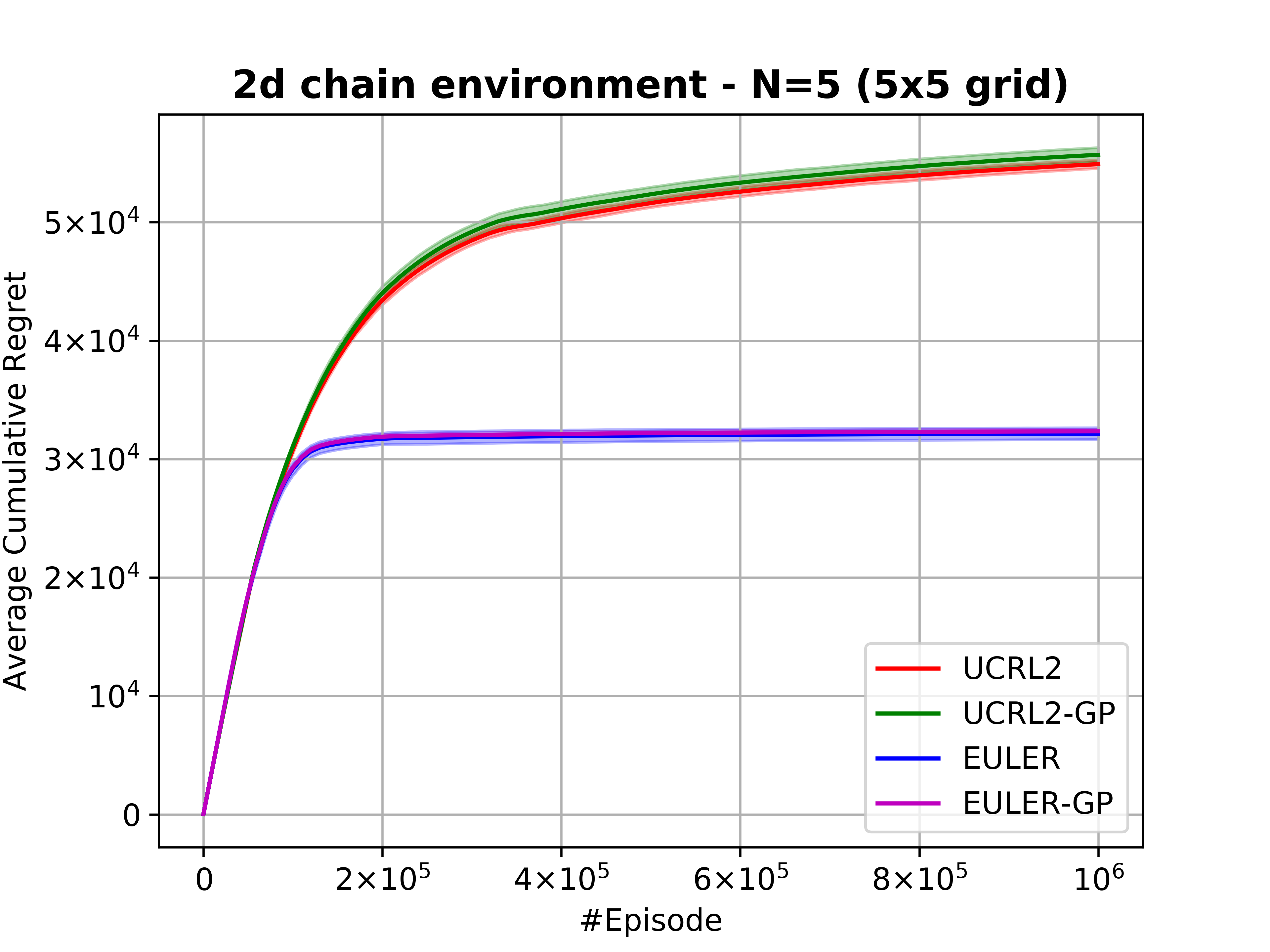}
\label{subfig:grid5}}
\caption{A comparison UCRL2 and EULER with their greedy counterpart. Results are averaged over 5 random seeds and are shown alongside error bars ($\pm3$std).}
\label{figure:experiments}
\end{figure}

The simulation results can be found in Figure \ref{figure:experiments}, and clearly indicate that using greedy planning leads to negligible degradation in the performance. Thus, the simulations verify our claim that greedy policy updates greatly improve the efficiency of the algorithm while maintaining the same performance.

\section{Related Work}

\textbf{Real-Time Dynamic Programming}: RTDP \citep{barto1995learning} has been extensively used and has many variants that exhibit superior empirical performance (e.g.,  \citep{bonet2003labeled, mcmahan2005bounded,smith2006focused}). For discounted MDPs, \citet{strehl2006pac} proved $(\epsilon,\delta)$-PAC bounds of $\tilde{O}\br*{SA/\epsilon^2\br*{1-\gamma}^4}$, for a modified version of RTDP in which the value updates occur only if the decrease in value is larger than $\epsilon\br*{1-\gamma}$. 
I.e., their algorithm explicitly use $\epsilon$ to mark states with accurate value estimate. We prove that RTDP converges in a rate of $\tilde{O}\br*{SH^2/\epsilon}$ without knowing $\epsilon$. Indeed, \cite{strehl2006pac} posed {\em whether the original RTDP is PAC} as an open problem. Furthermore, no regret bound for RTDP has been reported in the literature. 

\textbf{Regret bounds for RL}: The most renowned algorithms with regret guarantees for undiscounted infinite-horizon MDPs are UCRL2 \citep{jaksch2010near} and REGAL \citep{bartlett2009regal}, which have been extended throughout the years (e.g., by \citealt{fruit2018efficient,talebi2018variance}). Recently, there is an increasing interest in regret bounds for MDPs with finite horizon $H$ and stationary dynamics. In this scenario, UCRL2 enjoys a regret bound of order $HS\sqrt{AT}$. \citet{azar2017minimax} proposed UCBVI, with improved regret bound of order $\sqrt{HSAT}$, which is also asymptotically tight \citep{osband2016lower}. \citet{dann2018policy} presented ORLC that achieves tight regret bounds and  (nearly) tight PAC guarantees for non-stationary MDPs. Finally, \citet{zanette2019tighter} proposed EULER, an algorithm that enjoys tight minimax regret bounds and has additional problem-dependent bounds that encapsulate the MDP's complexity. All of these algorithms are model-based and require full-planning. Model-free RL was analyzed by \citep{jin2018q}. There, the authors exhibit regret bounds that are worse by a factor of $H$ relatively to the lower-bound. To the best of our knowledge, there are no model-based algorithms with regret guarantees that avoid full-planning. It is worth noting that while all the above algorithms, and the ones in this work, rely on the Optimism in the Face of Uncertainty principle \citep{lai1985asymptotically}, Thompson Sampling model-based RL algorithms exist \citep{osband2013more,gopalan2015thompson,agrawal2017optimistic, osband2017posterior}. There, a model is sampled from a distribution over models, on which full-planning takes place.

\textbf{Greedy policies in model-based RL:}
By adjusting RTDP to the case where the model is unknown, \citet{strehl2012incremental} formulated model-based RL algorithms that act using a greedy policy. They proved a  $\tilde{O}\br*{S^2A/\epsilon^3\br*{1-\gamma}^6}$ sample complexity bound for discounted MDPs. To the best of our knowledge, there are no regret bounds for model-based RL algorithms that act by greedy policies.

\textbf{Practical model-based RL}: Due to the high computational complexity of planning in model-based RL, most of the practical algorithms are model-free (e.g., \citealt{mnih2015human}). Algorithms that do use a model usually only take advantage of local information. For example, Dyna \citep{sutton1991dyna,peng2018deep} selects state-action pairs, either randomly or via prioritized sweeping \citep{moore1993prioritized,van2013planning}, and updates them according to a local model. Other papers use the local model to plan for a short horizon from the current state \citep{tamar2016value,hafner2018learning}. The performance of such algorithms depends heavily on the planning horizon, that in turn dramatically increases the computational complexity. 

\section{Conclusions and Future Work}

In this work, we established that tabular model-based RL algorithms can explore by 1-step planning instead of full-planning, without suffering from performance degradation. Specifically, exploring with model-based greedy policies can be minimax optimal in terms of regret. Differently put, the variance caused by exploring with greedy policies is smaller than the variance caused by learning a sufficiently good model. Indeed, the extra term which appears due to the greedy exploration is $\Olog(SH^2)$ (e.g., the first term in \eqref{eq: regret decomposition}); a constant term, smaller than the existing constant terms of UCRL2 and EULER. 

This work raises and highlights some interesting research questions. The obvious ones are extensions to average and discounted MDPs, as well as to Thompson sampling based RL algorithms. Although these scenarios are harder or different in terms of analysis, we believe this work introduces the relevant approach to tackle this question. Another interesting question is the applicability of the results in large-scale problems, when tabular representation is infeasible and approximation must be used. There, algorithms that act using lookahead policies, instead of 1-step planning, are expected to yield better performance, as they are less sensitive to value approximation errors (e.g., \citealt{bertsekas1996neuro,jiang2018feedback,efroni2018multiple,efroni2018combine}). Even then, full-planning, as opposed to using a short-horizon planning, might be unnecessary. Lastly, establishing whether the model-based approach is or is not provably better than the model-free approach, as the current state of the literature suggests, is yet an important and unsolved open problem.

\section*{Acknowledgments}
We thank Oren Louidor for illuminating discussions relating the Decreasing Bounded Process, and Esther Derman for the very helpful comments. This work was partially funded by the Israel Science Foundation under ISF grant number 1380/16.

\bibliographystyle{plainnat}
\bibliography{from_rtdp_to_rl}

\begin{thebibliography}{38}
\providecommand{\natexlab}[1]{#1}
\providecommand{\url}[1]{\texttt{#1}}
\expandafter\ifx\csname urlstyle\endcsname\relax
  \providecommand{\doi}[1]{doi: #1}\else
  \providecommand{\doi}{doi: \begingroup \urlstyle{rm}\Url}\fi

\bibitem[Agrawal and Jia(2017)]{agrawal2017optimistic}
Shipra Agrawal and Randy Jia.
\newblock Optimistic posterior sampling for reinforcement learning: worst-case
  regret bounds.
\newblock In \emph{Advances in Neural Information Processing Systems}, pages
  1184--1194, 2017.

\bibitem[Azar et~al.(2017)Azar, Osband, and Munos]{azar2017minimax}
Mohammad~Gheshlaghi Azar, Ian Osband, and R{\'e}mi Munos.
\newblock Minimax regret bounds for reinforcement learning.
\newblock In \emph{Proceedings of the 34th International Conference on Machine
  Learning-Volume 70}, pages 263--272. JMLR. org, 2017.

\bibitem[Bartlett and Tewari(2009)]{bartlett2009regal}
Peter~L Bartlett and Ambuj Tewari.
\newblock Regal: A regularization based algorithm for reinforcement learning in
  weakly communicating mdps.
\newblock In \emph{Proceedings of the Twenty-Fifth Conference on Uncertainty in
  Artificial Intelligence}, pages 35--42. AUAI Press, 2009.

\bibitem[Barto et~al.(1995)Barto, Bradtke, and Singh]{barto1995learning}
Andrew~G Barto, Steven~J Bradtke, and Satinder~P Singh.
\newblock Learning to act using real-time dynamic programming.
\newblock \emph{Artificial intelligence}, 72\penalty0 (1-2):\penalty0 81--138,
  1995.

\bibitem[Bertsekas and Tsitsiklis(1996)]{bertsekas1996neuro}
Dimitri~P Bertsekas and John~N Tsitsiklis.
\newblock \emph{Neuro-dynamic programming}, volume~5.
\newblock Athena Scientific Belmont, MA, 1996.

\bibitem[Bonet and Geffner(2003)]{bonet2003labeled}
Blai Bonet and Hector Geffner.
\newblock Labeled rtdp: Improving the convergence of real-time dynamic
  programming.
\newblock In \emph{ICAPS}, volume~3, pages 12--21, 2003.

\bibitem[Dann et~al.(2017)Dann, Lattimore, and Brunskill]{dann2017unifying}
Christoph Dann, Tor Lattimore, and Emma Brunskill.
\newblock Unifying pac and regret: Uniform pac bounds for episodic
  reinforcement learning.
\newblock In \emph{Advances in Neural Information Processing Systems}, pages
  5713--5723, 2017.

\bibitem[Dann et~al.(2018)Dann, Li, Wei, and Brunskill]{dann2018policy}
Christoph Dann, Lihong Li, Wei Wei, and Emma Brunskill.
\newblock Policy certificates: Towards accountable reinforcement learning.
\newblock \emph{arXiv preprint arXiv:1811.03056}, 2018.

\bibitem[de~la Pe{\~n}a et~al.(2007)de~la Pe{\~n}a, Klass, Lai,
  et~al.]{pena2007pseudo}
Victor~H de~la Pe{\~n}a, Michael~J Klass, Tze~Leung Lai, et~al.
\newblock Pseudo-maximization and self-normalized processes.
\newblock \emph{Probability Surveys}, 4:\penalty0 172--192, 2007.

\bibitem[de~la Pe{\~n}a et~al.(2008)de~la Pe{\~n}a, Lai, and
  Shao]{pena2008self}
Victor~H de~la Pe{\~n}a, Tze~Leung Lai, and Qi-Man Shao.
\newblock \emph{Self-normalized processes: Limit theory and Statistical
  Applications}.
\newblock Springer Science \& Business Media, 2008.

\bibitem[Efroni et~al.(2018{\natexlab{a}})Efroni, Dalal, Scherrer, and
  Mannor]{efroni2018combine}
Yonathan Efroni, Gal Dalal, Bruno Scherrer, and Shie Mannor.
\newblock How to combine tree-search methods in reinforcement learning.
\newblock \emph{arXiv preprint arXiv:1809.01843}, 2018{\natexlab{a}}.

\bibitem[Efroni et~al.(2018{\natexlab{b}})Efroni, Dalal, Scherrer, and
  Mannor]{efroni2018multiple}
Yonathan Efroni, Gal Dalal, Bruno Scherrer, and Shie Mannor.
\newblock Multiple-step greedy policies in approximate and online reinforcement
  learning.
\newblock In \emph{Advances in Neural Information Processing Systems}, pages
  5238--5247, 2018{\natexlab{b}}.

\bibitem[Fruit et~al.(2018)Fruit, Pirotta, Lazaric, and
  Ortner]{fruit2018efficient}
Ronan Fruit, Matteo Pirotta, Alessandro Lazaric, and Ronald Ortner.
\newblock Efficient bias-span-constrained exploration-exploitation in
  reinforcement learning.
\newblock \emph{arXiv preprint arXiv:1802.04020}, 2018.

\bibitem[Gopalan and Mannor(2015)]{gopalan2015thompson}
Aditya Gopalan and Shie Mannor.
\newblock Thompson sampling for learning parameterized markov decision
  processes.
\newblock In \emph{Conference on Learning Theory}, pages 861--898, 2015.

\bibitem[Hafner et~al.(2018)Hafner, Lillicrap, Fischer, Villegas, Ha, Lee, and
  Davidson]{hafner2018learning}
Danijar Hafner, Timothy Lillicrap, Ian Fischer, Ruben Villegas, David Ha,
  Honglak Lee, and James Davidson.
\newblock Learning latent dynamics for planning from pixels.
\newblock \emph{arXiv preprint arXiv:1811.04551}, 2018.

\bibitem[Jaksch et~al.(2010)Jaksch, Ortner, and Auer]{jaksch2010near}
Thomas Jaksch, Ronald Ortner, and Peter Auer.
\newblock Near-optimal regret bounds for reinforcement learning.
\newblock \emph{Journal of Machine Learning Research}, 11\penalty0
  (Apr):\penalty0 1563--1600, 2010.

\bibitem[Jiang et~al.(2018)Jiang, Ekwedike, and Liu]{jiang2018feedback}
Daniel~R Jiang, Emmanuel Ekwedike, and Han Liu.
\newblock Feedback-based tree search for reinforcement learning.
\newblock \emph{arXiv preprint arXiv:1805.05935}, 2018.

\bibitem[Jin et~al.(2018)Jin, Allen-Zhu, Bubeck, and Jordan]{jin2018q}
Chi Jin, Zeyuan Allen-Zhu, Sebastien Bubeck, and Michael~I Jordan.
\newblock Is q-learning provably efficient?
\newblock In \emph{Advances in Neural Information Processing Systems}, pages
  4863--4873, 2018.

\bibitem[Lai and Robbins(1985)]{lai1985asymptotically}
Tze~Leung Lai and Herbert Robbins.
\newblock Asymptotically efficient adaptive allocation rules.
\newblock \emph{Advances in applied mathematics}, 6\penalty0 (1):\penalty0
  4--22, 1985.

\bibitem[Maurer and Pontil(2009)]{maurer2009empirical}
Andreas Maurer and Massimiliano Pontil.
\newblock Empirical bernstein bounds and sample variance penalization.
\newblock \emph{arXiv preprint arXiv:0907.3740}, 2009.

\bibitem[McMahan et~al.(2005)McMahan, Likhachev, and
  Gordon]{mcmahan2005bounded}
H~Brendan McMahan, Maxim Likhachev, and Geoffrey~J Gordon.
\newblock Bounded real-time dynamic programming: Rtdp with monotone upper
  bounds and performance guarantees.
\newblock In \emph{Proceedings of the 22nd international conference on Machine
  learning}, pages 569--576. ACM, 2005.

\bibitem[Mnih et~al.(2015)Mnih, Kavukcuoglu, Silver, Rusu, Veness, Bellemare,
  Graves, Riedmiller, Fidjeland, Ostrovski, et~al.]{mnih2015human}
Volodymyr Mnih, Koray Kavukcuoglu, David Silver, Andrei~A Rusu, Joel Veness,
  Marc~G Bellemare, Alex Graves, Martin Riedmiller, Andreas~K Fidjeland, Georg
  Ostrovski, et~al.
\newblock Human-level control through deep reinforcement learning.
\newblock \emph{Nature}, 518\penalty0 (7540):\penalty0 529, 2015.

\bibitem[Moore and Atkeson(1993)]{moore1993prioritized}
Andrew~W Moore and Christopher~G Atkeson.
\newblock Prioritized sweeping: Reinforcement learning with less data and less
  time.
\newblock \emph{Machine learning}, 13\penalty0 (1):\penalty0 103--130, 1993.

\bibitem[Osband and Van~Roy(2016)]{osband2016lower}
Ian Osband and Benjamin Van~Roy.
\newblock On lower bounds for regret in reinforcement learning.
\newblock \emph{arXiv preprint arXiv:1608.02732}, 2016.

\bibitem[Osband and Van~Roy(2017)]{osband2017posterior}
Ian Osband and Benjamin Van~Roy.
\newblock Why is posterior sampling better than optimism for reinforcement
  learning?
\newblock In \emph{Proceedings of the 34th International Conference on Machine
  Learning-Volume 70}, pages 2701--2710. JMLR. org, 2017.

\bibitem[Osband et~al.(2013)Osband, Russo, and Van~Roy]{osband2013more}
Ian Osband, Daniel Russo, and Benjamin Van~Roy.
\newblock (more) efficient reinforcement learning via posterior sampling.
\newblock In \emph{Advances in Neural Information Processing Systems}, pages
  3003--3011, 2013.

\bibitem[Peng et~al.(2018)Peng, Li, Gao, Liu, Wong, and Su]{peng2018deep}
Baolin Peng, Xiujun Li, Jianfeng Gao, Jingjing Liu, Kam-Fai Wong, and Shang-Yu
  Su.
\newblock Deep dyna-q: Integrating planning for task-completion dialogue policy
  learning.
\newblock \emph{arXiv preprint arXiv:1801.06176}, 2018.

\bibitem[Smith and Simmons(2006)]{smith2006focused}
Trey Smith and Reid Simmons.
\newblock Focused real-time dynamic programming for mdps: Squeezing more out of
  a heuristic.
\newblock In \emph{AAAI}, pages 1227--1232, 2006.

\bibitem[Strehl and Littman(2008)]{strehl2008analysis}
Alexander~L Strehl and Michael~L Littman.
\newblock An analysis of model-based interval estimation for markov decision
  processes.
\newblock \emph{Journal of Computer and System Sciences}, 74\penalty0
  (8):\penalty0 1309--1331, 2008.

\bibitem[Strehl et~al.(2006)Strehl, Li, and Littman]{strehl2006pac}
Alexander~L Strehl, Lihong Li, and Michael~L Littman.
\newblock Pac reinforcement learning bounds for rtdp and rand-rtdp.
\newblock In \emph{Proceedings of AAAI workshop on learning for search}, 2006.

\bibitem[Strehl et~al.(2012)Strehl, Li, and Littman]{strehl2012incremental}
Alexander~L Strehl, Lihong Li, and Michael~L Littman.
\newblock Incremental model-based learners with formal learning-time
  guarantees.
\newblock \emph{arXiv preprint arXiv:1206.6870}, 2012.

\bibitem[Sutton(1991)]{sutton1991dyna}
Richard~S Sutton.
\newblock Dyna, an integrated architecture for learning, planning, and
  reacting.
\newblock \emph{ACM SIGART Bulletin}, 2\penalty0 (4):\penalty0 160--163, 1991.

\bibitem[Sutton and Barto(2018)]{sutton2018reinforcement}
Richard~S Sutton and Andrew~G Barto.
\newblock \emph{Reinforcement learning: An introduction}.
\newblock MIT press, 2018.

\bibitem[Talebi and Maillard(2018)]{talebi2018variance}
Mohammad~Sadegh Talebi and Odalric-Ambrym Maillard.
\newblock Variance-aware regret bounds for undiscounted reinforcement learning
  in mdps.
\newblock \emph{arXiv preprint arXiv:1803.01626}, 2018.

\bibitem[Tamar et~al.(2016)Tamar, Wu, Thomas, Levine, and
  Abbeel]{tamar2016value}
Aviv Tamar, Yi~Wu, Garrett Thomas, Sergey Levine, and Pieter Abbeel.
\newblock Value iteration networks.
\newblock In \emph{Advances in Neural Information Processing Systems}, pages
  2154--2162, 2016.

\bibitem[Van~Seijen and Sutton(2013)]{van2013planning}
Harm Van~Seijen and Richard~S Sutton.
\newblock Planning by prioritized sweeping with small backups.
\newblock \emph{arXiv preprint arXiv:1301.2343}, 2013.

\bibitem[Weissman et~al.(2003)Weissman, Ordentlich, Seroussi, Verdu, and
  Weinberger]{weissman2003inequalities}
Tsachy Weissman, Erik Ordentlich, Gadiel Seroussi, Sergio Verdu, and Marcelo~J
  Weinberger.
\newblock Inequalities for the l1 deviation of the empirical distribution.
\newblock \emph{Hewlett-Packard Labs, Tech. Rep}, 2003.

\bibitem[Zanette and Brunskill(2019)]{zanette2019tighter}
Andrea Zanette and Emma Brunskill.
\newblock Tighter problem-dependent regret bounds in reinforcement learning
  without domain knowledge using value function bounds.
\newblock \emph{arXiv preprint arXiv:1901.00210}, 2019.

\end{thebibliography}

\newpage

\appendix

 \section{Proofs on Decreasing Bounded Processes}\label{supp: proofs on DBP}

In this section, we state and prove useful results on Decreasing Bounded Processes (see Definition \ref{defn: DBP}). These results will be in use in proofs of the central theorems of this work.

\TheoremDecreasingProcess*
\begin{proof}
Without loss of generality, assume $C>0$, since otherwise the results are trivial. We start by remarking that $R_K$ is almost surely monotonically increasing, since $X_k\le X_{k-1}$. Define the martingale difference process 
\begin{align*}
\xi_k=X_k-\E\brs*{X_k \mid \F_{k-1}} 
= X_k-X_{k-1}-\E\brs*{X_k-X_{k-1} \mid \F_{k-1}}
\end{align*}
and the martingale process $M_K=\sum_{k=1}^K \xi_k$. Since $X_k\ge0$ almost surely, $R_K$ can be bounded by $R_K =M_K+X_0-X_K\le X_0+M_K$. Also define the quadratic variations as $\innorm{M}_K=\sum_{k=1}^K \E\brs*{\xi_k^2\mid \F_{k-1}}$ and $\brs*{M}_K=\sum_{k=1}^K \xi_k^2$. Next, recall Theorem 2.7 of \citep{pena2007pseudo}:

\begin{theorem}
Let $A$ and $B$ be two random variables, such that for all $\lambda\in\mathbb{R}$, we have
\begin{align}
\label{eq:canonical}
    \E\brs*{e^{\lambda A -\frac{\lambda^2B^2}{2}}}\le1 \enspace.
\end{align}
Then, $\forall x>0$,
\begin{align}
\label{eq:self normalized}
    \Pr\brc*{\frac{\abs{A}}{\sqrt{B^2+\E\brs*{B^2}}}> x} \le \sqrt{2}e^{-x^2/4}.
\end{align}
\end{theorem}

Condition (\ref{eq:canonical}) holds for $A_K=M_K$ and $B_K^2=\innorm{M}_K+\brs*{M}_K$, due to Theorem 9.21 of~\citep{pena2008self}. $A_K$ can be easily bounded by $\abs{A_K}\ge R_K-X_0\ge R_K-C$. To bound $B_K^2$, we first calculate $\xi_k^2$ and $\E\brs*{\xi_k^2\mid \F_{k-1}}$:
\begin{align*}
    &\xi_k^2
    =  \br*{X_k-X_{k-1}}^2-2\br*{X_k-X_{k-1}}\E\brs*{X_k-X_{k-1} \mid \F_{k-1}} + \br*{\E\brs*{X_k-X_{k-1} \mid \F_{k-1}}}^2, \\
    &\E\brs*{\xi_k^2\mid \F_{k-1}} = \E\brs*{\br*{X_k-X_{k-1}}^2\mid \F_{k-1}} - \br*{\E\brs*{X_k-X_{k-1} \mid \F_{k-1}}}^2.
\end{align*}
Thus, 
\begin{align*}
    &\xi_k^2 + \E\brs*{\xi_k^2\mid \F_{k-1}} \\
    &=  \br*{X_k-X_{k-1}}^2 + \E\brs*{\br*{X_k-X_{k-1}}^2\mid \F_{k-1}} -2\br*{X_k-X_{k-1}}\E\brs*{X_k-X_{k-1} \mid \F_{k-1}} \\
    & \stackrel{(*)}{\le} \br*{X_k-X_{k-1}}^2 + \E\brs*{\br*{X_k-X_{k-1}}^2\mid \F_{k-1}} \\
    & \stackrel{(**)}{\le} \br*{X_k-X_{k-1}}^2 + C\E\brs*{X_{k-1}-X_k\mid \F_{k-1}} \enspace .
\end{align*}
In $(*)$ we used the fact that $X_{k-1}- X_k \ge0$ a.s., which allows us to conclude that the cross-term is non-positive. In (**), we bounded $0\le X_{k-1}-X_k\le C$. We can also bound $\sum_{k=1}^K \br*{X_{k-1}-X_k}^2 \le C^2$, since each of the summands is a.s. non-negative, and thus,
\begin{align*}
    \sum_{k=1}^K \br*{X_{k-1}-X_k}^2
    \le \br*{\sum_{k=1}^K X_{k-1}-X_k}^2
    = \br*{X_K-X_0}^2
    \le C^2.
\end{align*}
Combining all of the above bounds yields
\begin{align*}
    B_K^2 & \le \sum_{k=1}^K \br*{ \br*{X_k-X_{k-1}}^2 +C\E\brs*{X_{k-1}-X_k\mid \F_{k-1}}} \\
    & \le C^2 +C \sum_{k=1}^K \E\brs*{X_{k-1}-X_k\mid \F_{k-1}} = C^2 + CR_K.
\end{align*}
Finally, we can bound $\E\brs*{B_K^2}$ by
\begin{align*}
    \E\brs*{B_K^2} & = \sum_{k=1}^K \E\brs*{\xi_k^2 + \E\brs*{\xi_k^2\mid \F_{k-1}}} = 2\sum_{k=1}^K \E\brs*{\E\brs*{\xi_k^2\mid \F_{k-1}}} \\
    & = 2\sum_{k=1}^K \E\brs*{\E\brs*{\br*{X_k-X_{k-1}}^2\mid \F_{k-1}} - \br*{\E\brs*{X_k-X_{k-1} \mid \F_{k-1}}}^2} \\
    & \le 2\sum_{k=1}^K \E\brs*{\E\brs*{\br*{X_k-X_{k-1}}^2\mid \F_{k-1}}} = 2\sum_{k=1}^K \br*{X_k-X_{k-1}}^2 \le 2C^2.
\end{align*}
Combining everything we obtain
\begin{align*}
    \frac{\abs{A}}{\sqrt{B^2+\E\brs*{B^2}}} 
    \ge \frac{R_K-C}{\sqrt{C^2 + CR_K + 2C^2}} 
    = \frac{R_K-C}{\sqrt{3C^2 + CR_K}}. 
\end{align*}
Or, substituting in~\eqref{eq:self normalized}, we have
\begin{align*}
    \Pr\brc*{\frac{R_K-C}{\sqrt{3C^2 + CR_K}} > x} \le\Pr\brc*{\frac{\abs{A}}{\sqrt{B^2+\E\brs*{B^2}}}> x} \le \sqrt{2}e^{-x^2/4}.
\end{align*}

Next, notice that for $C>0$, the function $f(y)=\frac{y-C}{\sqrt{3C^2+Cy}}$ is monotonically increasing for any $y>0$:

\begin{align*}
    f'(y)
    =\frac{\sqrt{3C^2+Cy} - \frac{C(y-C)}{2\sqrt{3C^2+Cy}}}{3C^2+Cy}
    = \frac{2(3C^2+ Cy) -Cy + C^2}{2\br*{3C^2+Cy}^{3/2}}
    = \frac{7C^2+ Cy}{2\br*{3C^2+Cy}^{3/2}}
    >0
\end{align*}

Moreover, for $y=C(1+x)^2$,

\begin{align*}
    f\br*{C(1+x)^2}
    &= \frac{C(1+x)^2-C}{\sqrt{3C^2+C^2(1+x)^2}}
    = \frac{Cx^2+2Cx}{\sqrt{4C^2+2C^2x+C^2x^2}}\\
    &> \frac{Cx^2+2Cx}{\sqrt{4C^2+4C^2x+C^2x^2}}
    = \frac{Cx^2+2Cx}{Cx+2C} =x\enspace,
\end{align*}

where the inequality holds since $x>0$. Thus, if $R_K \ge C(1+x)^2$, then $f(R_K)>x$, and we can bound the probability that $R_K \ge C(1+x)^2$ by 

\begin{align*}
    \Pr\brc*{R_K \ge C(1+x)^2} \le
    \Pr\brc*{\frac{R_K-C}{\sqrt{3C^2 + CR_K}} > x}  \le \sqrt{2}e^{-x^2/4} \enspace ,
\end{align*}
and setting $x=2\sqrt{\ln\frac{2}{\delta}}>0$, we obtain
\begin{align*}
    \Pr\brc*{R_K \ge C\br*{1+2\sqrt{\ln\frac{2}{\delta}}}^2} \le \delta \enspace .
\end{align*}
We remark that since $R_K$ is monotonically increasing a.s., this bound also implies that 
\begin{align*}
    \Pr\brc*{\exists N: 1\le N\le K, \; R_N \ge C\br*{1+2\sqrt{\ln\frac{2}{\delta}}}^2} \le \delta.
\end{align*}
To obtain a uniform bound, that is, bound that holds for all $K>0$, note that the random sequence 
$Z_K=\ind\brc*{\exists 1\le N\le K: R_N \ge C\br*{1+2\sqrt{\ln\frac{2}{\delta}}}^2}$ is monotonically increasing in $K$ and bounded. Thus, due to monotone convergence
\begin{align*}
    \Pr&\brc*{\exists K>0: R_K \ge C\br*{1+2\sqrt{\ln\frac{2}{\delta}}}^2}
    = \E\brs*{\lim_{K\to\infty}Z_K}
    = \lim_{K\to\infty}\E\brs*{Z_K} \\
    & = \lim_{K\to\infty}\Pr\brc*{\exists 1\le N\le K: R_N \ge C\br*{1+2\sqrt{\ln\frac{2}{\delta}}}^2} \le \delta.
\end{align*}
To conclude the proof, note that $\delta\le1$, and thus, $\ln\frac{3}{\delta}\ge1$. Therefore, we can bound
\begin{align*}
     C\br*{1+2\sqrt{\ln\frac{2}{\delta}}}^2
     \le C\br*{1+2\sqrt{\ln\frac{3}{\delta}}}^2
     \le C\br*{3\sqrt{\ln\frac{3}{\delta}}}^2 
     = 9C\ln\frac{3}{\delta}\enspace ,
\end{align*}
which yields the second bound.
\end{proof}

\begin{lemma}
\label{lemma: sum of decreasing processes}
Let $\{X^k_n\}_{k\geq 1}$ be a Bounded Decreasing Process in $[0,C]$ for any $n\in [N]$. The regret of the sum of processes is defined as $R(K) =\sum_{n=1}^N \sum_{k=1}^K X_n^{k-1}-\E[X^k_n\mid \mathcal{F}_{k-1}]$. Then, for any $\delta>0$, we have
\begin{align*}
    \Pr\brc*{\exists K>0: R(K)\geq 9CN\ln\frac{3N}{\delta}}\leq \delta.
\end{align*}
\end{lemma}

\begin{proof}
We first remark that if $X_n^0<C$, we can replace it to $X_n^0=C$, which only increases the regret, so we assume w.l.o.g. that $X_n^0=C$. Define
\begin{align*}
    R_{n}(K)\eqdef \sum_{k=1}^K X^{k-1}_n(s) - \E[X^{k}_n(s)\mid \F_{k-1}].
\end{align*}
Define the event $A_{n}\eqdef \brc*{\exists K>0: R_{n}(K)\ge 9CN\ln\frac{3N}{\delta} }$. By applying Theorem \ref{theorem: regret of decreasing process}, with probability $\frac{\delta}{N}$, it holds that for a fixed $n\in[N]$
\begin{align}
        \Pr\brc*{\exists K>0: R_{n}(K)\ge 9C\ln\frac{3N}{\delta}} = \Pr\brc*{A_{n}} \le \frac{\delta}{N}.\label{eq supp: fixed s t bound rtdp}
\end{align}
Finally, we obtain
\begin{align*}
    \Pr\brc*{\exists K>0: R(K) \ge 9NC\ln\frac{3N}{\delta}} &= \Pr\brc*{\exists K>0: \sum_{n=1}^N R_{n}(K) \ge 9NC\ln\frac{3N}{\delta}}\\
    &\stackrel{(1)}{\leq}  \Pr\brc*{\bigcup_{n=1}^N A_{n}} \stackrel{(2)}{\leq}  \sum_{n=1}^N \Pr\brc*{A_{n}} \stackrel{(3)}{\leq} \delta.
\end{align*}
Relation $(1)$ holds since $$ \brc*{ \exists K>0: \sum_{n=1}^N R_{n}(K) \ge 9NC\ln\frac{3N}{\delta}} \subseteq \bigcup_{n=1}^N A_{n}.$$ In $(2)$ we use the union bound and $(3)$ holds by \eqref{eq supp: fixed s t bound rtdp}. \end{proof}

\newpage
\section{Proof of Real-Time Dynamic Programming Bounds}\label{sec:supp proof RTDP}

\rtdpProperties*
\begin{proof}
Both claims are proven using induction. 

\paragraph{{\em (i)}} 
By the initialization, $\forall s,t,\  V^*_t(s) \leq V^0_t(s)$. Assume the claim holds for $k-1$ episodes. Let $s_t^{k}$ be the state the algorithm is at in the $t^{th}$ time-step of the $k^{th}$ episode. By the value update of Algorithm~\ref{algo: RTDP},
\begin{align*}
     \bar{V}_t^{k}(s_t^{k}) &= \max_a \; r(s_t^k,a)+\sum_{s'} p(s'\mid s_t^k,a) \bar{V}_{t+1}^{k-1}(s') \\
     &\geq \max_a \; r(s_t^k,a)+\sum_{s'} p(s'\mid s_t^k,a) \bar{V}_{t+1}^*(s') = V^*(s_t^k).
\end{align*}
The second relation holds by the induction hypothesis and the monotonicity of the optimal Bellman operator~\citep{bertsekas1996neuro}. The third relation holds by the recursion satisfied by the optimal value function (see Section~\ref{sec: setup}). Thus, the induction step is proven for the first claim.

\paragraph{{\em (ii)}} 
To prove the base case of the second claim we use the optimistic initialization. Let $s^1_t$ be the state the algorithm is at in the $t^{th}$ time-step of the first episode. By the update rule,
\begin{align*}
    \bar{V}^1_t(s^1_t)&= \max_a \; r(s^1_t,a)+\sum_{s'}p(s'\mid s^1_t,a)\bar{V}^0_{t+1}(s')\\
    & \stackrel{(1)}{=} \max_a \; r(s_t^1,a)+H-t \\
    &\stackrel{(2)}{\le} 1+H-t = H-(t-1) \stackrel{(3)}{=}\bar{V}^0_t(s^1_t).
\end{align*}
Relation $(1)$ holds by the initialization of the values, $(2)$ holds since $r(s,a)\in [0,1]$ and $(3)$ is by the initialization. States that were not visited on the first episode were not update, and thus the inequality trivially holds. 

Assume the second claim holds for $k-1$ episodes. Let $s_t^{k}$ be the state that the algorithm is at in the $t^{th}$ time-step of the $k^{th}$ episode. By the value update of Algorithm \ref{algo: RTDP}, we have
\begin{align*}
     \bar{V}_t^{k}(s_t^{k}) &= \max_a \; r(s_t^k,a)+\sum_{s'} p(s'\mid s_t^k,a) \bar{V}_{t+1}^{k-1}(s').
\end{align*}
If $s_t^k$ was previously updated, let  $\bar{k}$ be the previous episode in which the update occured. By the induction hypothesis, we have that $\forall s,t,\ \bar{V}^{\bar{k}}_t(s)\geq \bar{V}^{k-1}_t(s)$. Using the monotonicity of the Bellman operator~\citep{bertsekas1996neuro}, we may write
\begin{align*}
    &\max_a \; r(s_t^k,a)+\sum_{s'} p(s'\mid s_t^k,a) \bar{V}_{t+1}^{k-1}(s')\\
    &\leq \max_a \; r(s_t^k,a)+\sum_{s'} p(s'\mid s_t^k,a) \bar{V}_{t+1}^{\bar{k}-1}(s') = \bar{V}^{k-1}(s_t^k).
\end{align*}
Thus, $\bar{V}_t^{k}(s_t^{k}) \leq \bar{V}^{k-1}(s_t^k)$ and the induction step is proved. If $s_t^k$ was not previously updated, then $\bar{V}_t^{k-1}(s_t^{k})=\bar{V}_t^{0}(s_t^{k})$. In this case, the induction hypothesis implies that $\forall s', \bar{V}_{t+1}^{k-1}(s')\le \bar{V}_{t+1}^{0}(s')$ and the result can be proven similarly to the base case.
\end{proof}

\clearpage




\rdtpExpectedValueUpdate*

\begin{proof}
By the definition of $a_t^k$ and the update rule, the following holds:

\begin{align*}
     \E\brs*{\bar{V}_t^{k}(s_t^k) \mid \F_{k-1}} &= \E\brs*{r(s_t^k,a_t^k)+ p(\cdot \mid s_t^k,a_t^k)^T \bar{V}^{k-1}_{t+1} \mid \F_{k-1}}  \\
     &= \E\brs*{r(s_t^k,a_t^k)\mid \F_{k-1}}
     +\E\brs*{\sum_{\bar{s}_{t+1}}p(\bar{s}_{t+1}\mid s_t,\pi_k)\bar{V}_{t+1}^{k-1} (\bar{s}_{t+1})\mid \F_{k-1}}.
\end{align*}

Furthermore,
\begin{align}
    &\E\brs*{\sum_{\bar{s}_{t+1}}p(\bar{s}_{t+1}\mid s_t^k,\pi_k)\bar{V}_{t+1}^{k-1}(\bar{s}_{t+1})\mid \F_{k-1}} \nonumber\\
    &= \sum_{s_t^k} \Pr(s_t^k \mid s_1^k,\pi_k) \sum_{\bar{s}_{t+1}\in \mathcal{S}}p(\bar{s}_{t+1}\mid s_t^k,\pi_k)\bar{V}_{t+1}^{k-1}(\bar{s}_{t+1})\nonumber\\
    &= \sum_{s_{t+1}^k\in \mathcal{S}}  \Pr(s_{t+1}^k \mid s_1^k,\pi_k)  \bar{V}_{t+1}^{k-1}(s_{t+1}^k) 
    = \E\brs*{\bar{V}_{t+1}^{k-1}(s_{t+1}^k) \mid \F_{k-1}}. \label{eq:value transition update}
\end{align}
The first relation holds by definition and the second one holds by the Markovian property of the dynamics. Substituting back and summing both side from $t=1,\ldots,H$, we obtain

\begin{align*}
    \E\brs*{\sum_{t=1}^H \bar{V}_t^{k}(s_t^k) \mid \F_{k-1}} 
     &= \E\brs*{\sum_{t=1}^H r(s_t^k,a_t^k)\mid \F_{k-1}} + \E\brs*{\sum_{t=1}^H \bar{V}_{t+1}^{k-1}(s_{t+1}^k) \mid \F_{k-1}} \\
     &= \E\brs*{\sum_{t=1}^H r(s_t^k,a_t^k)\mid \F_{k-1}} + \E\brs*{\sum_{t=1}^H \bar{V}_t^{k-1}(s_t^k) \mid \F_{k-1}} - \bar{V}^{k-1}_1(s_1^k) \\
     &= V^{\pi_k}_1(s_1^k) + \E\brs*{\sum_{t=1}^H \bar{V}_t^{k-1}(s_t^k) \mid \F_{k-1}} - \bar{V}^{k-1}_1(s_1^k)
\end{align*}

The second line hold by shifting the index of the sum and using the fact that $\forall s,\ \bar{V}^{k}_{H+1}(s)=0$. The third line holds by the definition of the value function,
\begin{align*}
    \sum_{t=1}^H \E[r(s_t^k,a_t^k) \mid \F_{k-1}]  =\E[\sum_{t=1}^H r(s_t^k,a_t^k) \mid s_1= s_1^k ] = V^{\pi_k}_1(s_1^k).
\end{align*}
Reorganizing the equation yields the desired result.
\end{proof}

\clearpage

\TheoremRegretRTDP*
\begin{proof}
The following bounds on the regret hold.
\begin{align}
     \Regret(K)\eqdef \sum_{k=1}^K V^*(s^k_1)- V^{\pi_k}(s^k_1)
     & \leq  \sum_{k=1}^K \bar{V}_1^{k-1}(s^k_1)- V^{\pi_k}(s^k_1)\nonumber \\
     &\leq \sum_{k=1}^K\sum_{t=1}^{H} \E[ \bar{V}_t^{k-1}(s_t^k) - \bar{V}_t^{k}(s_t^k)\mid \F_{k-1}]. \label{supp eq: rtdp general bound}
\end{align}
The second relation is by the optimism of the value function (Lemma \ref{lemma:rtdp properties}), and the third relation is by Lemma \ref{lemma: RTDP expected value difference}.

To prove the bound on the expected regret, we take expectation on both sides of \eqref{supp eq: rtdp general bound}. Thus,
\begin{align*}
    \E[\Regret(K)] &\leq \sum_{k=1}^K \E[\E[\sum_{t=1}^{H} \bar{V}_t^{k-1}(s_t^k) - \bar{V}_t^{k}(s_t^k)\mid \F_{k-1}]] =  \E[\sum_{k=1}^K\sum_{t=1}^{H} \bar{V}_t^{k-1}(s_t^k) - \bar{V}_t^{k}(s_t^k)].
\end{align*}
Where the second relation holds by the tower property and linearity of expectation. Finally, for any run of RTDP, we have that
\begin{align*}
    \sum_{k=1}^K\sum_{t=1}^{H} \bar{V}_t^{k-1}(s_t^k) - \bar{V}_t^{k}(s_t^k) &= \sum_{s}\sum_{t=1}^H  \bar{V}_t^{0}(s) - \bar{V}_t^{K}(s) \leq \sum_{s}\sum_{t=1}^H  \bar{V}_t^{0}(s) - V_t^{*}(s) \leq SH^2.
\end{align*}
The first relation holds since per $s$, the sum is telescopic, thus, only the first and last term exist in the sum. Due to the update rule, on the first time a state appears, its value will be $\bar{V}_t^{0}(s)$. From the last time it appears, its value will not be updated and thus the last value of a state is  $\bar{V}_t^{K}(s)$. The second relation holds by Lemma \ref{lemma:rtdp properties}. The third relation holds since $\forall s,t,\ \bar{V}_t^{0}(s) - V_t^{*}(s) \in [0,H]$, summing on $SH$ such terms yields the result.

To prove the high-probability bound we apply Lemma \ref{lemma: regret to SH decreasing processes} by which,
\begin{align*}
    \eqref{supp eq: rtdp general bound}
    = \sum_{t=1}^H\sum_{s}\sum_{k=1}^K \bar{V}_t^{k-1}(s) - \E[\bar{V}_t^{k}(s)\mid \F_{k-1}].
\end{align*}

For a fixed $s,t$, $\brc*{\bar{V}_t^{k}(s)}_{k\geq 0}$ is a Decreasing Bounded Process by Lemma \ref{lemma:rtdp properties}, and its initial value is less than $H$. Thus, \eqref{supp eq: rtdp general bound} is a sum of $SH$ Decreasing Bounded Processes. We apply Lemma \ref{lemma: sum of decreasing processes} which provides a high-probability bound on a sum of Decreasing Bounded Processes to conclude the proof. 
\end{proof}

\clearpage

\CorPACrtdp*

\begin{proof}

Let $K_{N_\epsilon}$ be an episode index such that there are $N_\epsilon$ previous episodes $k\le K_{N_\epsilon}$ in which RTDP outputs a policy with $V_1^*(s_1^k) -V_1^{\pi_k}(s_1^k)>\epsilon$. The following relation holds,
\begin{align*}
    \forall \epsilon >0: kkvbvuvdhfirinhblbkudchurbknbulrN_\epsilon \epsilon \leq \Regret(K_{N_\epsilon}). 
\end{align*}
Thus,
\begin{align*}
    \brc*{\exists \epsilon>0: N_\epsilon \epsilon \geq 9SH^2 \ln\frac{3SH}{\delta}} &\subseteq	 \brc*{\Regret(K_{N_\epsilon}) \geq 9SH^2 \ln\frac{3SH}{\delta}} \\
    &\subseteq	 \brc*{\exists K>0: \Regret(K) \geq 9SH^2 \ln\frac{3SH}{\delta}}.
\end{align*}

Which results in
\begin{align*}
    \Pr\brc*{\exists \epsilon>0: N_\epsilon \epsilon \geq 9SH^2 \ln\frac{3SH}{\delta}} \leq \Pr\brc*{\exists K>0: \Regret(K) \geq 9SH^2 \ln\frac{3SH}{\delta}} \leq \delta.
\end{align*}


where the third relation holds by Theorem \ref{theorem: regret rtdp}.
\end{proof}

\CorPACrtdpTotalEpisodes*
\begin{proof}
We have that $N = N_{\Delta(\mathcal{M})}$ since $\Delta(\mathcal{M})$ is the minimal gap; in all rest of episodes in which the gap is smaller than $\Delta(\mathcal{M})$, the policy $\pi_k$ is necessarily the optimal one. Based on Corollary \ref{corollary: pac RTDP} we conclude that,
\begin{align*}
        \Pr\brc*{N \geq \frac{9SH^2\ln\frac{3SH}{\delta}}{ \Delta(\mathcal{M}) }}=\Pr\brc*{N_{\Delta(\mathcal{M})} \geq \frac{9SH^2\ln\frac{3SH}{\delta}}{ \Delta(\mathcal{M}) }}\leq \delta.
\end{align*}
\end{proof}

\newpage
\section{Proofs of Section \ref{sec: model base RL with 1-step greedy}}\label{supp proofs for general greedy policy RL}

\LemmaModelBaseDecomposition*
We prove a more general, Lemma \ref{lemma: FULL model base RL expected value difference},  of which  Lemma \ref{lemma: model base RL expected value difference} is a direct corollay (by setting $t=1$).

\begin{lemma}\label{lemma: FULL model base RL expected value difference}
The expected value update of Algorithm \ref{alg: general model based RL algorithm} in the $k^{th}$ episode at the state $t^{th}$ is bounded by
\begin{align*}
    \bar{V}_t^{k-1}(s_t^k) &- V_t^{\pi_k}(s_t^k) \leq  \sum_{t'=t}^{H}\E\brs*{ \bar{V}_{t'}^{k-1}(s_{t'}^k) - \bar{V}_{t'}^{k}(s_{t'}^k)\mid \F_{k-1},s_t^k}\\
    & + \sum_{t'=t}^{H}\E \brs*{(\tilde{r}_{k-1} - r)(s_{t'}^k,a_{t'}^k)+(\tilde{p}_{k-1} - p)(\cdot\mid s^k_{t'},a_{t'}^k)^T \bar{V}_{t'+1}^{k-1} \mid \F_{k-1},s_t^k  }\enspace.
\end{align*}
\end{lemma}

\begin{proof}
We closely follow the proof of Lemma \ref{lemma: RTDP expected value difference}. By the definition of $a_t^k$ and the update rule, for $t'\geq t$, the following holds.
\begin{align*}
     &\E\brs*{\bar{V}_{t'}^{k}(s_{t'}^k) \mid \F_{k-1},s_t^k }  \\
     &\stackrel{(1)}{\leq} \E \brs*{\tilde{r}_{k-1}(s_{t'}^k,a_{t'}^k)+\sum_{\bar{s}_{{t'}+1}\in \mathcal{S}}\tilde{p}_{k-1}(\bar{s}_{{t'}+1}\mid s^k_{t'},a_{t'})\bar{V}_{{t'}+1}^{k-1} (\bar{s}_{{t'}+1}) \mid \F_{k-1},s_t^k }\\
     &\stackrel{(2)}{=} \E\brs*{r(s_{t'}^k,a_{t'}^k)\mid \F_{k-1},s_{t'}^k }+\E\brs*{\sum_{\bar{s}_{{t'}+1}\in \mathcal{S}}p(\bar{s}_{{t'}+1}\mid s_{t'}^k,a_{t'}^k)\bar{V}_{{t'}+1}^{k-1} (\bar{s}_{{t'}+1}) \mid \F_{k-1},s_t^k }\\
     &\quad +\E\brs*{(\tilde{r}_{k-1}-r)(s_{t'}^k,a_{t'}^k)+\sum_{\bar{s}_{{t'}+1}\in \mathcal{S}}(\tilde{p}_{k-1}-p)(\bar{s}_{{t'}+1}\mid s^k_{t'},a_{t'}^k)\bar{V}_{{t'}+1}^{k-1} (\bar{s}_{{t'}+1}) \mid \F_{k-1},s_t^k }\\
     &\stackrel{(3)}{=} \E\brs*{r(s_{t'}^k,a_{t'}^k) \mid \mathcal{F}_{k-1},,s_t^k} +\E\brs*{\bar{V}_{{t'}+1}^{k-1}(s_{{t'}+1}^k) \mid \mathcal{F}_{k-1},s_t^k}\\
     &\quad + \E \brs*{(\tilde{r}_{k-1}-r)(s_{t'}^k,a_{t'}^k)+\sum_{\bar{s}_{{t'}+1}\in \mathcal{S}}(\tilde{p}_{k-1}-p)(\bar{s}_{{t'}+1}\mid s_{t'}^k,a_{t'}^k)\bar{V}_{{t'}+1}^{k-1} (\bar{s}_{{t'}+1})\mid \F_{k-1},s_t^k   }
\end{align*}
Relation $(1)$ holds by the update rule for $\bar{V}_t^k$. Next, $(2)$ holds by adding an subtracting the real reward and dynamics and using linearity of expectation. In $(3)$, we used the same reasoning as in Equation~\eqref{eq:value transition update}. 

Summing both side from $t'=t,\ldots,H$, we obtain:
\begin{align*}
    \E&\brs*{\sum_{t'=t}^H \bar{V}_{t'}^{k}(s_{t'}^k) \mid \F_{k-1},s_t^k} \\
    & \le \E\brs*{\sum_{{t'}=t}^Hr(s_{t'}^k,a_{t'}^k) \mid \F_{k-1},s_t^k} +\E\brs*{\sum_{{t'}=t}^H\bar{V}_{{t'}+1}^{k-1}(s_{{t'}+1}^k) \mid \F_{k-1},s_t^k}\\
     &\quad + \sum_{{t'}=t}^H\E \brs*{(\tilde{r}_{k-1}-r)(s_{t'}^k,a_{t'}^k)+\sum_{\bar{s}_{{t'}+1}\in \mathcal{S}}(\tilde{p}_{k-1}-p)(\bar{s}_{{t'}+1}\mid s_{t'}^k,a_{t'}^k)\bar{V}_{{t'}+1}^{k-1} (\bar{s}_{{t'}+1})\mid \F_{k-1},s_t^k  }\\
     &\stackrel{(1)}{=} V_t^{\pi_k}(s_t^k) +E\brs*{\sum_{{t'}=t}^H\bar{V}_{{t'}+1}^{k-1}(s_{{t'}+1}^k) \mid \F_{k-1},s_t^k}\\
     &\quad + \sum_{{t'}=t}^H\E \brs*{(\tilde{r}_{k-1}-r)(s_{t'}^k,a_{t'}^k)+\sum_{\bar{s}_{{t'}+1}\in \mathcal{S}}(\tilde{p}_{k-1}-p)(\bar{s}_{{t'}+1}\mid s_{t'}^k,a_{t'}^k)\bar{V}_{{t'}+1}^{k-1} (\bar{s}_{{t'}+1})\mid \F_{k-1},s_t^k  }\\
     &\stackrel{(2)}{=} V_t^{\pi_k}(s_t^k) +\E\brs*{\sum_{{t'}=t}^H \bar{V}_{t'}^{k-1}(s_{t'}^k) \mid\F_{k-1},s_t^k} - \bar{V}^{k-1}_t(s_t^k)\\
     &\quad + \sum_{{t'}=t}^H\E \brs*{(\tilde{r}_{k-1}-r)(s_{t'}^k,a_{t'}^k)+\sum_{\bar{s}_{{t'}+1}\in \mathcal{S}}(\tilde{p}_{k-1}-p)(\bar{s}_{{t'}+1}\mid s_{t'}^k,a_{t'}^k)\bar{V}_{{t'}+1}^{k-1} (\bar{s}_{{t'}+1})\mid \F_{k-1},s_t^k  }
\end{align*}
In $(1)$ we used the fact that $V_t^{\pi_k}(s_t^k) = \E[\sum_{t'=t}^{H} r(s_{t'}^k,\pi_k(s_{t'}^k))\mid \mathcal{F}_{k-1},s_t^k]$. Relation $(2)$ holds by shifting the index of the sum and using $\forall s,k,\ \bar{V}_{H+1}^{k-1}(s)=0$. Reorganizing the equation yields the desired result.
\end{proof}

\newpage
\section{Proof of Theorem \ref{theorem: UCRL2 with greedy policies}}\label{supp: UCRL2}

\begin{algorithm}
\begin{algorithmic}[1]
\caption{UCRL2 with Greedy Policies}
\label{alg supp: UCRL2 with greedy policies}
    \STATE Initialize: $\delta, \delta'=\frac{\delta}{4} \forall s\in \mathcal{S}, t\in [H],\ \bar{V}^0_t(s)=H-(t-1).$
    \FOR{$k=1,2,..$}
        \STATE Initialize $s_1^k$
        \FOR{$t=1,..,H$}
            \STATE {\color{gray}{$\#$Update Upper Bound on $V^*$}}
            \FOR{$a\in \mathcal{A}$}
                \STATE $\tilde{r}_{k-1}(s_t^k,a) = \hat{r}_{k-1}(s_t^k,a) + \sqrt{\frac{2\ln \frac{2SAT}{\delta'}}{n_{k-1}(s_t^k,a)\vee 1}}$
                \STATE  $CI(s_t^k,a) = \brc*{P'\in \mathcal{P}(\mathcal{S}): \norm{P'(\cdot)-\hat{p}_{k-1}(\cdot\mid s_t^k,a)}_1\leq \sqrt{\frac{4S\ln\frac{3SAT}{\delta'}}{n_{k-1}(s_t^k,a)\vee 1}}}$
                \STATE $\tilde{p}_{k-1}(\cdot\mid s_t^k,a) = \arg\max_{P'\in CI(s_t^k,a)} P'(\cdot\mid s_t^k,a )^T \bar{V}^{k-1}_{t+1}$ 
                \STATE $\bar{Q}(s_t^k,a) = \tilde{r}_{k-1}(s_t^k,a)+\tilde{p}_{k-1}(\cdot\mid s_t^k,a)^T\bar{V}^{k-1}_{t+1}$
                \ENDFOR
          
            \STATE $a_t^k\in \arg\max_a \bar{Q}(s_t^k,a)$
            \STATE $\bar{V}^{k}_{t}(s_t^k) =\min\brc*{\bar{V}^{k-1}_{t}(s_1^k),\bar{Q}(s_t^k,a_t^k)}$ 
            \STATE {\color{gray}{$\#$Act by the Greedy Policy}}
            \STATE Apply $a_t^k$ and observe $s_{t+1}^k$.
        \ENDFOR
            \STATE Update $\hat{r}_k,\hat{p}_k,n_{k}$ with all experience gathered in the episode. 
    \ENDFOR
\end{algorithmic}
\end{algorithm}

We provide the full proof of Theorem \ref{theorem: UCRL2 with greedy policies} which establishes a regret bound for UCRL2 with Greedy Policies (UCRL2-GP) in finite horizon MDPs. In the following, we present the structure of this section. 

We define the failure events for UCRL2-GP in Section \ref{supp: failure events ucrl2}. Most of the events are standard low-probability failure events, derived using, e.g., Hoeffding's inequality. We add to the standard set of events a failure event which holds when a sum of Decreasing Bounded Processes is large in its value. Using uniform bounds, the failure events are shown to hold jointly. When all failure events do not occur for all time-steps we say the algorithm is outside the failure event. In Section \ref{supp: optimism UCRL2} we establish that UCRL2-GP is optimistic, and, more specifically, that for all $s,t,k$ $\bar{V}_t^k(s)\geq V^*_t(s)$, outside the failure event. Lastly, in Section \ref{supp: UCRL2 detailed proof} we give the full proof of Theorem \ref{theorem: UCRL2 with greedy policies}, based on a new regret decomposition using on Lemma \ref{lemma: model base RL expected value difference}, the new results on Decreasing Bounded Processes (see Appendix \ref{supp: proofs on DBP}), and existing techniques (e.g., \citep{dann2017unifying,zanette2019tighter}).

\subsection{Failure Events for UCRL2 with Greedy Policies}\label{supp: failure events ucrl2}

Define the following failure events.
\begin{align*}
    &F^r_k=\brc*{\exists s,a:\ |r(s,a) - \hat{r}_{k-1}(s,a)| \geq \sqrt{\frac{2\ln \frac{2SAT}{\delta'}}{n_{k-1}(s,a)\vee 1}} }\\
    &F^p_k=\brc*{\exists s,a:\ \norm{p(\cdot\mid s,a)- \hat{p}_{k-1}(\cdot\mid s,a)}_1 \geq \sqrt{\frac{4S\ln\frac{3SAT}{\delta'}}{n_{k-1}(s,a)\vee 1
    }}}\\
    &F^N_k = \brc*{\exists s,a: n_{k-1}(s,a) \le \frac{1}{2} \sum_{j<k} w_j(s,a)-H\ln\frac{SAH}{\delta'}}.\\
    &F^{DBP} =\brc*{\exists k>0: \sum_{k=1}^K\sum_{t=1}^{H}\sum_{s} \bar{V}_t^{k-1}(s) -\E[\bar{V}_t^{k}(s)\mid \F_{k-1}] \geq 9SH^2\ln\frac{3SH}{\delta'}}
\end{align*}

Furthermore, the following relations hold.

\begin{itemize}
    \item Let $F^r=\bigcup_{k=1}^K F^r_k.$ Then $\Pr\brc*{F^r}\leq \delta'$, by Hoeffding's inequality, and using a union bound argument on all $s,a$, possible values of $n_{k}(s,a)$ and $k$. Furthermore, for $n(s,a)=0$ the bound holds trivially since $R\in[0,1]$. 
    \item Let $F^P=\bigcup_{k=1}^K F^{p}_k.$ Then $\Pr\brc*{ F^p}\leq \delta'$, holds by \citep{weissman2003inequalities} while applying union bound on all $s,a,n_{k-1}(s,a)$ and possible values of $k$ (e.g., \citealt{azar2017minimax,zanette2019tighter}). Furthermore, for $n(s,a)=0$ the bound holds trivially. 
    \item Let $F^N=\bigcup_{k=1}^K F^N_k.$ Then, $\Pr\brc*{F^N}\leq \delta'$. The proof is given in \citep{dann2017unifying} Corollary E.4 (and is used in \citealt{zanette2019tighter} Appendix D.4).
    \item  By construction of Algorithm \ref{alg: general model based RL algorithm}, $\forall s,t,\ \bar{V}^k_t(s)$ is a decreasing function of $k$, with $\bar{V}^0_t(s)=H$. Furthermore, since $\hat{r}_{k-1}(s,a)$ and $\hat{p}_{k-1}(\cdot\mid s,a)$ are non-negative, and $\bar{V}^0_t(s)>0$, a simple induction allows us to conclude that $\forall s,t,\ \bar{V}^k_t(s)\ge0$. Thus, by applying Lemma \ref{lemma: sum of decreasing processes}, $\Pr\brc*{F^{DBP}}\leq \delta'$.
\end{itemize}

\begin{lemma}\label{lemma: ucrl failure events}
Setting $\delta'=\frac{\delta}{4}$ then $\Pr\brc{F^r \bigcup F^p\bigcup F^N \bigcup F^{DBP}}\leq \delta$. When the failure events does not hold we say the algorithm is outside the failure event.
\end{lemma}

\subsection{UCRL2 with Greedy Policies is Optimistic}\label{supp: optimism UCRL2}

\begin{lemma}\label{lemma: optimism of UCRL2 with greedy policies}
Outside the failure event UCRL2-GP is Optimistic,
\begin{align*}
    \forall s,t,k\ \bar{V}^{k}_t(s)\geq V^*_t(s).
\end{align*}
\end{lemma}
\begin{proof}
We prove by induction. The base case holds by the initialization of the algorithm, $\bar{V}^0_t(s)=H-(t-1)\geq V^*_t(s)$. Assume the induction hypothesis holds for $k-1$ episodes. At the $k^{th}$ episode, states there were not visited at step $t$ will not be updated, and thus by the induction hypothesis, the result hold for these states. For states that were visited, if the minimum at the update stage equals to $\bar{V}^{k-1}_t(s)$, then the result similarly holds. Let $s_t^k$ be a state that was updated according to the optimistic model, and let
\begin{align*}
    &\tilde{a}^*\in \arg\max_a \tilde{r}_{k-1}(s^k_t,a) +\tilde{p}_{k-1}(\cdot\mid s^k_t,a)v^{k-1}_{t+1}\\
    &a^*\in \arg\max_a r(s^k_t,a) +p(\cdot\mid s^k_t,a)V^{*}_{t+1}.    
\end{align*}
Then,

\begin{align*}
    \bar{V}^k_t(s^k_t) &= \max_{a} \tilde{r}_{k-1}(s^k_t,a) +\tilde{p}_{k-1}(\cdot\mid s^k_t,a)\bar{V}^{k-1}_{t+1}\\
    &\stackrel{(1)}{=}\tilde{r}_{k-1}(s^k_t,\tilde{a}^*) +\tilde{p}_{k-1}(\cdot\mid s^k_t,\tilde{a}^*)\bar{V}^{k-1}_{t+1}\\
    &\stackrel{(2)}{\geq} \tilde{r}_{k-1}(s^k_t,a^*) +\tilde{p}_{k-1}(\cdot\mid s^k_t,a^*)\bar{V}^{k-1}_{t+1}\\
    &\stackrel{(3)}{\geq} r(s^k_t,a^*) +p(\cdot\mid s^k_t,a^*)\bar{V}^{k-1}_{t+1}\\
    &\stackrel{(4)}{\geq} r(s^k_t,a^*) +p(\cdot\mid s^k_t,a^*)V^{*}_{t+1} \\
    &\stackrel{(5)}{=} V^*_t(s^k_t).
\end{align*}
Relations $(1)$ and $(2)$ are by the definition and optimality of $\tilde{a}^*$, respectively. $(3)$ holds since outside failure event $F^r_k$, $\tilde{r}_{k-1}(s_t^k,a^*)\geq r(s_t^k,a^*)$. Furthermore, outside failure event $F^p_k$, the real transition probabilities $p(\cdot \mid s,a^*)\in CI(s_t^k,a^*)$, and thus
\begin{align*}
   \max_{P'\in CI(s_t^k,a^*)}P'(\cdot \mid s_t^k,a^*)\bar{V}_{t+1}^{k-1}=\tilde{p}_{k-1}(\cdot \mid s_t^k,a^*)\bar{V}_{t+1}^{k-1}\geq p_{k-1}(\cdot \mid s_t^k,a^*)\bar{V}_{t+1}^{k-1}.
\end{align*}

Finally, $(4)$ holds by the induction hypothesis $\forall s,a,t,\ \bar{V}^{k-1}_t(s)\geq V^{*}_t(s)$ and $(5)$ holds by the Bellman recursion.
\end{proof}

\subsection{Proof of Theorem \ref{theorem: UCRL2 with greedy policies}}\label{supp: UCRL2 detailed proof}
\TheoremGreedyPoliciesUCRL*

\begin{proof}
By the optimism of the value (Lemma \ref{lemma: optimism of UCRL2 with greedy policies}), we have that
\begin{align}
    \sum_{k=1}^K V^*_1(s_1^k) &- V^{\pi_k}_1(s_1^k) 
    \leq \sum_{k=1}^K \bar{V}^{k-1}_1(s_1^k) - V^{\pi_k}_1(s_1^k) \nonumber\\
    &\leq \underset{(A)}{\underbrace{\sum_{k=1}^K\sum_{t=1}^{H} \E[ \bar{V}_t^{k-1}(s_t^k) - \bar{V}_t^{k}(s_t^k)\mid \F_{k-1}]}} \nonumber\\
    &\quad+\underset{(B)}{\underbrace{\sum_{k=1}^K\sum_{t=1}^{H} \E \left[ (\tilde{r}_{k-1} - r)(s_t^k,a_t^k)+(\tilde{p}_{k-1} - p)(\cdot\mid s^k_t,a_t^k)^T \bar{V}_{t+1}^{k-1} \mid \F_{k-1}  \right]}}\enspace . \label{eq: UCRL2 term A and B}
\end{align}
The first relation is by the optimism of the value, and the second relation is by Lemma \ref{lemma: model base RL expected value difference}.
We now bound the two terms outside the failure event.

{\bf Bounding (A).} By Lemma \ref{lemma: regret to SH decreasing processes} (Appendix \ref{sec: supp general lemmas}),
\begin{align*}
    (A) = \sum_{k=1}^K\sum_{t=1}^{H}\sum_{s} \bar{V}_t^{k-1}(s) -\E[\bar{V}_t^{k}(s)\mid \F_{k-1}].
\end{align*}
Outside the failure event, the sum is bounded by $9SH^2\ln \frac{3SH}{\delta'}$ (see event $F^{DBP}$). Thus,
\begin{align*}
    (A)\leq \Olog(SH^2).
\end{align*}

{\bf Bounding (B).} Outside failure event $F^r_k$ the following inequality holds:
\begin{align}
    &\sum_{k=1}^K\sum_{t=1}^{H}\E \left[  (\tilde{r}_{k-1} - r)(s_t^k,\pi_k(s_t^k) \mid \F_{k-1}  \right] \nonumber\\
    &\lesssim  \sum_{k=1}^K\sum_{t=1}^{H} \E\brs*{\sqrt{\frac{1}{n_{k-1}(s_t^k,\pi_k(s_t^k))\vee 1}} \mid \F_{k-1}}\lesssim \Olog(\sqrt{SAT} +SAH), \label{eq: UCRL2 term B first term}
\end{align}
where the second inequality is by Lemma \ref{lemma: supp 1 over sqrt n sum}. 
It is worth noting that Lemma \ref{lemma: supp 1 over sqrt n sum} is proven by defining $L_k$, the set of 'good' state-action pairs, that contains pairs that were visited sufficiently often in the past \citep{dann2017unifying,zanette2019tighter}. The term we bound is then analyzed separately for state-action pairs inside and outside $L_k$. The definition of $L_k$ can be found in Definition \ref{defn:good set}, and its properties (including Lemma) are analyzed in Appendix \ref{sec: Lk definition}.

\clearpage

Furthermore, outside the failure event,

\begin{align}
    &\sum_{k=1}^K\sum_{t=1}^{H}\E \left[(\tilde{p}_{k-1} - p)(\cdot\mid s^k_t,a_t^k)^T \bar{V}_{t+1}^{k-1} \mid \F_{k-1}  \right]\nonumber\\
    &=\sum_{k=1}^K\sum_{t=1}^{H}\E \left[(\tilde{p}_{k-1} - \hat{p}_{k-1})(\cdot\mid s^k_t,a_t^k)^T \bar{V}_{t+1}^{k-1} \mid \F_{k-1}  \right]\nonumber\\
    &\qquad\qquad +\E \left[(\hat{p}_{k-1} - p)(\cdot\mid s^k_t,a_t^k)^T \bar{V}_{t+1}^{k-1} \mid \F_{k-1}  \right]\nonumber\\
    &\stackrel{(1)}{\leq}\sum_{k=1}^K\sum_{t=1}^{H}\E \left[\norm{(\tilde{p}_{k-1} - \hat{p}_{k-1})(\cdot\mid s^k_t,a_t^k)}_1 \norm{\bar{V}_{t+1}^{k-1}}_{\infty} \mid \F_{k-1}  \right]\nonumber\\
    &\qquad\quad +\E \left[\norm{(\hat{p}_{k-1} - p)(\cdot\mid s^k_t,a_t^k)} \norm{\bar{V}_{t+1}^{k-1}}_{\infty} \mid \F_{k-1}  \right]\nonumber\\
    &\stackrel{(2)}{\leq}  H\sum_{k=1}^K\sum_{t=1}^{H}\E \left[\norm{(\hat{p}_{k-1} - p)(\cdot\mid s^k_t,a_t^k)}_1 + \norm{(\tilde{p}_{k-1} - \hat{p}_{k-1})(\cdot\mid s^k_t,a_t^k)}_1  \mid \F_{k-1}  \right]\nonumber\\ 
    &\stackrel{(3)}{\lesssim}  H\sqrt{S}\sum_{k=1}^K\sum_{t=1}^{H}\E \left[\sqrt{\frac{1}{n_{k-1}(s_t^k,a_t^k)\vee 1}}  \mid \F_{k-1}  \right] \nonumber \\
    &\stackrel{(4)}{\lesssim} \Olog (HS\sqrt{AT} + H^2\sqrt{S}SA). \label{eq: UCRL2 term B second term} 
\end{align}

Relation $(1)$ holds by H\"older's inequality. Next, $(2)$ holds since $\forall s,t,k,\ 0\le\bar{V}^{k}_t(s)\leq H$. The lower bounds holds by Lemma \ref{lemma: optimism of UCRL2 with greedy policies} and since $V_t^*\ge0$. The upper bound is since the value can only decrease by Algorithm \ref{alg: general model based RL algorithm} and the inequality holds for the initialized value. Finally, $(3)$ holds outside failure event $F^p$ (Lemma \ref{lemma: ucrl failure events}), and $(4)$ holds by Lemma \ref{lemma: supp 1 over sqrt n sum}.

Combining \eqref{eq: UCRL2 term B first term}, \eqref{eq: UCRL2 term B second term} we conclude that,
\begin{align*}
    (B) \leq \Olog (HS\sqrt{AT} + H^2\sqrt{S}SA)
\end{align*}

Combining the bounds on (A) and (B) in \eqref{eq: UCRL2 term A and B} concludes the proof.
\end{proof}

\newpage
\section{Proof of Theorem \ref{theorem: EULER with greedy policies}}\label{supp: EULER}

\begin{algorithm}
\begin{algorithmic}[1]
\caption{EULER with Greedy Policies}
\label{alg supp: EULER with greedy policies}
    \STATE Initialize: $\delta,\delta'=\frac{\delta}{9},\forall s\in \mathcal{S}, t\in [H],\ \bar{V}^0_t(s)=H-(t-1),\underline{V}^0_t(s)=0,$
    \STATE  $\quad\quad\quad\quad \phi(s,a)=\sqrt{\frac{2\hat{\VAR}_{\hat{p}_{k-1}(s,a)}(\bar{V}^{k-1}_{t+1})\ln\frac{4SAT}{\delta'}}{n_{k-1}(s,a)}} + \frac{2H\ln\frac{4SAT}{\delta'}}{3n_{k-1}(s,a)}, L = 2\sqrt{\ln\frac{4SAT}{\delta'}},$
    \STATE  $\quad\quad\quad\quad J = \frac{2H\ln\frac{4SAT}{\delta'}}{3}, B_v = \sqrt{2\ln\frac{4SAT}{\delta'}},B_p=H\sqrt{2\ln\frac{4SAT}{\delta'}}.$
    \FOR{$k=1,2,..$}
        \STATE Initialize $s_1^k$
        \FOR{$t=1,..,H$}
            \STATE {\color{gray}{$\#$Update Upper Bound on $V^*$}}
            \FOR{$a\in \mathcal{A}$}
                \STATE $b^r_k(s_t^k,a) = \sqrt{\frac{2\hat{\VAR}(R(s_t^k,a))\ln\frac{4SAT}{\delta'}}{n_{k-1}(s_t^k,a)\vee 1}} + \frac{14\ln \frac{4SAT}{\delta'}}{3n_{k-1}(s_t^k,a)\vee 1} $
            	\STATE $b_k^{pv}(s_t^k,a) = \phi(\hat{p}_{k-1}(\cdot \mid s_t^k,a),\bar{V}^{k-1}_{t+1})  + \frac{4J+B_p}{n_{k-1}(s_t^k,a)\vee 1} + \frac{B_v \norm{ \bar{V}^{k-1}_{t+1} - \underline{V}^{k-1}_{t+1}  }_{2,\hat{p}}}{\sqrt{n_{k-1}(s_t^k,a)\vee 1}}$
            	 \STATE $\bar{Q}(s_t^k,a) = \hat{r}_{k-1}(s_t^k,a) + b_k^r(s_t^k,a) + \hat{p}_{k-1}(\cdot \mid s_t^k,a)^T \bar{V}^{k-1}_{t+1} + b_k^{pv}(s_t^k,a)$
            \ENDFOR
            \STATE $a_t^k\in \arg\max_a \bar{Q}(s_t^k,a)$
            \STATE $\bar{V}^{k}_{t}(s_t^k) =\min\brc*{\bar{V}^{k-1}_{t}(s_t^k),\bar{Q}(s_t^k,a_t^k)}$ \label{supp euler: bar v decreases}
            \STATE {\color{gray}{$\#$Update Lower Bound on $V^*$}}
           	 \STATE $b_k^{pv}(s_t^k,a_t^k) = \phi(\hat{p}_{k-1}(\cdot \mid s_t^k,a_t^k),\underline{V}^{k-1}_{t+1})  + \frac{4J+B_p}{n_{k-1}(s_t^k,a_t^k)\vee 1} + \frac{B_v \norm{ \bar{V}^{k-1}_{t+1} - \underline{V}^{k-1}_{t+1}  }_{2,\hat{p}}}{\sqrt{n_{k-1}(s_t^k,a_t^k)\vee 1}}$
    	 \STATE $\underline{Q}(s_t^k,a_t^k) =  \hat{r}_{k-1}(s_t^k,a_t^k) - b_k^r(s_t^k,a_t^k) + \hat{p}_{k-1}(\cdot \mid s_t^k,a)^T \underline{V}^{k-1}_{t+1} - b_k^{pv}(s_t^k,a_t^k)$
        \STATE $\underline{V}^{k}_t(s_t^k) = \max\brc*{\underline{V}^{k-1}_t(s_t^k),\underline{Q}(s_t^k,a_t^k)}$\label{supp euler: underline v increases}
            \STATE {\color{gray}{$\#$Act by the Greedy Policy}}
            \STATE Apply $a_t^k$ and observe $s_{t+1}^k$.
        \ENDFOR
            \STATE Update $\hat{r}_k,\hat{p}_k,n_{k}$ with all experience gathered in episode. 
    \ENDFOR
\end{algorithmic}
\end{algorithm}

\begin{remark}
Note that the algorithm does not explicitly define $\tilde{r}_{k-1}(s,a)$ and $\tilde{p}_{k-1}(\cdot\mid s,a)$. While we can directly set $\tilde{r}_{k-1}(s,a)=\hat{r}_{k-1}(s,a)+b^r_k(s_t^k,a)$, the optimistic transition kernel is only implicitly defined as 
\begin{align*}
    \tilde{p}_{k-1}(\cdot\mid s,a)^T\bar{V}_t^{k-1} = \hat{p}_{k-1}(\cdot\mid s,a)^T\bar{V}_{t+1}^{k-1} + b_k^{pv}(s,a)
\end{align*}
Nevertheless, throughout the proofs we are only interested in the above quantity, and thus, except for some abuse of notation, all of the proofs hold. We use this notation since it is common in previous works (e.g., \citealt{zanette2019tighter,dann2017unifying}) and for brevity.
\end{remark}

In this section, we provide the full proof of Theorem~\ref{theorem: EULER with greedy policies} which establishes a regret bound for EULER with Greedy Policies (EULER-GP). In \cite{zanette2019tighter} the authors prove their results using a general confidence interval, which they refer as \emph{admissible confidence interval}. In Section~\ref{sec:properties of confidence} we state there definition and state some useful properties. In Section~\ref{sec:EULER failure} we define the set of failure events and show that with high-probability the failure events do not occur. The set of failure events includes high-probability events derived using empirical Bernstein inequalities \citep{maurer2009empirical}, as well as high probability events on Decreasing Bounded Process, as we establish in Appendix~\ref{supp: proofs on DBP}. In Section~\ref{sec: supp optimism euler} we analyze the optimism EULER-GP and prove it satisfies the same optimism and pessimism as in \citet{zanette2019tighter}, outside the failure event for all $s,t,k$ $\underline{V}_t^k(s)\leq  V^*_t(s) \leq \bar{V}_t^k(s)$.

In Section \ref{supp: EULER full proof}, using the above, we give the full proof of Theorem \ref{theorem: EULER with greedy policies}. As for the proof of UCRL2-GP, we apply the new suggested regret decomposition, based on Lemma \ref{lemma: model base RL expected value difference}, and use the new results on Decreasing Bounded Processes. In section \ref{sec:helpful lemma proof} we modify some results of \citep{zanette2019tighter} to our setting, and utilize the new results to bound each term in the regret decomposition in section \ref{sec:euler regret decomp bounds}.

\subsection{Properties of Confidence Intervals} \label{sec:properties of confidence}

In this section, we cite the important properties of the confidence intervals of EULER, as was stated in \citep{zanette2019tighter}. We start from their definition of an admissible confidence interval:

\begin{definition}
\label{defn:admissible}
A confidence interval $\phi$ is called admissible for EULER if the following properties hold: 
\begin{enumerate}
\item $\phi(p,V)$ takes the following functional form:
    \begin{align*}
    & \phi(p,V) = \frac{g(p,V)}{\sqrt{n_{k-1}(s,a) \vee 1}} + \frac{j(p,V)}{n_{k-1}(s,a) \vee 1} \enspace,
    \end{align*}
    for some functions $j(p,V) \leq J \in \R$, and 
    \begin{align*}
    \abs*{g(p,V_1) - g(p,  V_2) } \leq B_v \norm*{ V_1 - V_2}_{2,p} 
    \end{align*}
    If the value function is uniform then:
    \begin{equation*}
        g(p,\alpha \ind) = 0, \quad \forall \alpha \in \R.
    \end{equation*} \label{def-prop:admiss definition}
\item With probability at least $1-\delta'$ it holds that:
    \begin{equation*}
    \abs{\br*{\hat p_{k-1}(\cdot\mid s,a) - p(\cdot\mid s,a)}^TV^*_{t+1} } \leq \phi(p(\cdot\mid s,a),V^*_{t+1})
    \end{equation*}
    jointly for all timesteps $t$, episodes $k$, states $s$ and actions $a$. \label{def-prop:admiss phatV concentrated}
\item With probability at least $1-\delta'$ it holds that:
    \begin{align*}
    \abs*{g(\hat{p}_{k-1}(\cdot\mid s,a), V^*_{t+1}) - g(p(\cdot\mid s,a), V^*_{t+1})  } & \leq \frac{B_p}{\sqrt{n_{k-1}(s,a) \vee 1}}
    \end{align*}
    jointly for all episodes $k$, timesteps $t$, states $s$, actions $a$ and some constant $B_p$ that does not depend on $\sqrt{n_{k-1}(s,a) \vee 1}$. \label{def-prop:admiss phat minus p}
\end{enumerate}

\end{definition}

An admissible confidence interval enjoys many properties, which are summarized in the following lemma:
\begin{lemma}
\label{lemma:admiss properties}
If $\phi$ is admissible for EULER, and under the events that properties \ref{def-prop:admiss phatV concentrated},\ref{def-prop:admiss phat minus p} of Definition \ref{defn:admissible} hold, then:
\begin{enumerate}
    \item For any $V\in\R^S$ with $\norm{V}_\infty\leq H$, it holds that $\abs*{g(p,V)}\le B_vH$. \label{lemma-part:g bound}
    \item For any $V\in\R^S$,
    \begin{align*}
        \abs{\phi(\hat{p}_{k-1}(\cdot\mid s,a),V)-\phi(p(\cdot\mid s,a),V^*_{t+1})} \le \frac{B_v\norm{V-V^*_{t+1}}_{2,\hat p}}{\sqrt{n_{k-1}(s,a) \vee 1}} + \frac{B_p+4J}{n_{k-1}(s,a) \vee 1}
    \end{align*}
    \label{lemma-part:p change}
    \item Let $b_k^{pv}$ be the transition bonus, which is defined as
    \begin{align*}
        b_k^{pv}\br*{\hat{p}_{k-1}(\cdot\mid s,a),V_1,V_2}
        = \phi(\hat{p}_{k-1}(\cdot\mid s,a),V_1) + \frac{B_p+4J}{n_{k-1}(s,a) \vee 1} + \frac{B_v\norm{V_2-V_1}_{2,\hat p}}{\sqrt{n_{k-1}(s,a) \vee 1}} \enspace.
    \end{align*}
    For any $V_1,V_2\in\R^S$ such that $V_1\le V^*\le V_2$ pointwise, it holds that
    \begin{align*}
        b_k^{pv}\br*{\hat{p}_{k-1}(\cdot\mid s,a),V_1,V_2}
        \ge \phi(p(\cdot\mid s,a),V^*)\\
        b_k^{pv}\br*{\hat{p}_{k-1}(\cdot\mid s,a),V_2,V_1}
        \ge \phi(p(\cdot\mid s,a),V^*)
    \end{align*} \label{lemma-part:bonus phi bound}
\end{enumerate}

\end{lemma}

\begin{proof}
The first property is due to Corollary 1.3 of \citep{zanette2019tighter}, and the second one is Lemma 4 of their paper. The third property is equivalent to Proposition 3 in \citep{zanette2019tighter}, but since we allow general value functions $V_1,V_2$, we write the full proof for completeness.

we start by proving that if $V_1\le V^*\le V_2$, then for any transition probability vector $p$, 
\begin{align} 
    \norm{V_2-V^*}_{2,p}\le \norm{V_2-V_1}_{2,p} \nonumber \\
    \norm{V_1-V^*}_{2,p}\le \norm{V_2-V_1}_{2,p} \label{eq:value norm bound}
\end{align}
To this end, notice that $\forall s$
\begin{align*}
    0\le V_2(s)-V^*(s) \le V_2(s)-V_1(s) \enspace ,
\end{align*}
and since all of the quantities are non-negative, it also holds that  
\begin{align*}
    0\le \br*{V_2(s)-V^*(s)}^2 \le \br*{V_2(s)-V_1(s)}^2 \enspace .
\end{align*}

The inequality holds pointwise, and therefore holds for any linear combination with non-negative constants:
\begin{align*}
    0\le \sum_s p(s)\br*{V_2(s)-V^*(s)}^2 \le \sum_s p(s)\br*{V_2(s)-V_1(s)}^2 \enspace .
\end{align*}
Taking the root of this inequality yields Inequality (\ref{eq:value norm bound}). Substituting in the definition of $b_k^{pv}\br*{\hat{p}_{k-1}(\cdot\mid s,a),V_2,V_1}$ yields:

\begin{align*}
        b_k^{pv}\br*{\hat{p}_{k-1}(\cdot\mid s,a),V_2,V_1}
        &= \phi(\hat{p}_{k-1}(\cdot\mid s,a),V_2) + \frac{B_p+4J}{n_{k-1}(s,a) \vee 1} + \frac{B_v\norm{V_2-V_1}_{2,\hat p}}{\sqrt{n_{k-1}(s,a) \vee 1}}\\
        & \ge \phi(\hat{p}_{k-1}(\cdot\mid s,a),V_2) + \frac{B_p+4J}{n_{k-1}(s,a) \vee 1} + \frac{B_v\norm{V_2-V^*}_{2,\hat p}}{\sqrt{n_{k-1}(s,a) \vee 1}} \\
        & \ge \phi(p(\cdot\mid s,a),V^*)
\end{align*}

The first inequality is due to (\ref{eq:value norm bound}), and the second is due to the second part of the Lemma. The result for $b_k^{pv}\br*{\hat{p}_{k-1}(\cdot\mid s,a),V_1,V_2}$ can be proven similarly, and thus omitted.

\end{proof}

Another property that will be useful throughout the proof is the following upper bound on $b_k^{pv}\br*{\hat{p}_{k-1}(\cdot\mid s,a),V_1,V_2}$
\begin{lemma} \label{lemma: bkv bound}
For any $V_1,V_2$ such that for all $s$, $V_1(s),V_2(s)\in[0,H]$ 
\begin{align*}
    b_k^{pv}\br*{\hat{p}_{k-1}(\cdot\mid s,a),V_2,V_1} \leq \frac{2B_vH+5J+B_p}{\sqrt{n_{k-1}(s,a) \vee 1}},
\end{align*}
\end{lemma}

\begin{proof}
We bound $b_k^{pv}\br*{\hat{p}_{k-1}(\cdot\mid s,a),V_2,V_1}$ as follows:

\begin{align*}
    b_k^{pv}&\br*{\hat{p}_{k-1}(\cdot\mid s,a),V_2,V_1}
         = \phi(\hat{p}_{k-1}(\cdot\mid s,a),V_2) + \frac{B_p+4J}{n_{k-1}(s,a) \vee 1} + \frac{B_v\norm{V_2-V_1}_{2,\hat p}}{\sqrt{n_{k-1}(s,a) \vee 1}}\\
        & \stackrel{(1)}{\le}  \frac{g(p,V)}{\sqrt{n_{k-1}(s,a) \vee 1}} + \frac{j(p,V)}{n_{k-1}(s,a) \vee 1} + \frac{B_p+4J}{n_{k-1}(s,a) \vee 1} + \frac{B_vH}{\sqrt{n_{k-1}(s,a) \vee 1}}\\
        & \stackrel{(2)}{\le}  \frac{B_vH}{\sqrt{n_{k-1}(s,a) \vee 1}} + \frac{J}{n_{k-1}(s,a) \vee 1} + \frac{B_p+4J}{n_{k-1}(s,a) \vee 1} + \frac{B_vH}{\sqrt{n_{k-1}(s,a) \vee 1}}\\
        & \stackrel{(3)}{\le}  \frac{2B_vH+5J+B_p}{\sqrt{n_{k-1}(s,a) \vee 1}}
\end{align*}

In $(1)$, we substituted $\phi$ and bounded $\norm{V_2-V_1}_{2,\hat p}\le H$. $(2)$ is by Lemma \ref{lemma:admiss properties} and Definition \ref{defn:admissible}, and $(3)$ is by noting that  $n\ge \sqrt{n}$ for $n\ge1$.
\end{proof}

\clearpage

We end this section by stating that Bernstein's inequality induces an admissible $\phi$. The proof can be found in \citep{zanette2019tighter}, Proposition 2. 

\begin{lemma} \label{lemma: bernstein admissible}
Bernstein inequality induces an admissible confidence interval with $g(p,V) = \sqrt{2\VAR_{s'\sim p(\cdot\mid s,a)}V\ln\frac{2SAT}{\delta'}}$and $j(p,V)=2H\ln\frac{2SAT}{\delta'}$, or explicitly:
\begin{align*}
    \phi\br*{p(\cdot\mid s,a), V} = \sqrt{\frac{2\VAR_{s'\sim p(\cdot\mid s,a)}V\ln\frac{2SAT}{\delta'}}{n_{k-1}(s,a) \vee 1}} + \frac{2H\ln\frac{2SAT}{\delta'}}{3n_{k-1}(s,a) \vee 1}
\end{align*}
with the constants $J = \frac{2H\ln\frac{2SAT}{\delta'}}{3}=\Olog(H), B_v = \sqrt{2\ln\frac{2SAT}{\delta'}}=\Olog(1)$ and $B_p=H\sqrt{2\ln\frac{2SAT}{\delta'}}=\Olog(H)$. Using lemma \ref{lemma: bkv bound}, it also implies that for any $V_1,V_2$ such that for all $s$, $V_1(s),V_2(s)\in[0,H]$, it holds that  $b_k^{pv}\br*{\hat{p}_{k-1}(\cdot\mid s,a),V_2,V_1} \lesssim \Olog(H)$.
\end{lemma}

\subsection{Failure Events} \label{sec:EULER failure}

\subsubsection{Failure Events of EULER}
\label{supp: euler failure}
We start by recalling the failure events as stated in \citep{zanette2019tighter}, Appendix D. These events are high probability bounds that are based on the Empirical Bernstein Inequality \citep{maurer2009empirical} and leads to the bonus terms of the algorithm. Importantly, these events depend on the state-action visitation counter, and, thus, are indifferent to the greedy exploration scheme which we consider.

Define the following failure events.
{\small
\begin{align*}
    &F^r=\brc*{\exists s,a,k:\ \abs*{r(s,a) - \hat{r}_{k-1}(s,a)} \geq \sqrt{\frac{2\hat{\VAR}_{k-1}R(s,a)\ln \frac{4SAT}{\delta'}}{n_{k-1}(s,a)\vee 1}} + \frac{14\ln\frac{4SAT}{\delta'}}{3(n_{k-1}(s,a)\vee 1)} }\\
    &F^{vr}=\brc*{\exists s,a,k:\ \abs*{\sqrt{\hat{\VAR}_{k-1}\ R(s,a)} - \sqrt{\VAR\ R(s,a)}} \geq \sqrt{\frac{4\ln \frac{2SAT}{\delta'}}{n_{k-1}(s,a)\vee 1}}  }\\
    &F^{pv}=\brc*{\exists s,a,t,k:\ \abs*{\br*{\hat{p}_{k-1}(\cdot \mid s,a)-p(\cdot \mid s,a)}^T V_{t+1}^*} \geq \sqrt{\frac{2\VAR_{s'\sim p(\cdot\mid s,a)} V^*_{t+1}\ln \frac{4SAT}{\delta'}}{n_{k-1}(s,a)\vee 1}} + \frac{2H\ln\frac{2SAT}{\delta'}}{3(n_{k-1}(s,a)\vee 1)}}\\
    &F^{pv2}=\brc*{\exists s,a,t,k:\     \abs*{\norm{V_{t}^*}_{2,\hat{p}} - \norm{V_{t}^*}_{2,p} } \geq H\sqrt{\frac{4\ln \frac{2SAT}{\delta'}}{n_{k-1}(s,a)\vee 1}}}\\
    &F^{ps}=\brc*{\exists s,s',a,k:\ |\hat{p}_{k-1}(s'\mid s,a) - p_{k-1}(s'\mid s,a)   | \geq \sqrt{\frac{p(s'\mid s,a)(1-p(s'\mid s,a))\ln \frac{2TS^2A}{\delta'}}{n_{k-1}(s,a)\vee 1}} + \frac{2\ln \frac{2TS^2A}{\delta'}}{3(n_{k-1}(s,a)\vee 1)}}\\
    &F^{pn1}=\brc*{\exists s,a,k:\         \norm{\hat{p}_{k-1}(\cdot\mid s,a) - p(\cdot\mid s,a)  }_1 \geq \sqrt{\frac{4S \ln \frac{3SAT}{\delta'}}{n_{k-1}(s,a)\vee 1}}}\\
    &F^N_k = \brc*{\exists s,a,k: n_{k-1}(s,a) \le \frac{1}{2} \sum_{j<k} w_j(s,a)-H\ln\frac{SAH}{\delta'}}.\\
\end{align*}
}
where $w_j(s,a)\eqdef \sum_{t=1}^H w_{tj}(s,a)$. In \citep{zanette2019tighter}, Appendix D, it is shown these events hold individually with probability at most $\delta'$.

\clearpage
\subsubsection{Failure Events of Decreasing Bounded Processes} \label{sec:EULER failure decreasing}
In this section, we add another failure events to the total set of failure events. This set of failure event is not present in previous analysis of regret in optimistic RL algorithms (e.g., in \citealt{azar2017minimax,dann2017unifying,dann2018policy,zanette2019tighter}).

We define the following failure events.
{\small
\begin{align*}
    &F^{vDP} = \brc*{\exists K\geq 0:     \sum_{k=1}^{K}\sum_{t=1}^{H}\sum_s \bar{V}_t^{k-1}(s) - \E[\bar{V}_t^{k}(s)\mid \mathcal{F}_{k-1}] \geq 9SH^2\ln\frac{3SH}{\delta'} }\\
    &F^{vsDP} = \brc*{\exists K\geq 0: \sum_{k=1}^{K}\sum_{t=1}^{H}\sum_{s} (\bar{V}_{t}^{k-1}(s)  - \underline{V}_t^{k-1}(s))^2 - \E[(\bar{V}_t^{k}(s) - \underline{V}_t^{k}(s))^2\mid \mathcal{F}_{k-1}] \geq 9SH^3\ln\frac{3SH}{\delta'}}
\end{align*}
}

In this section, we prove that both of these failure events occur with low probability $\delta'$.

We start by proving that $\brc*{\bar{V}^k_t(s)}$ is a decreasing processes, independently to the previously defined failure events. We continue and prove that $ \brc*{\bar{V}^k_t(s) - \underline{V}^k_t(s)}^2$ starts as a decreasing process and then becomes and increasing process.

\begin{lemma}\label{lemma: 3 properties of decreasing process of EULER} The following claims hold.
\begin{enumerate}
    \item For every $s,t$, $\brc*{\bar{V}^k_t(s)}_k$ is a decreasing process and is bounded by $[0,H-(t-1)]$.
    \item For every $s,t$, $\brc*{\underline{V}^k_t(s)}_k$ is an increasing process and is bounded by $[0,H-(t-1)]$.
    \item For every $s,t$, $\brc*{\br*{\bar{V}^k_t(s) - \underline{V}^k_t(s)}^2}_k$ starts as a decreasing process bounded by ${[0,(H-(t-1)^2)]}$ and then, possibly, becomes an increasing process. 
\end{enumerate}
\end{lemma}

\begin{proof}
We start by proving the first claim.The following holds. By the initialization of the algorithm $\forall s,t,\ \bar{V}^0_t(s) = H-(t-1)$. By construction of the update rule $\bar{V}^k_t(s)$ can only decrease (see Line~\ref{supp euler: bar v decreases}). 

We now prove that for every $s,t,k,\ \brc*{\bar{V}^k_t(s)}_k$ is bounded from below by $0$. By assumption $r(s,a)\in [0,1]$, and thus $\hat{r}_{k-1}(s,a)\geq 0$ a.s. . By induction, this implies $\bar{V}^{k-1}_{t}\geq 0$. The base case holds by initialization, and the induction step by the fact $\hat{r}_{k-1}\geq 0$ and that the bonus terms are positive.

Proving the second claim is done with similar argument, while using $\hat{r}_{k-1}(s,a)\leq 1$ a.s.. By the update rule (see Line \ref{supp euler: underline v increases}), $\brc*{\underline{V}^k_t(s)}_k$ is an Increasing Bounded Process in $[0,H-(t-1)]$ (similar definition as in \ref{defn: DBP} with opposite inequality). 

To prove the third claim we combine the two claims. Thus, $\brc*{\br*{\bar{V}^k_t(s) - \underline{V}^k_t(s)}^2}_k$ starts as a decreasing process. Then, if the upper and lower value function crosses one another, the process becomes an increasing process. 
\end{proof}
\begin{remark}
Notice that the upper bound and lower bound of the optimal value crosses one another only inside the failure events defined in Section \ref{supp: euler failure}). Yet, the analysis in the following will be indifferent to whether the failure event takes place or not.
\end{remark}

\begin{lemma}\label{lemma: decreasing process value for EULER}
$\Pr\brc*{F^{vDP}} \leq \delta'$.
\end{lemma}
\begin{proof}
We wish to bound
\begin{align*}
    \Pr\brc{\exists K\geq 0:     \sum_{k=1}^{K}\sum_{t=1}^{H}\sum_s \bar{V}_h^{k-1}(s) - \E[\bar{V}_h^{k}(s)\mid \mathcal{F}_{k-1}] \geq 9SH^2\ln\frac{3SH}{\delta'}}.
\end{align*}
According to Lemma \ref{lemma: 3 properties of decreasing process of EULER}, for every $s,t,$  $\brc*{\bar{V}_t^{k}(s)}_{k\geq 1}$ is a decreasing process. Applying Lemma \ref{lemma: sum of decreasing processes} (Appendix \ref{supp: proofs on DBP}) which bounds the sum of Decreasing Bounded Processes we conclude the proof.
\end{proof}

\begin{lemma}\label{lemma: euler with local models square decreasing process}
$\Pr\brc*{F^{vsDP}} \leq \delta'$.
\end{lemma}
\begin{proof}
We wish to bound
{\small
\begin{align*}
    \Pr\brc*{\exists K\geq 0: \sum_{k=1}^{K}\sum_{t=1}^{H}\sum_{s} (\bar{V}_{t}^{k-1}(s)  - \underline{V}_t^{k-1}(s))^2 - \E[(\bar{V}_t^{k}(s) - \underline{V}_t^{k}(s))^2\mid \mathcal{F}_{k-1}] \geq 9SH^3\ln\frac{3SH}{\delta'}}.
\end{align*}
}
Consider a fixed $s,t$. Furthermore, define the following event
\begin{align*}
    \mathbb{A}_{k-1} = \brc*{\bar{V}_{t}^{k-1}(s)>   \underline{V}_t^{k-1}(s)}.
\end{align*}
We have that
\begin{align*}
    &\sum_{k=1}^{K} (\bar{V}_{t}^{k-1}(s)  - \underline{V}_t^{k-1}(s))^2 - \E[(\bar{V}_t^{k}(s) - \underline{V}_t^{k}(s))^2\mid \mathcal{F}_{k-1}]\\
    & \leq \sum_{k=1}^{K} \br*{(\bar{V}_{t}^{k-1}(s)  - \underline{V}_t^{k-1}(s))^2 - \E[(\bar{V}_t^{k}(s) - \underline{V}_t^{k}(s))^2\mid \mathcal{F}_{k-1}]} \ind\brc*{\mathbb{A}_{k-1}}\\
    & = \sum_{k=1}^{K} (\bar{V}_{t}^{k-1}(s)  - \underline{V}_t^{k-1}(s))^2 \ind\brc*{\mathbb{A}_{k-1}}- \E[(\bar{V}_t^{k}(s) - \underline{V}_t^{k}(s))^2 \ind\brc*{\mathbb{A}_{k-1}}\mid \mathcal{F}_{k-1}] \\
    & = \sum_{k=1}^{K} (\bar{V}_{t}^{k-1}(s)  - \underline{V}_t^{k-1}(s))^2 \ind\brc*{\mathbb{A}_{k-1}}- \E[(\bar{V}_t^{k}(s) - \underline{V}_t^{k}(s))^2 \ind\brc*{\mathbb{A}_{k}}\mid \mathcal{F}_{k-1}]\\
    & \quad - \sum_{k=1}^{K} \E[(\bar{V}_t^{k}(s) - \underline{V}_t^{k}(s))^2 \br*{\ind\brc*{\mathbb{A}_{k-1}} - \ind\brc*{\mathbb{A}_{k}}}\mid \mathcal{F}_{k-1}]\\
    & \leq \sum_{k=1}^{K} (\bar{V}_{t}^{k-1}(s)  - \underline{V}_t^{k-1}(s))^2 \ind\brc*{\mathbb{A}_{k-1}}- \E[(\bar{V}_t^{k}(s) - \underline{V}_t^{k}(s))^2 \ind\brc*{\mathbb{A}_{k}}\mid \mathcal{F}_{k-1}].
\end{align*}

The first relation holds by definition, if the event $\mathbb{A}_{k-1}$ is false then the term is negative, since the process becomes increasing, and only decreases the sum. The second relation holds since $\ind\brc*{\mathbb{A}_{k-1}}$ is $\mathcal{F}_{K-1}$ measureable. The forth relation holds since $(\bar{V}_{t}^{k-1}(s)  - \underline{V}_t^{k-1}(s))^2\geq 0$ and $\br*{\ind\brc*{\mathbb{A}_{k-1}} - \ind\brc*{\mathbb{A}_{k}}} \geq 0$. Where the latter holds since $\ind\brc*{\mathbb{A}_{k}} = 1 \rightarrow \ind\brc*{\mathbb{A}_{k-1}} = 1$, i.e.,
\begin{align*}
    \bar{V}_{t}^{k}(s) > \underline{V}_t^{k}(s)\
    \rightarrow  \bar{V}_{t}^{k-1}(s) > \underline{V}_t^{k-1}(s).
\end{align*}
Differently put, if at the $k^{th}$ episode $\bar{V}_{t}^{k}(s)>   \underline{V}_t^{k}(s)$ then it also holds for the $k-1^{th}$ episode, $\bar{V}_{t}^{k-1}(s)>   \underline{V}_t^{k}(s-1)$, as the process $\brc*{\bar{V}_t^{k}(s)}_{k\geq 0}$ is increasing and $\brc*{\underline{V}_t^{k}(s)}_{k\geq 0}$ is decreasing by Lemma~\ref{lemma: 3 properties of decreasing process of EULER}.

Furthermore, by Lemma \ref{lemma: 3 properties of decreasing process of EULER}, $\brc*{(\bar{V}_{t}^{k}(s)  - \underline{V}_t^{k}(s))^2 \ind\brc*{\mathbb{A}_{k}}}_k$ is a Decreasing Bounded Process in $[0,H^2]$. Initially, it decreases since  $\ind\brc*{\mathbb{A}_{k}}_k=1$ and $\brc*{(\bar{V}_{t}^{k}(s)  - \underline{V}_t^{k}(s))^2}$ is initially decreasing. Furthermore, when $\ind\brc*{\mathbb{A}_{k}}=0$ it cannot increase. Lastly, $(\bar{V}_{t}^{0}(s)  - \underline{V}_t^{0}(s))^2 \ind\brc*{\mathbb{A}_{0}} \leq H^2$.

Applying Theorem \ref{theorem: regret of decreasing process} we get that for a fixed $s,t$, with probability $\frac{\delta'}{SH}$
\begin{align*}
    \sum_{k=1}^{K} (\bar{V}_{t}^{k-1}(s)  - \underline{V}_t^{k}(s))^2 &- \E[(\bar{V}_t^{k}(s) - \underline{V}_t^{k-1}(s))^2\mid \mathcal{F}_{k-1}] \geq9SH^3\ln\frac{3SH}{\delta'}.
\end{align*}
By applying Lemma \ref{lemma: sum of decreasing processes} (Appendix \ref{supp: proofs on DBP}), which extends this bound to the sum on $s,t$ we conclude the proof.
\end{proof}

\begin{lemma}\label{lemma: all failure events EULER}
(All Failure Events) If $\delta'= \frac{\delta}{9}$, then $$F \eqdef F^r  \bigcup F^{vr}  \bigcup F^{pr}   \bigcup F^{pv}    \bigcup F^{pv2}   \bigcup F^{ps}   \bigcup F^{pn1}   \bigcup F^{vDP}   \bigcup F^{vsDP} $$
holds with probability at most $\delta$. If the event $F$ does not hold we say the algorithm is outside the failure event.
\end{lemma}
\begin{proof}
    Applying a union bound on all events, which hold individually with probability at most $\delta'$ yield the result.
\end{proof}

\subsection{EULER with Greedy Policies is Optimistic}\label{sec: supp optimism euler}

Our algorithm modifies the exploration bonus of \citep{zanette2019tighter} by using $\bar{V}_{k-1},\underline{V}_{k-1}$ instead of $\bar{V}_{k},\underline{V}_{k}$, and uses the following bonus (with some abuse of notation):
\begin{align*}
    b_k^{pv}(s,a) = b_k^{pv}\br*{\hat{p}_{k-1}(\cdot\mid s,a), \bar{V}_{k-1},\underline{V}_{k-1}}\enspace .
\end{align*}

We now show that the modified bonus retains the optimism of the algorithm:
\begin{lemma}\label{lemma: optimism EULER}
    Outside the failure event of the estimation (see Lemma \ref{lemma: all failure events EULER}), if the confidence interval is admissible, then the relation
    \begin{align*}
        \underline{V}_t^{k-1}\leq V_t^*\leq \bar{V}^{k-1}_{t}
    \end{align*}
    holds pointwise for all timesteps $t$ and episodes $k$.
\end{lemma}
\begin{proof}
We follow the proof of \citep{zanette2019tighter}, Proposition 4, and prove by induction. We first prove that for all $k$, $V_t^*\leq \bar{V}^{k}_{t}$.

The claim trivially holds for $k=0$, due to the initialization of the value. Suppose that the result holds for any state $s$ and timestep $t$ in the $k-1^{\textrm{th}}$ episode. If 
\begin{align*}
	\hat{r}_{k-1}(s_t^k,a_t^k) + b_{k-1}^r(s_t^k,a_t^k) &+ \hat{p}_{k-1}(\cdot\mid s,a_t^k)^T \bar{V}^{k-1}_{t+1} + b_{k-1}^{pv}(s_t^k,a_t^k) 
	\ge \bar{V}_t^{k-1}(s_t)\enspace,
\end{align*}
then by the induction's assumption we are done. Otherwise, denote the optimal action in the real MDP at state $s_t^k$ by $a_t^*$. The value is updated as follows:

\begin{align*}
	\bar{V}_k(s_t^k)&=\hat{r}_{k-1}(s_t^k,a_t^k) + b_{k-1}^r(s_t^k,a_t^k) + \hat{p}_{k-1}(\cdot\mid s,a_t^k)^T \bar{V}^{k-1}_{t+1} + b_{k-1}^{pv}(s_t^k,a_t^k) \\
	&\ge \hat{r}_{k-1}(s_t^ka_t^*) + b_{k-1}^r(s_t^k,a_t^*) + \hat{p}_{k-1}(\cdot\mid s,a_t^*)^T \bar{V}^{k-1}_{t+1} + b_{k-1}^{pv}(s_t^k,a_t^*) \\
	&\ge r(s_t^k,a_t^*) + \hat{p}_{k-1}(\cdot\mid s,a_t^*)^T \bar{V}^{k-1}_{t+1} + b_{k-1}^{pv}(s_t^k,a_t^*)
\end{align*}

The first inequality is since $a_t^k$ is the action that maximizes the greedy value and the second inequality is due to the optimism of the reward when the reward bonus is added, outside the failure events (Lemma \ref{lemma: all failure events EULER}). Next, using the inductive hypothesis ($V^*_{t+1}\le \bar{V}^{k-1}_{t+1}$ element-wise), we get 

\begin{align*}
	\bar{V}^k_t(s_t^k)
	\ge r(s_t^k,a_t^*) + \hat{p}_{k-1}(\cdot\mid s,a_t^*)^T V^*_{t+1} + b_{k-1}^{pv}(s_t^k,a_t^*)
\end{align*}

We now apply Lemma \ref{lemma:admiss properties}, which implies that 
\begin{align*}
b_{k-1}^{pv}(s_t^k,a_t^*)\ge \phi(p(\cdot\mid s_t^k,a_t^*),V^*) \enspace,
\end{align*}
and thus 
\begin{align*}
	\bar{V}^k_t(s_t^k)
	\ge r(s_t^k,a_t^*) + \hat{p}_{k-1}(\cdot\mid s,a_t^*)^T V^*_{t+1} + \phi(p(\cdot\mid s_t^k,a_t^*),v^*)
\end{align*}

Finally, since $\phi$ is admissible, we get the desired result from property (\ref{def-prop:admiss phatV concentrated}) of Definition \ref{defn:admissible}:

\begin{align*}
	\bar{V}_t^k(s_t^k)
	\ge r(s_t^k,a_t^*) + p(\cdot\mid s,a_t^*)^T V^*_{t+1} = V^*_{t+1}(s_t^k)
\end{align*}

The proof for $\underline{V}_t^{k-1}\le V^*_t$ is almost identical, and thus omitted from this paper.
\end{proof}

\clearpage

\subsection{Proof of Theorem \ref{theorem: EULER with greedy policies}}\label{supp: EULER full proof}

\begin{proof}

Throughout the proof, we assume that we are outside the failure events that were defined in Section \ref{sec:EULER failure}, which happens with probability of at least $1-\delta$ (Lemma \ref{lemma: all failure events EULER}). Specifically, it implies that the value function is optimistic, namely $V_1^*(s) \le V_1^k(s)$ (Lemma \ref{lemma: optimism EULER}), and we can bound the regret by,
\begin{align*}
    \mathrm{Regret}(K) &= \sum_{k=1}^K V_1^*(s_1^k) - V_1^{\pi_k}(s_1^k) \leq \sum_{k=1}^K \bar{V}_1^{k-1}(s_1^k) - V_1^{\pi_k}(s_1^k).
\end{align*}
Next, by applying Lemma \ref{lemma: model base RL expected value difference} the following bound holds,
\begin{align}
    &\leq \underset{(A)}{\underbrace{\sum_{k=1}^K\sum_{t=1}^{H} \E[ \bar{V}_t^{k-1}(s_t^k) - \bar{V}_t^{k}(s_t^k)\mid \F_{k-1}]}} \nonumber\\
    &+\underset{(B)}{\underbrace{\sum_{k=1}^K\sum_{t=1}^{H} \E \brs*{ (\tilde{r}_{k-1} - r)(s_t^k,a_t^k)+(\tilde{p}_{k-1} - p)(\cdot\mid s^k_t,a_t^k)^T \bar{V}_{t+1}^{k-1} \mid \F_{k-1}  }}}. 
\end{align}

The regret is thus upper bounded by two terms. The first term $(A)$ also appears in the analysis of RTDP (Theorem \ref{theorem: regret rtdp}). Specifically, by Lemma \ref{lemma: regret to SH decreasing processes}  (Appendix \ref{sec: supp general lemmas}), we can express this term as a sum of $SH$ Decreasing Bounded Process in $[0,H]$:

\begin{align*}
    (A) = \sum_s \sum_{t=1}^H\sum_{k=1}^K \bar{V}^{k-1}_t(s)- \E[\bar{V}^{k}_t(s) \mid \mathcal{F}_{k-1}].
\end{align*}

{\bf Bounding (A).} Outside failure event $F^{vDP}$, this term is bounded by $9SH^2\ln \frac{3SH}{\delta'}$. Thus,
\begin{align*}
    (A) \lesssim \Olog(SH^2)
\end{align*}

{\bf Bounding (B).}  The term (B) is almost the same term that is bounded in \citep{zanette2019tighter}, and its presence is common in recent literature on exploration in RL (e.g., \citealt{dann2017unifying,dann2018policy,zanette2019tighter}). The only difference between (B) and the term bounded in \citep{zanette2019tighter} is the presence of $\bar{V}^{k-1}$, the value  \emph{before} the update, instead of $\bar{V}^k$, the value after applying the update rule. This is since existing algorithms perform planning from the end of an episode and backwards. Thus, when choosing an action at some timestep $t$, these algorithms have access to the updated value of step $t+1$. In contrast, we avoid the planning stage, and therefore must rely on the previous value $\bar{V}^{k-1}$. We will later see that we can overcome this without affecting the regret.

Next, let $L_k$ be the set of 'good' state-action pairs, which is defined in Definition \ref{defn:good set} and analyzed thoroughly in Appendix \ref{sec: Lk definition}. We now decompose the sum of $(B)$ to state actions in and outside $L_k$. We also note that except for the $s_t^k,a_t^k$, all of the variables in $(B)$ are $\F_{k-1}$ measurable, which allows us to explicitly write the conditional expectation using $w_{tk}(s,a)$, as follows:

\clearpage
{\small
\begin{align*}
    (B) & = \sum_{k=1}^K\sum_{t=1}^{H} w_{tk}(s,a) \br*{ (\tilde{r}_{k-1} - r)(s,a)+(\tilde{p}_{k-1} - p)(\cdot\mid s,a)^T \bar{V}_{t+1}^{k-1} } \\
    & = \sum_{k=1}^K\sum_{t=1}^{H} \sum_{(s,a)\in L_k} w_{tk}(s,a) \br*{ (\tilde{r}_{k-1} - r)(s,a)+(\tilde{p}_{k-1} - p)(\cdot\mid s,a)^T \bar{V}_{t+1}^{k-1} } \\
    &\quad + \sum_{k=1}^K\sum_{t=1}^{H} \sum_{(s,a)\notin L_k} w_{tk}(s,a) \br*{ (\tilde{r}_{k-1} - r)(s,a)+(\tilde{p}_{k-1} - p)(\cdot\mid s,a)^T \bar{V}_{t+1}^{k-1} } \\
    &\stackrel{(1)}{\lesssim} \sum_{k=1}^K\sum_{t=1}^{H} \sum_{(s,a)\in L_k} w_{tk}(s,a) \br*{ (\tilde{r}_{k-1} - r)(s,a)+(\tilde{p}_{k-1} - p)(\cdot\mid s,a)^T \bar{V}_{t+1}^{k-1} } \\
    &\quad + H\sum_{k=1}^K\sum_{t=1}^{H} \sum_{(s,a)\notin L_k} w_{tk}(s,a)\\
    &\stackrel{(2)}{\lesssim} \sum_{k=1}^K\sum_{t=1}^{H} \sum_{(s,a)\in L_k} w_{tk}(s,a) \br*{ (\tilde{r}_{k-1} - r)(s,a)+(\tilde{p}_{k-1} - p)(\cdot\mid s,a)^T \bar{V}_{t+1}^{k-1} } + \Olog(SAH^2)
\end{align*}
}
For $(1)$, we bound $(\tilde{r}_{k-1} - r)(s,a) \le \tilde{r}_{k-1}(s,a)$ and $(\tilde{p}_{k-1} - p)(\cdot\mid s,a)^T \bar{V}_{t+1}^{k-1} \le \tilde{p}_{k-1}(\cdot\mid s,a)^T \bar{V}_{t+1}^{k-1}$. The estimated reward is in $[0,1]$ and it's bonus is at most $\Olog(1)$, and thus the optimistic reward $\tilde{r}_{k-1}(s,a)$ is $\Olog(1)$. Due to Lemma \ref{lemma: optimism EULER}, the optimistic value $\bar{V}_{t+1}^{k-1}\le H$, and thus $\hat{p}_{k-1}(\cdot\mid s,a)^T\le H$. The transition bonus is $b_k^{pv}(s,a)=\Olog(H)$ due to Lemma \ref{lemma: bernstein admissible}, which implies that the second term is $\Olog(H)$. Together, both terms are $\Olog(H)$. $(2)$ is due to Lemma \ref{lemma: visitation outside good set} of Appendix \ref{sec: Lk definition}.

As in \citep{zanette2019tighter}, we continue the decomposition of the remaining term by adding and subtracting  cross-terms that depends on $\hat{p}_{k-1}(\cdot\mid s,a)$

\begin{align}
    &\sum_{k=1}^K\sum_{t=1}^{H} \sum_{(s,a)\in L_k} w_{tk}(s,a) \br*{ (\tilde{r}_{k-1} - r)(s,a)+(\tilde{p}_{k-1} - p)(\cdot\mid s,a)^T \bar{V}_{t+1}^{k-1} } \nonumber\\
    & = \sum_{k=1}^K\sum_{t=1}^H  \sum_{(s,a)\in L_k} \underbrace{w_{tk}(s,a)(\tilde{r}_{k-1} - r)(s,a)}_{(1)}+\underbrace{w_{tk}(s,a)(\tilde{p}_{k-1} - \hat{p}_{k-1})^T(\cdot \mid s,a)\bar{V}_{t+1}^{k-1}}_{(2)} \nonumber\\
    &+\underbrace{w_{tk}(s,a)(\hat{p}_{k-1} - p)(\cdot\mid s,a)^TV_{t+1}^*}_{(3)} +\underbrace{w_{tk}(s,a)(\hat{p}_{k-1}-p)(\cdot\mid s,a)^T(\bar{V}_{t+1}^{k-1} - V_{t+1}^*)}_{(4)}. \label{eq: regert of euler paper}
\end{align}

Recall that we use Bernstein's inequality as the admissible confidence interval. Thus, by Lemma~\ref{lemma: central lemma of EULER analysis}, it holds that $J = \frac{2H\ln\frac{2SAT}{\delta'}}{3}=\Olog(H), B_v = \sqrt{2\ln\frac{2SAT}{\delta'}}=\Olog(1)$ and $B_p=H\sqrt{2\ln\frac{2SAT}{\delta'}}=\Olog(H)$. Also let $F,D$ be the constants defined in Lemma~\ref{lemma: next episode value function differences}, and specifically

\begin{align*}
    &F \eqdef 2L+LH\sqrt{S}+6B_vH =\Olog\br*{H\sqrt{S}} \\
    &D\eqdef 18J+4B_p+4L^2 =\Olog\br*{H}
\end{align*}

Substituting these constants, terms $(1)-(4)$ are bounded in Lemmas \ref{lemma: optimistic reward bound}, \ref{lemma: optimistic transition bound}, \ref{lemma: empirical transition bound} and \ref{lemma: lower order} respectively as follows:
\begin{align*}
    &(1) \lesssim \sqrt{\mathbb{C}_r^*SAT}+SA  \\
    &(2) \lesssim \min\brc*{\sqrt{\mathbb{C}^*SAT} + S\sqrt{S}AH^2+SAH^{\frac{5}{2}}, \sqrt{\mathbb{C}^{\pi}SAT} + S\sqrt{S}AH^{\frac{5}{2}}}  \\
    &(3) \lesssim \min\brc*{\sqrt{\mathbb{C}^*SAT}+SAH, \sqrt{\mathbb{C}^{\pi}SAT}+ S\sqrt{S}AH^{\frac{5}{2}}}  \\
    &(4) \lesssim S^2AH^2 +S\sqrt{S}AH^{\frac{5}{2}} 
\end{align*}

Thus, term $(B)$ of the regret is bounded by:
\begin{align*}
    (B) & \lesssim  \min\brc*{\sqrt{\mathbb{C}^*SAT} , \sqrt{\mathbb{C}^{\pi}SAT}} + \sqrt{\mathbb{C}_r^*SAT} +  S^2AH^2 +S\sqrt{S}AH^{\frac{5}{2}} \\
    & \lesssim \sqrt{\min\brc*{\mathbb{C}^*+\mathbb{C}_r^*,\mathbb{C}^{\pi}+\mathbb{C}_r^*}SAT}+  S\sqrt{S}AH^2\br*{\sqrt{S}+\sqrt{H}}
\end{align*}

Finally, using Lemma \ref{lemma: bernstein problem dependent}, we can bound this term by 
\begin{align*}
    (B) \lesssim \sqrt{\min\brc*{\mathbb{Q}^*,\frac{\mathcal{G}^2}{H}}SAT}+  S\sqrt{S}AH^2\br*{\sqrt{S}+\sqrt{H}}\enspace
\end{align*}
and noticing that $(A)$ is negligible compared to $(B)$, we get
\begin{align*}
    \mathrm{Regret}(K)  \lesssim \sqrt{\min\brc*{\mathbb{Q}^*,\frac{\mathcal{G}^2}{H}}SAT}+  S\sqrt{S}AH^2\br*{\sqrt{S}+\sqrt{H}}\enspace
\end{align*}

To derive the problem independent bound, we use the fact that the maximal reward in a trajectory is bounded by $\mathcal{G}\le H$, which yields
\begin{align*}
    \mathrm{Regret}(K)  \lesssim \sqrt{HSAT}+  S\sqrt{S}AH^2\br*{\sqrt{S}+\sqrt{H}}\enspace
\end{align*}

\end{proof}

\subsection{Cumulative Squared Value Difference} \label{sec:helpful lemma proof}

In this section, we aim to bound the expected cumulative squared value difference. Specifically, we are interested in a bound for the following quantities:

\begin{align} 
    &\sum_{k=1}^{K}\sum_{t=1}^{H}\sum_{s,a}w_{tk}(s,a)p(\cdot \mid s,a)^T\br*{\bar{V}^{k-1}_{t+1} - \underline{V}^{k-1}_{t+1}}^2. \label{eq:cumulative squared value differences1} \\
    &\sum_{k=1}^{K}\sum_{t=1}^{H}\sum_{s,a}w_{tk}(s,a)p(\cdot \mid s,a)^T\br*{\bar{V}^{k-1}_{t+1} - V^{\pi_k}_{t+1}}^2\label{eq:cumulative squared value differences2}
\end{align}

The first quantity allows us to replace Lemma 12 \citep{zanette2019tighter}, and the second allows us to prove Lemma 14 of the same paper. Together, they enable us to use the same analysis of \citep{zanette2019tighter}. The final results are stated in Lemmas \ref{lemma: central lemma of EULER analysis} and \ref{lemma: squared difference upper to real} by the end of this section. Most of the section will focus on bounding \eqref{eq:cumulative squared value differences1}, which requires a much more delicate analysis than the bound of \citep{zanette2019tighter}.

In order to bound \eqref{eq:cumulative squared value differences1}, we start by bounding $\br*{\bar{V}^{k-1}_{t+1} - \underline{V}^{k-1}_{t+1}}^2$ in the following lemma, which corresponds to Proposition 5 of \citep{zanette2019tighter}:

\begin{lemma}\label{lemma: next episode value function differences}
Outside the failure event, the following holds:
\begin{align*}
    \bar{V}^{k}_t(s_t^k) - \underline{V}^{k}_t(s_t^k) \le \E[\bar{V}^{k-1}_{t+1}(s_{t+1}^k) - \underline{V}^{k-1}_{t+1}(s_{t+1}^k)\mid \mathcal{F}_{k-1},s_t^k] + \min\brc*{\frac{F+D}{\sqrt{n_{k-1}(s_t^k,a_t^k) \vee 1 }},H},
\end{align*}
where $F  \eqdef 2L+LH\sqrt{S}+6B_vH,\ D\eqdef 18J+4B_p+4L^2$, the constants $J,B_v,B_p$ are defined in Definition \ref{defn:admissible} and $L \eqdef 2\sqrt{\ln\frac{4SAT}{\delta'}}$.

\end{lemma}
\begin{proof}
The proof is similar to \citep{zanette2019tighter} Proposition 5, which is presented here with the needed adaptation.

If the state $s_t^k$ is encountered in the $k^{th}$ episode at the $t^{th}$ time-step, then $\bar{V}_t^k(s_t^k),\underline{V}^k_t(s_t^k)$ will be updated according to the update rule. Thus,
{\small
\begin{align*}
    &\bar{V}_t^{k}(s_t^k) \leq \hat{r}_{k-1}(s_t^k,a_t^k) + b^r_{k-1}(s_t^k,a_t^k) + \hat{p}_{k-1}(\cdot\mid s_t^k,a_t^k)^T \bar{V}^{k-1}_{t+1} + b_k^{pv}(\hat{p}_{k-1}(\cdot\mid s_t^k,a_t^k),\bar{V}^{k-1}_{t+1},\underline{V}^{k-1}_{t+1})\\
    &\underline{V}_t^{k}(s_t^k) \geq \hat{r}_{k-1}(s_t^k,a_t^k) - b^r_{k-1}(s_t^k,a_t^k) + \hat{p}_{k-1}(\cdot\mid s_t^k,a_t^k)^T \underline{V}^{k-1}_{t+1} - b_k^{pv}(\hat{p}_{k-1}(\cdot\mid s_t^k,a_t^k),\underline{V}^{k-1}_{t+1},\bar{V}^{k-1}_{t+1}).
\end{align*}
}

Subtraction yields:
\begin{align*}
    \bar{V}_t^{k}(s_t^k) - \underline{V}_t^{k}(s_t^k) & \leq    2b^r_{k-1}(s_t^k,a_t^k) + \hat{p}_{k-1}(\cdot\mid s_t^k,a_t^k)^T (\bar{V}^{k-1}_{t+1} - \underline{V}^{k-1}_{t+1}) \\ 
    &\quad +b_k^{pv}(\hat{p}_{k-1}(\cdot\mid s_t^k,a_t^k),\bar{V}^{k-1}_{t+1},\underline{V}^{k-1}_{t+1}) 
    + b_k^{pv}(\hat{p}_{k-1}(\cdot\mid s_t^k,a_t^k),\underline{V}^{k-1}_{t+1},\bar{V}^{k-1}_{t+1}).
\end{align*}

Next, we substitute the definition of the confidence bonus, which yields
\begin{align*}
     \bar{V}_t^{k}(s_t^k) - \underline{V}_t^{k}(s_t^k) \leq & 2b^r_{k-1}(s_t^k,a_t^k) + \hat{p}_{k-1}(\cdot\mid s_t^k,a)^T (\bar{V}^{k-1}_{t+1} - \underline{V}^{k-1}_{t+1})\\
     &+ \phi(\hat{p}_{k-1}(\cdot\mid s_t^k,a_t^k),\bar{V}_{t+1}^{k-1}) + \frac{4J+B_p}{n_{k-1}(s_t^k,a_t^k)\vee 1} + \frac{B_v \norm{\bar{V}^{k-1}_{t+1} - \underline{V}^{k-1}_{t+1}}_{2,\hat{p}}}{\sqrt{n_{k-1}(s_t^k,a_t^k)\vee 1}}\\
     &+ \phi(\hat{p}_{k-1}(\cdot\mid s_t^k,a_t^k),\underline{V}_{t+1}^{k-1}) + \frac{4J+B_p}{n_{k-1}(s_t^k,a_t^k)\vee 1} + \frac{B_v \norm{\bar{V}^{k-1}_{t+1}- \underline{V}^{k-1}_{t+1}}_{2,\hat{p}}}{\sqrt{n_{k-1}(s_t^k,a_t^k)\vee 1}}.     
\end{align*}

Using Lemma \ref{lemma:admiss properties}, property (\ref{lemma-part:p change}), and Inequalities (\ref{eq:value norm bound}), we get,
\begin{align*}
    \bar{V}_t^{k}(s_t^k) - \underline{V}_t^{k}(s_t^k) &\leq  2b^r_{k-1}(s_t^k,a_t^k) + \hat{p}_{k-1}(\cdot\mid s_t^k,a)^T (\bar{V}^{k-1}_{t+1} - \underline{V}^{k-1}_{t+1})\\
    &\quad +2\phi(p(\cdot\mid s_t^k,a_t^k),\underline{V}_{t+1}^*)
     +4\br*{\frac{4J+B_p}{n_{k-1}(s_t^k,a_t^k)\vee 1} \!+ \!\frac{B_v \norm{\bar{V}^{k-1}_{t+1} - \underline{V}^{k-1}_{t+1}}_{2,\hat{p}} }{\sqrt{n_{k-1}(s_t^k,a_t^k)\vee 1}} }\\
     & = 2b^r_{k-1}(s_t^k,a_t^k) + p(\cdot\mid s_t^k,a)^T (\bar{V}^{k-1}_{t+1} - \underline{V}^{k-1}_{t+1})\\
    &\quad + (\hat{p}_{k-1}(\cdot\mid s_t^k,a_t^k) - p(\cdot \mid s_t^k,a_t^k))^T (\bar{V}^{k-1}_{t+1} - \underline{V}^{k-1}_{t+1})\\
    &\quad + 2 \frac{g(p(\cdot \mid s_t^k,a_t^k),V^*_{t+1})}{\sqrt{n_{k-1}(s_t^k,a_t^k)\vee 1}} + 2\frac{J}{n_{k-1}(s_t^k,a_t^k)\vee 1} \\ 
    &\quad + 4\br*{\frac{4J+B_p}{n_{k-1}(s_t^k,a_t^k)\vee 1} + \frac{B_v \norm{\bar{V}^{k-1}_{t+1} - \underline{V}^{k-1}_{t+1}}_{2,\hat{p}} }{\sqrt{n_{k-1}(s_t^k,a_t^k)\vee 1}} },
\end{align*}
where in the last relation we substituted $\phi$ and added and subtracted $p(\cdot\mid s_t^k,a)^T (\bar{V}^{k-1}_{t+1} - \underline{V}^{k-1}_{t+1})$.

By Lemma \ref{lemma: 3 properties of decreasing process of EULER}, we know that $\bar{V}^{k-1}_{t+1},\underline{V}^{k-1}_{t+1}\in\brs*{0,H}$. Thus, $\norm*{\bar{V}^{k-1}_{t+1} - \underline{V}^{k-1}_{t+1}}\le H$, which also implies that $\norm{\bar{V}^{k-1}_{t+1} - \underline{V}^{k-1}_{t+1}}_{2,\hat{p}}\le H$. In addition, using H{\"o}lder's inequality, and outside failure event $F^{pn1}$, we can bound 

\begin{align*}
    (\hat{p}_{k-1}(\cdot\mid s_t^k,a_t^k) &- p(\cdot \mid s_t^k,a_t^k))^T (\bar{V}^{k-1}_{t+1} - \underline{V}^{k-1}_{t+1}) \\
    & \le \norm*{{\hat{p}_{k-1}(\cdot\mid s_t^k,a_t^k) - p(\cdot \mid s_t^k,a_t^k)}}_1 \norm*{\bar{V}^{k-1}_{t+1} - \underline{V}^{k-1}_{t+1}}_\infty \\
    & \le H\sqrt{\frac{4S \ln \frac{2SAT}{\delta'}}{n_{k-1}(s,a)\vee 1}}
    = L H \sqrt{\frac{S}{n_{k-1}(s_t^k,a_t^k)\vee 1}}
\end{align*}

Substituting both of these bounds, we get

\begin{align}
    \bar{V}_t^{k}(s_t^k) - \underline{V}_t^{k}(s_t^k) 
    &\leq 2b^r_{k-1}(s_t^k,a_t^k) + p(\cdot\mid s_t^k,a)^T (\bar{V}^{k-1}_{t+1} - \underline{V}^{k-1}_{t+1})
    + L H \sqrt{\frac{S}{n_{k-1}(s_t^k,a_t^k)\vee 1}}\nonumber \\
    &\quad + 2 \frac{g(p(\cdot \mid s_t^k,a_t^k),V^*_{t+1})}{\sqrt{n_{k-1}(s_t^k,a_t^k)\vee 1}} + 2\frac{J}{n_{k-1}(s_t^k,a_t^k)\vee 1} \nonumber\\ 
    &\quad + 4\br*{\frac{4J+B_p}{n_{k-1}(s_t^k,a_t^k)\vee 1} + \frac{B_v H }{\sqrt{n_{k-1}(s_t^k,a_t^k)\vee 1}} } \label{eq:value diff partial bound}
\end{align}

We now bound the remaining terms. First, using Lemma \ref{lemma:admiss properties}, property (\ref{lemma-part:g bound}), we can bound  $g(p,V^*_{t+1}) \leq B_v H$. Second, notice that 
\begin{align*}
    p(\cdot \mid s_t^k,a))^T (\bar{V}^{k-1}_{t+1} - \underline{V}^{k-1}_{t+1}) 
    &= \sum_{s_{t+1}^k} p(s_{t+1}^k \mid s_t^k,a))^T (\bar{V}^{k-1}_{t+1}(s_{t+1}^k) - \underline{V}^{k-1}_{t+1}(s_{t+1}^k)) \\
    & = \E\brs*{\bar{V}^{k-1}_{t+1}(s_{t+1}^k) - \underline{V}^{k-1}_{t+1}(s_{t+1}^k)\mid \mathcal{F}_{k-1},s_t^k}.
\end{align*}

Finally, outside failure event $F^r$, the reward bonus can be bounded by

\begin{align*}
b^r_k(s_t^k,a_t^k) &= \sqrt{\frac{2\hat{\VAR}(R(s_t^k,a_t^k))\ln\frac{4SAT}{\delta'}}{n_{k-1}(s_t^k,a_t^k)\vee 1}} + \frac{14\ln \frac{4SAT}{\delta'}}{3n_{k-1}(s_t^k,a_t^k)\vee 1} \\
& \le \frac{L}{\sqrt{n_{k-1}(s_t^k,a_t^k)\vee 1}} + \frac{2L^2}{n_{k-1}(s_t^k,a_t^k)\vee 1}
\end{align*}
where we used the fact that for variables in $\brs*{0,1}$, $\hat{\VAR}(R(s_t^k,a_t^k))\le1$.

Putting it all together in (\ref{eq:value diff partial bound}), we get

\begin{align*}
    \bar{V}_t^{k}(s_t^k) - \underline{V}_t^{k}(s_t^k) 
    & \le \E\brs*{\bar{V}^{k-1}_{t+1}(s_{t+1}^k) - \underline{V}^{k-1}_{t+1}(s_{t+1}^k)\mid \mathcal{F}_{k-1},s_t^k} \\
    & \quad+ \frac{2L+LH\sqrt{S}+6B_vH}{\sqrt{n_{k-1}(s_t^k,a_t^k)\vee 1}} + \frac{18J+4B_p+4L^2}{n_{k-1}(s_t^k,a_t^k)\vee 1} \\
    & \le \E\brs*{\bar{V}^{k-1}_{t+1}(s_{t+1}^k) - \underline{V}^{k-1}_{t+1}(s_{t+1}^k)\mid \mathcal{F}_{k-1},s_t^k} 
    + \frac{F+D}{n_{k-1}(s_t^k,a_t^k)\vee 1}
\end{align*}

where in the last relation we substituted $F$ and $D$ and used $\sqrt{n}\leq n$ for $n\geq 1$.

To finalize the proof note that outside the failure event, $\underline{V}^{k}_t(s)\le\bar{V}^{k}_t(s)$, and the first term is therefore positive. combined with $\bar{V}^{k}_t(s_t^k) - \underline{V}^{k}_t(s_t^k) \le H$ yields
{\small
\begin{align*}
    \bar{V}^{k}_t(s_t^k) - \underline{V}^{k}_t(s_t^k) 
    &\le \min\brc*{\E\brs*{\bar{V}^{k-1}_{t+1}(s_{t+1}^k) - \underline{V}^{k-1}_{t+1}(s_{t+1}^k)\mid \mathcal{F}_{k-1},s_t^k} + \frac{F+D}{\sqrt{n_{k-1}(s_t^k,a_t^k)\vee 1}},H} \\
    & \le \E\brs*{\bar{V}^{k-1}_{t+1}(s_{t+1}^k) - \underline{V}^{k-1}_{t+1}(s_{t+1}^k)\mid \mathcal{F}_{k-1},s_t^k} + \min\brc*{\frac{F+D}{\sqrt{n_{k-1}(s_t^k,a_t^k)\vee 1}},H} .
\end{align*}
}
\end{proof}

\begin{remark}\label{remark: similar dependence as in original EULER}
See that the first term in Equation Lemma \ref{lemma: next episode value function differences} does not appear in the analysis of \citep{zanette2019tighter}. Its existence is a direct consequence of the fact we use 1-step greedy policies, and not solving the approximate model at the beginning of each episode. Remarkably, we will later see that this term is comparable to the other previously existing terms. 
\end{remark}

\clearpage

We now move to bounding the expected squared value difference, as formally stated in as follows:

\begin{lemma}\label{lemma: v k-1 square diff bound}
Let $\Delta^{k}_{t} \eqdef \br*{\bar{V}^{k-1}_t(s^k_t) - \underline{V}^{k-1}_t(s^k_t)} - \br*{\bar{V}^{k}_t(s^k_t) - \underline{V}^{k}_t(s^k_t)}$. Then, outside the failure event,
\begin{align*}
    &\E[\br*{\bar{V}^{k-1}_t(s_t^k) - \underline{V}^{k-1}_t(s_t^k)}^2\mid \mathcal{F}_{k-1}] \\
    &\le 2H\sum_{t'=t}^{H-1} \E\brs*{{\Delta_{t'}^k(s^k_{t'})}^2  + \min\brc*{\frac{(F+D)^2}{n_{k-1}(s_{t'}^k,a_{t'}^k)\vee 1},H^2}\mid \mathcal{F}_{k-1}},
\end{align*}
where $F+D$ is defined in Lemma \ref{lemma: next episode value function differences}.
\end{lemma}
\begin{proof}

Before proving the bound, we express the bound of Lemma \ref{lemma: next episode value function differences} in terms of $\Delta_t^k$. For brevity, we denote $Y_{k}(s,a)\eqdef\min\brc*{\frac{F+D}{\sqrt{n_{k}(s,a)}\vee 1},H}$, which is $\F_{k}$ measurable.

Assume the state $s_t^k$ is visited in the $k^{th}$ episode at the $t^{th}$ time-step. Then, by Lemma \ref{lemma: next episode value function differences},
\begin{align}
    \bar{V}^{k-1}_t(s_t^k) &- \underline{V}^{k-1}_6(s_t^k)
    =\Delta^{k}_{t} + \bar{V}^{k}_t(s^k_t) - \underline{V}^{k}_t(s^k_t) \nonumber\\
    &\le  \Delta^{k}_{t} + Y_{k-1}(s_t^k,a_t^k) + \underset{(*)}{\underbrace{\E\brs*{\bar{V}^{k-1}_{t+1}(s_{t+1}^k) - \underline{V}^{k-1}_{t+1}(s_{t+1}^k)\mid \mathcal{F}_{k-1},s_t^k}}}\enspace. \label{eq: lemma k-1 differences 1st equation}
\end{align}

Next, by substituting Equation (\ref{eq: lemma k-1 differences 1st equation}) in (*), we get
\begin{align*}
    (*) &\le  \E\brs*{\Delta^{k}_{t+1} + Y_{k-1}(s_{t+1}^k,a_{t+1}^k) + \E\brs*{ \bar{V}^{k-1}_{t+2}(s_{t+2}^k) - \underline{V}^{k-1}_{t+2}(s_{t+2}^k) \mid \mathcal{F}_{k-1},s_{t+1}^k} \mid  \mathcal{F}_{k-1},s_t^k}\\
    &=  \E\brs*{\Delta^{k}_{t+1} + Y_{k-1}(s_{t+1}^k,a_{t+1}^k) + \bar{V}^{k-1}_{t+2}(s_{t+2}^k) - \underline{V}^{k-1}_{t+2}(s_{t+2}^k) \mid  \mathcal{F}_{k-1},s_t^k},
\end{align*}

where the last relation holds by the tower property.

Iterating using this technique until $t=H$, and using $\bar{V}_{H+1}=\underline{V}_{H+1}=0$, we conclude the following bound:
\begin{align*}
    \bar{V}^{k-1}_t(s_t^k) - \underline{V}^{k-1}_t(s_t^k)
    \le \sum_{t'=t}^{H} \E[\Delta_{t'}^k(s^k_{t'})  + Y_{k-1}(s_{t'}^k,a_{t'}^k)\mid \mathcal{F}_{k-1},s_t^k], 
\end{align*}

With this bound at hand, we can derive the desired result as follows:
\begin{align*}
    \br*{\bar{V}^{k-1}_t(s_t^k) - \underline{V}^{k-1}_t(s_t^k)}^2
    &\le \br*{\sum_{t'=t}^{H} \E\brs*{\Delta_{t'}^k(s^k_{t'})  + Y_{k-1}(s_{t'}^k,a_{t'}^k)\mid \mathcal{F}_{k-1},s_t^k}}^2 \\
    & \stackrel{(CS)}{\le} (H-t+1)\sum_{t'=t}^{H} \E\brs*{\Delta_{t'}^k(s^k_{t'})  + Y_{k-1}(s_{t'}^k,a_{t'}^k)\mid \mathcal{F}_{k-1},s_t^k}^2 \\
    & \stackrel{(J)}{\le} (H-t+1)\sum_{t'=t}^{H} \E\brs*{\br*{\Delta_{t'}^k(s^k_{t'})  + 
    Y_{k-1}(s_{t'}^k,a_{t'}^k)}^2\mid \mathcal{F}_{k-1},s_t^k} \\
    & \stackrel{(CS)}{\le} 2(H-t+1)\sum_{t'=t}^{H} \E\brs*{\Delta_{t'}^k(s^k_{t'})^2  + Y_{k-1}^2(s_{t'}^k,a_{t'}^k)\mid \mathcal{F}_{k-1},s_t^k} \\
    & \le 2H \sum_{t'=t}^{H} \E\brs*{\Delta_{t'}^k(s^k_{t'})^2  + Y_{k-1}^2(s_{t'}^k,a_{t'}^k)\mid \mathcal{F}_{k-1},s_t^k} 
\end{align*}

$(CS)$ denotes Cauchy-Schwarz inequality, and specifically $\br*{\sum_{i=1}^n a_i}^2 \le n\sum_{i=1}^n a_i^2$. $(J)$ is Jensen's inequality. Taking the conditional expectation $\E\brs{\cdot \mid \F_{k-1}}$, using the tower property and substituting $Y_{k}(s,a)$ gives the desired result.
\end{proof}

After bounding the expected squared value difference in a single state, we now move to bounding its sum over different time-steps and episode. The main difficulty is in bounding the sum over the first term, which we bound in the following lemma:

\begin{lemma} \label{lemma: value differences square sum on episodes}
Outside the failure event,
\begin{align*}
    &\sum_{k=1}^K\sum_{t=1}^{H}\sum_{t'=t}^{H-1}\E[{\Delta_{t'}^k(s^k_{t'})}^2 \mid \mathcal{F}_{k-1}]\leq \Olog(SH^4),
\end{align*}
where $\Delta_t^k(s_t^k)$ is defined in Lemma \ref{lemma: v k-1 square diff bound}.
\end{lemma}
\begin{proof}
We have that
\begin{align*}
    \sum_{t=1}^{H}\sum_{t'=t}^{H}\E[{\Delta_{t'}^k(s^k_{t'})}^2 \mid \mathcal{F}_{k-1}] = \sum_{t=1}^{H}t\E[{\Delta_t^k(s^k_{t})}^2 \mid \mathcal{F}_{k-1}] \leq H \sum_{t=1}^{H} \E[{\Delta_t^k(s^k_{t})}^2 \mid \mathcal{F}_{k-1}].
\end{align*}

Furthermore,
\begin{align*}
    \br*{\Delta_t^k(s^k_{t})}^2  &= \br*{\br*{\bar{V}^{k-1}_t(s^k_{t}) - \underline{V}^{k-1}_t(s^k_{t})} - \br*{\bar{V}^{k}_t(s^k_{t}) - \underline{V}^{k}_t(s^k_{t}) }}^2\\
     &= \br*{\bar{V}^{k-1}_t(s^k_{t}) - \underline{V}^{k-1}_t(s^k_{t})}^2 + \br*{\bar{V}^{k}_t(s^k_{t}) - \underline{V}^{k}_t(s^k_{t})}^2 \\
     &\quad -2\br*{\bar{V}^{k-1}_t(s^k_{t}) - \underline{V}^{k-1}_t(s^k_{t})}\br*{\bar{V}^{k}_t(s^k_{t}) - \underline{V}^{k}_t(s^k_{t})}\\
      &\leq \br*{\bar{V}^{k-1}_t(s^k_{t}) - \underline{V}^{k-1}_t(s^k_{t})}^2 + \br*{\bar{V}^{k}_t(s^k_{t}) - \underline{V}^{k}_t(s^k_{t})}^2 - 2\br*{\bar{V}^{k}_t(s^k_{t}) - \underline{V}^{k}_t(s^k_{t})}^2\\
      & = \br*{\bar{V}^{k-1}_t(s^k_{t}) - \underline{V}^{k-1}_t(s^k_{t})}^2 - \br*{\bar{V}^{k}_t(s^k_{t}) - \underline{V}^{k}_t(s^k_{t})}^2\enspace,
\end{align*}
where the third relation holds since $\bar{V}^k(s), \underline{V}^k(s)$ decreases and increases, respectively, by Lemma \ref{lemma: 3 properties of decreasing process of EULER}, and since outside of the failure event $\bar{V}^k(s) \ge \underline{V}^k(s), \forall k$ (Lemma \ref{lemma: optimism EULER}). Another implication these properties is that

$$\br*{\bar{V}^{k-1}_t(s^k_{t}) - \underline{V}^{k-1}_t(s^k_{t})}^2 \geq \br*{\bar{V}^{k}_t(s^k_{t}) - \underline{V}^{k}_t(s^k_{t})}^2,$$
Thus,
\begin{align}
     &H \sum_{k=1}^{K}\sum_{t=1}^{H} \E[{\Delta_t^k(s^k_{t})}^2 \mid \mathcal{F}_{k-1}] \nonumber\\
     & \leq  H  \sum_{k=1}^{K}\sum_{t=1}^{H} \E[ \br*{\bar{V}^{k-1}_t(s^k_{t}) - \underline{V}^{k-1}_t(s^k_{t})}^2 - \br*{\bar{V}^{k}_t(s^k_{t}) - \underline{V}^{k}_t(s^k_{t})}^2 \mid \mathcal{F}_{k-1}] \label{eq: quadratic decreasing process for euler proof 1 relation}\enspace.
\end{align}
\clearpage
For brevity, we define $\Delta V^k_t (s) = \bar{V}^{k}_t(s) - \underline{V}^{k}_t(s)$. Similarly to the technique used in Lemma \ref{lemma: regret to SH decreasing processes} (Appendix \ref{sec: supp general lemmas}), 
\begin{align*}
    &\sum_{k=1}^{K}\sum_{t=1}^{H} \E[ \br*{\bar{V}^{k-1}_t(s^k_{t}) - \underline{V}^{k-1}_t(s^k_{t})}^2 - \br*{\bar{V}^{k}_t(s^k_{t}) - \underline{V}^{k}_t(s^k_{t})}^2 \mid \mathcal{F}_{k-1}]\\
    & =\sum_{k=1}^{K}\sum_{t=1}^{H} \E[ \Delta V^{k-1}_t (s^k_{t})^2 - \Delta V^{k}_t (s^k_{t})^2 \mid \mathcal{F}_{k-1}]\\
    & \stackrel{(1)}{=}\sum_{k=1}^{K}\sum_{t=1}^{H}\sum_s \E[  \ind\brc*{s_t^k= s}\Delta V^{k-1}_t (s)^2 -  \ind\brc*{s_t^k= s}\Delta V^{k}_t (s)^2 \mid \mathcal{F}_{k-1}]\\
    & \stackrel{(2)}{=}\sum_{k=1}^{K}\sum_{t=1}^{H}\sum_s \E[  \ind\brc*{s_t^k= s}\Delta V^{k-1}_t (s)^2 + \ind\brc*{s_t^k\neq s}\Delta V^{k-1}_t (s)^2\mid\mathcal{F}_{k-1}]\\
    &\qquad\qquad\qquad-\E[\ind\brc*{s_t^k= s}\Delta V^{k}_t (s)^2 + \ind\brc*{s_t^k\neq s}\Delta V^{k-1}_t (s)^2    \mid \mathcal{F}_{k-1}]\\
    & \stackrel{(3)}{=}\sum_{k=1}^{K}\sum_{t=1}^{H}\sum_s \Delta V^{k-1}_t (s)^2 - \E[\Delta V^{k}_t (s)^2 \mid \mathcal{F}_{k-1}]\\
\end{align*}

Relation $(1)$ holds by adding and subtracting $\ind\brc*{s \neq s_t^k}\bar{V}_t^{k-1}(s)$ while using the linearity of expectation. $(2)$ holds since for any event $\ind\brc{A}+\ind\brc{A^c}=1$ and since $\Delta V_t^{k-1}$ is $\F_{k-1}$ measurable. $(3)$ holds by the definition of the update rule. If state $s$ is visited in the $k^{th}$ episode at time-step $t$, then both $\bar{V}^k_t(s),\underline{V}^k_t(s)$ are updated. If not, their value remains as in the $k-1$ iteration.

Lastly,
\begin{align*}
    &\sum_{k=1}^{K}\sum_{t=1}^{H}\sum_s \Delta V^{k-1}_t (s)^2 - \E[\Delta V^{k}_t (s)^2 \mid \mathcal{F}_{k-1}]\\
    & =\sum_{k=1}^{K}\sum_{t=1}^{H}\sum_s \br*{\bar{V}^{k-1}_t(s) - \underline{V}^{k-1}_t(s)}^2 - \E[\br*{\bar{V}^{k}_t(s) - \underline{V}^{k}_t(s)}^2 \mid \mathcal{F}_{k-1}]\leq \Olog(SH^3),
\end{align*}
where the inequality holds outside the failure event $F^{vsDP}$, which is defined in Appendix \ref{sec:EULER failure decreasing}. Plugging this into \eqref{eq: quadratic decreasing process for euler proof 1 relation} concludes the proof. 
\end{proof}

\clearpage

We are now ready to prove the main results of this section and bound \eqref{eq:cumulative squared value differences1} and \eqref{eq:cumulative squared value differences2}:

\begin{restatable}{lemma}{EULERCentralLemma}\label{lemma: central lemma of EULER analysis}
Outside the failure event.
\begin{align*}
    &\sum_{k=1}^{K}\sum_{t=1}^{H}\sum_{s,a}w_{tk}(s,a)p(\cdot \mid s,a)^T\br*{\bar{V}^{k-1}_{t+1} - \underline{V}^{k-1}_{t+1}}^2
    \leq  \Olog(SAH^2(F+D)^2 + SAH^5 ).
\end{align*}
where $F+D$ is defined in Lemma \ref{lemma: next episode value function differences}
\end{restatable}

\begin{proof}

Recall that $w_{tk}(s,a)=\Pr(s^k_{t} \mid s^k_1,\pi_k)$ is the probability when following $\pi^k$ in the true MDP the state-action in the $k^{th}$ episode at the $t^{th}$ time-step is $(s^k_t,a_t^k)=(s,a)$. Thus, the following relation holds.
\begin{align*}
    &\sum_{s,a} w_{tk}(s,a) p(\cdot \mid s,a)^T(\bar{V}_{t+1}^{k-1} - \underline{V}_{t+1}^{k-1})^2\\
    &= \sum_{s_t}\Pr(s^k_{t} \mid s^k_1,\pi_k)\sum_{s_{t+1}}p(s^k_{t+1} \mid s^k_t,a_t^k)(\bar{V}_{t+1}^{k-1}(s^k_{t+1}) - \underline{V}_{t+1}^{k-1}(s^k_{t+1}))^2 \\
    &= \sum_{s_{t+1}}\Pr(s^k_{t+1} \mid s^k_1,\pi_k)(\bar{V}_{t+1}^{k-1}(s^k_{t+1}) - \underline{V}_{t+1}^{k-1}(s^k_{t+1}))^2 \\
    &= \E[(\bar{V}_{t+1}^{k-1}(s_{t+1}) - \underline{V}_{t+1}^{k-1}(s_{t+1}))^2 \mid \mathcal{F}_{k-1}].
\end{align*}

Since $\bar{V}_{H+1}^{k-1}(s_{t+1}) = \underline{V}_{H+1}^{k-1}(s_{t+1} = 0$, we obtain,
\begin{align*}
    &\sum_{k=1}^{K}\sum_{t=1}^{H}\sum_{s,a} w_{tk}(s,a) p(\cdot \mid s,a)(\bar{V}_{t+1}^{k-1} - \underline{V}_{t+1}^{k-1})^2\\
    &= \sum_{k=1}^{K}\sum_{t=1}^{H} \E[ (\bar{V}_{t+1}^{k-1}(s^k_{t}) - \underline{V}_{t+1}^{k-1}(s^k_{t})^2\mid \mathcal{F}_{k-1}]\\
    &\leq  \sum_{k=1}^{K}\sum_{t=1}^{H} \E[ (\bar{V}_{t}^{k-1}(s^k_{t}) - \underline{V}_{t}^{k-1}(s^k_{t})^2\mid \mathcal{F}_{k-1}].
\end{align*}

Thus,
{\small
\begin{align}
    &\sum_{k=1}^{K}\sum_{t=1}^{H} \sum_{s,a}w_{tk}(s,a)p(\cdot \mid s,a)^T\br*{\bar{V}^{k-1}_{t+1} - \underline{V}^{k-1}_{t+1}}^2 \nonumber\\
    &\leq \sum_{k=1}^{K}\sum_{t=1}^{H} \E[\br*{\bar{V}^{k-1}_{t+1}(s^k_{t}) - \underline{V}^{k-1}_{t+1}(s^k_{t})}^2\mid \mathcal{F}_{k-1}] \nonumber\\
    & \stackrel{(*)}{\leq}  2H \sum_{k=1}^{K}\sum_{t=1}^{H}\sum_{t'=t}^{H} \E[{\Delta_{t'}^k(s^k_{t'})}^2\mid \mathcal{F}_{k-1}]  
    + 2H \sum_{k=1}^{K}\sum_{t=1}^{H}\sum_{t'=t}^{H} \E\brs*{ \min\brc*{\frac{(F+D)^2}{n_{k-1}(s^k_{t'},a_{t'}^k)\vee 1},H^2}\mid \mathcal{F}_{k-1}} \nonumber\\
    & = 2H\sum_{k=1}^{K}\sum_{t=1}^{H}t \E[{\Delta_t^k(s^k_t)}^2\mid \mathcal{F}_{k-1}]  
    + 2H \sum_{k=1}^{K}\sum_{t=1}^{H}t \E\brs*{ \min\brc*{\frac{(F+D)^2}{n_{k-1}(s^k_t,a_t^k)\vee 1},H^2}\mid \mathcal{F}_{k-1}} \nonumber \\
    & \le 2H^2\sum_{k=1}^{K}\sum_{t=1}^{H} \E[{\Delta_t^k(s^k_t)}^2\mid \mathcal{F}_{k-1}]  
    + 2H^2 \sum_{k=1}^{K}\sum_{t=1}^{H} \E\brs*{ \min\brc*{\frac{(F+D)^2}{n_{k-1}(s^k_t,a_t^k)\vee 1},H^2}\mid \mathcal{F}_{k-1}} \label{eq: sum squared value diff bound}
\end{align}
}
where $(*)$ last relation holds by Lemma \ref{lemma: v k-1 square diff bound}, in which $\Delta_t^k(s_t^k)$ is defined. The first term is bounded in Lemma \ref{lemma: value differences square sum on episodes} by $\Olog(SH^5)$. The second term is bounded outside the failure event Using the 'Good Set' $L_k$, which is defined and analyzed in Appendix \ref{sec: Lk definition}. The bound for this term can be found in Lemma \ref{lemma: sum of 1 over n}. Combining both of the results and substituting into \eqref{eq: sum squared value diff bound} yields

\begin{align*}
    \sum_{k=1}^{K}\sum_{t=1}^{H}\sum_{s,a}w_{tk}(s,a)p(\cdot \mid s,a)^T\br*{\bar{V}^{k-1}_{t+1} - \underline{V}^{k-1}_{t+1}}^2
    &\leq \Olog(SH^5)+\Olog(SAH^2(F+D)^2 + SAH^5)\\
    &=\Olog(SAH^2(F+D)^2 + SAH^5)
\end{align*}
\end{proof}

\begin{restatable}{lemma}{SquaredDiffUpperReal}\label{lemma: squared difference upper to real}
Outside the failure event.
\begin{align*}
    &\sum_{k=1}^{K}\sum_{t=1}^{H}\sum_{s,a}w_{tk}(s,a)p(\cdot \mid s,a)^T\br*{\bar{V}^{k-1}_{t+1} - V^{\pi_k}_{t+1}}^2
    \leq \Olog(SAH^3(F+D)^2 + SAH^5)
\end{align*}
where $F+D$ is defined in Lemma \ref{lemma: next episode value function differences}
\end{restatable}

\begin{proof}

Similarly to Lemma \ref{lemma: central lemma of EULER analysis}, we have that
\begin{align}
    &\sum_{k=1}^{K}\sum_{t=1}^{H}\sum_{s,a}w_{tk}(s,a)p(\cdot \mid s,a)^T\br*{\bar{V}^{k-1}_{t+1} - V^{\pi_k}_{t+1}}^2 \nonumber\\
    &\leq \sum_{k=1}^{K}\sum_{t=1}^{H}
    \E\brs*{\br*{\bar{V}^{k-1}_{t+1}(s_t^k) - V^{\pi_k}_{t+1}(s_t^k)}^2\mid \F_{k-1}} \nonumber\\
    &\leq \sum_{k=1}^{K}\sum_{t=1}^{H}
    \E\brs*{\br*{\bar{V}^{k-1}_{t}(s_t^k) - V^{\pi_k}_{t}(s_t^k)}^2\mid \F_{k-1}} \label{eq: squared difference lemma first eq}.
\end{align}

where the last inequality is since $\bar{V}_{H+1}^{k-1}(s_{t+1}) = V_{H+1}^{\pi_k}(s_{t+1} = 0$. Applying Lemma \ref{lemma: model base RL expected value difference}, we get,
{\small
\begin{align*}
    &\E[\br*{\bar{V}^{k-1}_{t}(s_t^k) - V^{\pi_k}_{t}(s_t^k)}^2\mid \F_{k-1}]\\
    &\stackrel{(1)}{\leq} \E\brs*{\br*{\sum_{t'=t}^H \E\brs*{\bar{V}^{k-1}(s_{t'}^k) - \bar{V}^{k}(s_{t'}^k) + (\tilde r_{k-1} - r)(s_{t'}^k,a_{t'}^k)+(\tilde p_{k-1} - p)(s_{t'}^k,a_{t'}^k)\bar V^{k-1}_{t+1} \mid \F_{k-1},s_t^k }}^2\mid \F_{k-1}}\\
    &\stackrel{(2)}{\leq} 3H\E\brs*{\sum_{t'=t}^H \E\brs*{\br*{\bar{V}^{k-1}(s_{t'}^k) - \bar{V}^{k}(s_{t'}^k)}^2 \mid \F_{k-1},s_t^k }\mid \F_{k-1}}\\  
    &\quad +3H\E\brs*{\sum_{t'=t}^H \E\brs*{\br*{(\tilde r_{k-1} - r)(s_{t'}^k,a_{t'}^k)}^2+\br*{(\tilde p_{k-1} - p)(s_{t'}^k,a_{t'}^k)\bar V^{k-1}_{t+1}}^2 \mid \F_{k-1},s_t^k }\mid \F_{k-1}}\\ 
    &\stackrel{(3)}{=} 3H\sum_{t'=t}^H\E\brs*{ \br*{\bar{V}^{k-1}(s_{t'}^k) - \bar{V}^{k}(s_{t'}^k)}^2\mid \F_{k-1}} \\
    & \quad + 3H\sum_{t'=t}^H\E\brs*{ \br*{(\tilde r_{k-1} - r)(s_{t'}^k,a_{t'}^k)}^2+\br*{(\tilde p_{k-1} - p)(s_{t'}^k,a_{t'}^k)\bar V^{k-1}_{t+1}}^2 \mid \F_{k-1}}.
\end{align*}
}
Inequality $(1)$ is by Lemma \ref{lemma: FULL model base RL expected value difference}. $(2)$ is due to Jensen's inequality, and using the inequality $\br*{\sum_{i=1}^n a_i}^2\le n\sum_{i=1}^n a_i^2$, and  $(3)$ is by the tower property. 

Plugging this back into \eqref{eq: squared difference lemma first eq},
\begin{align}
    \eqref{eq: squared difference lemma first eq} &\leq   3H\sum_{k=1}^{K}\sum_{t=1}^{H}\sum_{t'=t}^H\E\brs*{ \br*{\bar{V}^{k-1}(s_{t'}^k) - \bar{V}^{k}(s_{t'}^k)}^2  \mid \F_{k-1}} \nonumber \\
    &+ 3H\sum_{k=1}^{K}\sum_{t=1}^{H}\sum_{t'=t}^H \E\brs*{\br*{(\tilde r_{k-1} - r)(s_{t'}^k,a_{t'}^k)}^2+\br*{(\tilde p_{k-1} - p)(s_{t'}^k,a_{t'}^k)\bar V^{k-1}_{t+1}}^2 \mid \F_{k-1}} \nonumber\\
    & \leq   3H^2\sum_{k=1}^{K}\sum_{t=1}^H\underset{(*)}{\underbrace{\E\brs*{ \br*{\bar{V}^{k-1}(s_t^k) - \bar{V}^{k}(s_t^k)}^2  \mid \F_{k-1}}} \nonumber} \\
    &+ 3H^2\sum_{k=1}^{K}\sum_{t=1}^{H} \E\brs*{\underset{(**)}{\underbrace{\br*{(\tilde r_{k-1} - r)(s_t^k,a_{t}^k)}^2}}+ \underbrace{\br*{(\tilde p_{k-1} - p)(s_t^k,a_{t}^k)\bar V^{k-1}_{t+1}}^2}_{(***)} \mid \F_{k-1}}. \label{eq:squared difference decomposition}
\end{align}

We now bound each term of the above. First, we have that
\begin{align*}
    &(*)=\sum_{t=1}^{H}\sum_{t=1}^H\E\brs*{ \br*{\bar{V}^{k-1}(s_t^k) - \bar{V}^{k}(s_t^k)}^2  \mid \F_{k-1}}  \\
    & =\sum_{t=1}^{H}\sum_{t=1}^H\E\brs*{ \br*{\bar{V}^{k-1}(s_t^k)}^2 + \br*{\bar{V}^{k}(s_t^k)}^2 - 2\bar{V}^{k}(s_t^k)\bar{V}^{k-1}(s_t^k)   \mid \F_{k-1}}  \\
    & \stackrel{(1)}{\leq} \sum_{t=1}^{H}\sum_{t=1}^H\E\brs*{ \br*{\bar{V}^{k-1}(s_t^k)}^2 + \br*{\bar{V}^{k}(s_t^k)}^2 - 2\br*{\bar{V}^{k}(s_t^k)}^2   \mid \F_{k-1}} \\
    & = \sum_{t=1}^{H}\sum_{t=1}^H\E\brs*{ \br*{\bar{V}^{k-1}(s_t^k)}^2 - \br*{\bar{V}^{k}(s_t^k)}^2 \mid \F_{k-1}} \\
    & \stackrel{(2)}{=} \sum_{t=1}^{H}\sum_{t=1}^H\sum_{s} \br*{\bar{V}^{k-1}(s)}^2 - \E\brs*{\br*{\bar{V}^{k}(s}^2 \mid \F_{k-1}}. 
\end{align*}

Relation $(1)$ holds since $0\le\bar V^{k} \leq \bar V^{k-1}$ (see Lemma \ref{lemma: 3 properties of decreasing process of EULER}). $(2)$ is proven similarly to Lemma \ref{lemma: regret to SH decreasing processes} (Appendix \ref{sec: supp general lemmas}), as follows
\begin{align*}
    &\sum_{k=1}^{K}\sum_{t=1}^{H} \E\brs*{ \br*{\bar{V}^{k-1}(s_t^k)}^2 - \br*{\bar{V}^k(s_t^k)}^2 \mid \mathcal{F}_{k-1}}\\
    & \stackrel{(1)}{=}\sum_{k=1}^{K}\sum_{t=1}^{H}\sum_s \E[  \ind\brc*{s_t^k= s}\br*{\bar{V}^{k-1}(s)}^2 -  \ind\brc*{s_t^k= s}\br*{\bar{V}^k(s)}^2 \mid \mathcal{F}_{k-1}]\\
    & \stackrel{(2)}{=}\sum_{k=1}^{K}\sum_{t=1}^{H}\sum_s \E[  \ind\brc*{s_t^k= s}\br*{\bar{V}^{k-1}(s)}^2 + \ind\brc*{s_t^k\neq s}\br*{\bar{V}^{k-1}(s)}^2\mid\mathcal{F}_{k-1}]\\
    &\qquad\qquad\qquad-\E[\ind\brc*{s_t^k= s}\br*{\bar{V}^k(s)}^2 + \ind\brc*{s_t^k\neq s}\br*{\bar{V}^{k-1}(s)}^2    \mid \mathcal{F}_{k-1}]\\
    & \stackrel{(3)}{=}\sum_{k=1}^{K}\sum_{t=1}^{H}\sum_s \br*{\bar{V}^{k-1}(s)}^2 - \E[\br*{\bar{V}^k(s)}^2 \mid \mathcal{F}_{k-1}]\\
\end{align*}
$(1)$ holds by adding and subtracting $\ind\brc*{s \neq s_t^k}\bar{V}_t^{k-1}(s)$ while using the linearity of expectation. $(2)$ holds since for any event $\ind\brc{A}+\ind\brc{A^c}=1$ and since $\Delta V_t^{k-1}$ is $\F_{k-1}$ measurable. $(3)$ holds by the definition of the update rule. If state $s$ is visited in the $k^{th}$ episode at time-step $t$, then both $\bar{V}^k_t(s),\underline{V}^k_t(s)$ are updated. If not, their value remains as in the $k-1$ iteration.

Next, by Lemma \ref{lemma: 3 properties of decreasing process of EULER} for a fixed $s,t$, $\brc*{\bar V^{k}_t(s)}_{k\geq 0}$ is a Decreasing Bounded Process in $[0,H^2]$. Applying Lemma \ref{lemma: sum of decreasing processes} we conclude that
\begin{align*}
    (*) \le \sum_{k=1}^{K}\sum_{t=1}^{H}\sum_s \br*{\bar{V}^{k-1}(s)}^2 - \E[\br*{\bar{V}^k(s)}^2 \mid \mathcal{F}_{k-1}]\lesssim \Olog( H^3S).
\end{align*}

We now turn to bound $(**)$.
\begin{align*}
    &\sum_{k=1}^{K}\sum_{t=1}^{H} \E\brs*{\br*{(\tilde r_{k-1} - r)(s_t^k,a_{t}^k)}^2\mid \F_{k-1}}\\
    &\stackrel{(1)}{\le} 2\sum_{k=1}^{K}\sum_{t=1}^{H} \E\brs*{\br*{(\hat r_{k-1} - r)(s_t^k,a_{t}^k)}^2\mid \F_{k-1}}
    + 2\sum_{k=1}^{K}\sum_{t=1}^{H} \E\brs*{\br*{b_k^r(s_t^k,a_{t}^k)}^2\mid \F_{k-1}}\\
    &\stackrel{(2)}{\le} 4\sum_{k=1}^{K}\sum_{t=1}^{H} \E\brs*{\br*{b_k^r(s_t^k,a_{t}^k)}^2\mid \F_{k-1}}\\
    &= 4\sum_{k=1}^{K}\sum_{t=1}^{H} \E\brs*{\br*{\sqrt{\frac{2\hat{\VAR}(R(s_t^k,a_t^k))\ln\frac{4SAT}{\delta'}}{n_{k-1}(s_t^k,a_t^k)\vee 1}} + \frac{14\ln \frac{4SAT}{\delta'}}{3n_{k-1}(s_t^k,a_t^k)\vee 1}}^2\mid \F_{k-1}}\\
    &\stackrel{(3)}{\lesssim} \sum_{k=1}^{K}\sum_{t=1}^{H} \E\brs*{\frac{1}{n_{k-1}(s_t^k,a_t^k) \vee 1} \mid \F_{k-1} } \\
    &\stackrel{(4)}{\lesssim} \Olog (SAH).
\end{align*}
In $(1)$, we used the definition of $\tilde{r}_{k-1}$ and the inequality $(a+b)^2\le 2a^2+2b^2$. $(2)$ is since outside the failure event $F^r$, $(\hat r_{k-1} - r)(s,a)\le b_k^r(s,a)$. $(3)$ uses the fact that $R(s,a)\in[0,1]$, and thus $\hat{\VAR}(R(s,a)\le 1$, and $\sqrt{n}\le n$ for $n\ge 1$. Finally, $(4)$ is due to Lemma \ref{lemma: sum of 1 over n}.

Lastly, we bound $(***)$.
\begin{align*}
   &\sum_{k=1}^{K}\sum_{t=1}^{H} \E\brs*{ \br*{(\tilde p_{k-1} - p)(s_t^k,a_{t}^k)^T\bar V^{k-1}_{t+1}}^2  \mid \F_{k-1} }\\
   &=  \sum_{k=1}^{K}\sum_{t=1}^{H} \E\brs*{ \br*{(\hat p_{k-1} - p)(s_t^k,a_{t}^k)^T\bar V^{k-1}_{t+1} + b_{k}^{pv}(s_t^k,a_t^k)}^2  \mid \F_{k-1} }\\
    &\stackrel{(1)}{\leq}  2 \sum_{k=1}^{K}\sum_{t=1}^{H} \E\brs*{ \br*{(\hat p_{k-1} - p)(s_t^k,a_{t}^k)^T\bar V^{k-1}_{t+1}}^2 + \br*{b_{k}^{pv}(s_t^k,a_t^k)}^2   \mid \F_{k-1} }\\
    &\stackrel{(2)}{\leq}  2 \sum_{k=1}^{K}\sum_{t=1}^{H} \E\brs*{ \br*{\norm{\hat p_{k-1} - p}_{1}\norm{\bar V^{k-1}_{t+1}}_{\infty}}^2 + \br*{b_{k}^{pv}(s_t^k,a_t^k)}^2   \mid \F_{k-1} }\\
    &\stackrel{(3)}{\leq}  2 \sum_{k=1}^{K}\sum_{t=1}^{H} \E\brs*{ H^2\norm{\hat p_{k-1} - p}_{1}^2 + \br*{b_{k}^{pv}(s_t^k,a_t^k)}^2   \mid \F_{k-1} }\\
    &\stackrel{(4)}{\lesssim}  \sum_{k=1}^{K}\sum_{t=1}^{H} \E\brs*{\frac{H^2S}{n_{k-1}(s_t^k,a_t^k)} + \frac{\br*{2B_vH+5J+B_p}^2}{n_{k-1}(s_t^k,a_t^k)}   \mid \F_{k-1} }.\\
    & \stackrel{(5)}{\lesssim} \Olog(SAH(F+D)^2)
\end{align*}

Similarly to the bound on the reward, $(1)$ uses the inequality $(a+b)^2\le2a^2+2b^2$. Inequality $(2)$ is due to H\"older's inequality, and $(3)$ bounds $\norm{\bar V^{k-1}_{t+1}}_{\infty}\le H$, which is due to Lemma \ref{lemma: 3 properties of decreasing process of EULER}. Next, $(4)$ bounds the transition error outside to failure event $F^{pn1}$ and $b_{k}^{pv}$ according to Lemma \ref{lemma: bkv bound}. Finally, $(5)$ is by Lemma \ref{lemma: sum of 1 over n} and noting that $H^2S + (2B_vH+5J+B_p)^2\lesssim (F+D)^2$.

Substituting all of the results into \eqref{eq:squared difference decomposition}, and remembering the $H^2$ factor in this equation, gives the desired result.

\end{proof}

\subsection{Bounding Different Terms in the Regret Decomposition} \label{sec:euler regret decomp bounds}

In this section, we bound each of the individual terms of the regret decomposition (Equation \ref{eq: regert of euler paper}), relaying on results from \citep{zanette2019tighter}, as well as on the new lemmas derived in Section~\ref{sec:helpful lemma proof}, Lemma~\ref{lemma: central lemma of EULER analysis} and Lemma~\ref{lemma: squared difference upper to real}. First, we present the problem dependent constants of \citep{zanette2019tighter} for general admissible confidence intervals, and their relation to problem dependent constants with Bernstein's inequality

\begin{lemma} \label{lemma: bernstein problem dependent}
Let $\mathbb{C}^*$ and $\mathbb{C}^{\pi}$ be upper dependent bounds on the following qualities:
\begin{align*}
    &\mathbb{C}^* \ge\frac{1}{T}\sum_{k=1}^K\sum_{t=1}^H\sum_{s,a} w_{tk}(s,a)g(p,V^*_{t+1})^2 \\
    & \mathbb{C}^{\pi} \ge\frac{1}{T}\sum_{k=1}^K\sum_{t=1}^H\sum_{s,a} w_{tk}(s,a)g(p,V^{\pi_k}_{t+1})^2 \enspace,
\end{align*}
with $g(p,V) = \sqrt{2\VAR_{s'\sim p(\cdot\mid s,a)}V(s')\ln\frac{2SAT}{\delta'}}$, and let 
\begin{align*}
\mathbb{C}_r^* = \frac{1}{T}\br*{\sum_{k=1}^K\sum_{t=1}^H  \sum_{(s,a)\in L_k} w_{tk}(s,a)\VAR R(s,a)}\enspace ,
\end{align*}
where $L_k$ is defined in Definition \ref{defn:good set}. 
Finally, let ${\mathbb{Q}^*\eqdef \max_{s,a,t} \br*{\VAR{R(s,a)+\VAR_{s'\sim p(\cdot\mid s,a)}V^*_{t+1}(s')}}}$. Then,
\begin{align*}
    &\mathbb{C}_r^* + \mathbb{C}^* \lesssim \mathbb{Q}^* \\
    & \mathbb{C}^{\pi} \lesssim \frac{\mathcal{G}^2}{H} \\
    & \mathbb{C}_r^* \le \frac{\mathcal{G}^2}{H} 
\end{align*}
\end{lemma}
\begin{proof}
We follow proposition 6 of \citep{zanette2019tighter}, and start by substituting $g(p,V)$ into $\mathbb{C}_r^* + \mathbb{C}^*$
\begin{align*}
    \mathbb{C}_r^* + \mathbb{C}^*
    &\lesssim \frac{1}{T}\br*{\sum_{k=1}^K\sum_{t=1}^H  \sum_{(s,a)\in L_k} w_{tk}(s,a)\VAR R(s,a)} \\
    & \quad + \frac{1}{T}\sum_{k=1}^K\sum_{t=1}^H\sum_{s,a} w_{tk}(s,a)\VAR_{s'\sim p(\cdot\mid s,a)}V^*_{t+1}(s') \\
    & \le  \frac{1}{T}\br*{\sum_{k=1}^K\sum_{t=1}^H  \sum_{(s,a)} w_{tk}(s,a)\br*{\VAR R(s,a)} +\VAR_{s'\sim p(\cdot\mid s,a)}V^*_{t+1}(s')}\\
    & \le \frac{1}{T}\br*{\sum_{k=1}^K\sum_{t=1}^H  \sum_{(s,a)} w_{tk}(s,a)\max_{s,a,t}\brc*{\VAR R(s,a)} +\VAR_{s'\sim p(\cdot\mid s,a)}V^*_{t+1}(s')}\\
    & =  \frac{\mathbb{Q}^*}{T}\br*{\sum_{k=1}^K\sum_{t=1}^H  \sum_{(s,a)} w_{tk}(s,a)} \\
    & = \mathbb{Q}^*
\end{align*}
where the last equality is since $\sum_{(s,a)} w_{tk}(s,a)=1$ and $T=HK$. 

Next, we bound $\mathbb{C}^{\pi}$:
\begin{align*}
    \mathbb{C}^{\pi}  
    &\lesssim \frac{1}{T}\sum_{k=1}^K\sum_{t=1}^H\sum_{s,a} w_{tk}(s,a)\VAR_{s'\sim p(\cdot\mid s,a)}V^{\pi_k}_{t+1}(s') \\
    &\stackrel{(1)}{=} \frac{1}{T}\sum_{k=1}^K\E\brs*{\br*{\sum_{t=1}^Hr(s_t^k,a_t^k) - V^{\pi_k}_1(s_1^k)}^2 \mid \F_{k-1}} \\
    & \le\frac{1}{T}\sum_{k=1}^K\E\brs*{\br*{\sum_{t=1}^Hr(s_t^k,a_t^k)}^2 \mid \F_{k-1}} \\
    & \stackrel{(2)}{\le} \frac{1}{T}K \mathcal{G}^2 
    = \frac{\mathcal{G}^2}{H} \enspace,
\end{align*}
where $(1)$ is due to the Law of Total Variance (LTV), which was used in \citep{azar2017minimax}, and was stated formally in Lemma 15 of \citep{zanette2019tighter}. In $(2)$, we bound the reward in an episode by $\mathcal{G}$.

Finally, the bound on $\mathbb{C}_r^*$ is proven in Lemma 8 of \citep{zanette2019tighter}, which concludes this proof.

\end{proof}

We also prove the following lemma that helps translating bounds that depend on $\mathbb{C}^*$ to bounds that depends on $\mathbb{C}^\pi$. This lemma is equivalent to lemma 14 of \citep{zanette2019tighter}, but the prove requires Lemma \ref{lemma: squared difference upper to real}, that was not proved in their paper. This is since they rely on the inequality $\underline{V}_t^{k-1}\le V^{\pi_k}$, which does not seem to hold.

\begin{lemma}[Bound Translation Lemma] \label{lemma: bound translation}
Outside the failure event, it holds that 

\begin{align*}
    \sum_{k=1}^K\sum_{t=1}^H  \sum_{(s,a)\in L_k} &w_{tk}(s,a)\frac{g(p,V_{t+1}^*)}{\sqrt{n_{k-1}(s,a) \vee 1}} 
    - \sum_{k=1}^K\sum_{t=1}^H  \sum_{(s,a)\in L_k} w_{tk}(s,a)\frac{g(p,V_{t+1}^{\pi_k})}{\sqrt{n_{k-1}(s,a) \vee 1}} \\
    &= \Olog\br*{B_vSAH^{\frac{3}{2}}(F+D)+B_vSAH^{\frac{5}{2}}}
\end{align*}

where $F,D$ are defined in Lemma \ref{lemma: next episode value function differences}.
\end{lemma}

\begin{proof}
We start as in the original Lemma 14 of \citep{zanette2019tighter}:
\begin{align*}
    \sum_{k=1}^K\sum_{t=1}^H  &\sum_{(s,a)\in L_k} w_{tk}(s,a)\frac{g(p,V_{t+1}^*)}{\sqrt{n_{k-1}(s,a) \vee 1}} 
    - \sum_{k=1}^K\sum_{t=1}^H  \sum_{(s,a)\in L_k} w_{tk}(s,a)\frac{g(p,V_{t+1}^{\pi_k})}{\sqrt{n_{k-1}(s,a) \vee 1}} \\
    &\stackrel{(1)}{\le} B_v\sum_{k=1}^K\sum_{t=1}^H  \sum_{(s,a)\in L_k} w_{tk}(s,a)\frac{\norm*{V_{t+1}^*-V_{t+1}^{\pi_k}}_{2,p}}{\sqrt{n_{k-1}(s,a) \vee 1}} \\
    & \stackrel{(2)}{\le} B_v\sqrt{\sum_{k=1}^K\sum_{t=1}^H  \sum_{(s,a)\in L_k} \frac{w_{tk}(s,a)}{n_{k-1}(s,a) \vee 1}}\sqrt{\sum_{k=1}^K\sum_{t=1}^H  \sum_{(s,a)\in L_k} w_{tk}(s,a)\norm*{V_{t+1}^*-V_{t+1}^{\pi_k}}_{2,p}^2} \\
    & \stackrel{(3)}{\lesssim} B_v \sqrt{SA} \sqrt{\sum_{k=1}^K\sum_{t=1}^H  \sum_{(s,a)} w_{tk}(s,a)p(\cdot \mid s,a)^T\br*{\bar{V}^{k-1}_{t+1} - V^{\pi_k}_{t+1}}^2}
\end{align*}

where in $(1)$ we use property \ref{def-prop:admiss definition} of Definition \ref{defn:admissible}, and $(2)$ is due to Cauchy-Schwarz inequality. In $(3)$ we used Lemma \ref{lemma: good set visitation count ratio}. Next, we apply Lemma \ref{lemma: squared difference upper to real} to bound the remaining term by $\Olog\br*{\sqrt{SAH^3(F+D)^2 + SAH^5}}$ and bound $\sqrt{SAH^3(F+D)^2 + SAH^5}\le \sqrt{SAH^3(F+D)^2}+\sqrt{SAH^5}$, which yields the desired result
\end{proof}

We are now ready to bound each of the terms of the regret. To bound the first term, we cite Lemma 8 of \citep{zanette2019tighter}:

\begin{lemma}[Optimistic Reward Bound] \label{lemma: optimistic reward bound}
Outside the failure event, it holds that 

\begin{align*}
    \sum_{k=1}^K\sum_{t=1}^H  \sum_{(s,a)\in L_k} w_{tk}(s,a)(\tilde{r}_{k-1} - r)(s_t^k,a_t^k) = \Olog\br*{\sqrt{\mathbb{C}_r^*SAT}+SA}
\end{align*}
\end{lemma}

The next three lemmas correspond to the remaining terms, and follow Lemmas 9,10 and 11 of \citep{zanette2019tighter}, with slight modifications:

\begin{lemma}[Empirical Transition Bound] \label{lemma: empirical transition bound}
Outside the failure event, it holds that 

\begin{align*}
    \sum_{k=1}^K\sum_{t=1}^H  \sum_{(s,a)\in L_k} w_{tk}(s,a)(\hat{p}_{k-1} - p_{k-1})(\cdot\mid s,a)^TV_{t+1}^* = \Olog\br*{\sqrt{\mathbb{C}^*SAT}+JSA}
\end{align*}
The following bound also holds:
\begin{align*}
    &\sum_{k=1}^K\sum_{t=1}^H  \sum_{(s,a)\in L_k} w_{tk}(s,a)(\hat{p}_{k-1} - p_{k-1})(\cdot\mid s,a)^TV_{t+1}^* \\
    &= \Olog\br*{\sqrt{\mathbb{C}^{\pi}SAT}+JSA + B_vSAH^{\frac{3}{2}}(F+D)+B_vSAH^{\frac{5}{2}}}
\end{align*}
where $F,D$ are defined in Lemma \ref{lemma: next episode value function differences}.
\end{lemma}

\begin{proof}
Similarly to Lemma 9 of \citep{zanette2019tighter}, by the definition of $\phi$ (Definition \ref{defn:admissible}), and outside failure event $F^{pv}$, 

\begin{align}
    &\sum_{k=1}^K\sum_{t=1}^H  \sum_{(s,a)\in L_k} w_{tk}(s,a)(\hat{p}_{k-1} - p_{k-1})(\cdot\mid s,a)^TV_{t+1}^* \nonumber\\
    &\le \sum_{k=1}^K\sum_{t=1}^H  \sum_{(s,a)\in L_k} w_{tk}(s,a) \br*{\frac{g(p,V_{t+1}^*)}{\sqrt{n_{k-1}(s,a) \vee 1}} + \frac{J}{n_{k-1}(s,a) \vee 1}} \label{eq:empirical transition reccurring}\\
    &\stackrel{(*)}{\le} \sqrt{\sum_{k=1}^K\sum_{t=1}^H\sum_{(s,a)\in L_k} w_{tk}(s,a) g(p,V_{t+1}^*)^2 } \sqrt{\sum_{k=1}^K\sum_{t=1}^H\sum_{(s,a)\in L_k}  \frac{w_{tk}(s,a)}{n_{k-1}(s,a) \vee 1} } \nonumber \\
    &\quad +  J\sum_{k=1}^K\sum_{t=1}^H  \sum_{(s,a)\in L_k} \frac{w_{tk}(s,a)}{n_{k-1}(s,a) \vee 1} \nonumber
\end{align}
where the last inequality is by Cauchy-Schwarz Inequality. Substituting the definition of $\mathbb{C}^*$, and using Lemma \ref{lemma: good set visitation count ratio}, we get 
\begin{align*}
    &\lesssim \sqrt{T\mathbb{C}^*} \sqrt{SA} +  JSA \enspace,
\end{align*}
which concludes the first statement of the lemma. To get the second statement, we apply Lemma \ref{lemma: bound translation} before inequality $(*)$ and only then use Cauchy-Schwarz Inequality. This creates the additional constant term of $\Olog\br*{B_vSAH^{\frac{3}{2}}(F+D)+B_vSAH^{\frac{5}{2}}}$. Then, by applying Lemma \ref{lemma: good set visitation count ratio}, we get the bound with $\mathbb{C}^{\pi}$.
\end{proof}

\clearpage 
\begin{lemma}[Lower Order Term] \label{lemma: lower order}
Let $F,D$ be the constants defined in Lemma \ref{lemma: next episode value function differences}. Outside the failure event, it holds that 

\begin{align*}
    \sum_{k=1}^K\sum_{t=1}^H  \sum_{(s,a)\in L_k} &w_{tk}(s,a)\abs*{(\hat{p}_{k-1} - p)(\cdot\mid s,a)^T(\bar{V}_{t+1}^{k-1}-V_{t+1}^*)} \\
    &= \Olog\br*{S^{\frac{3}{2}}AH(F+D+H^{\frac{3}{2}})+S^2AH}
\end{align*}
\end{lemma}

\begin{proof}
Similarly to Lemma 11 of \citep{zanette2019tighter}, by the definition of $\phi$ (Definition \ref{defn:admissible}), and outside failure event $F^{ps}$, 

\begin{align*}
    &\sum_{k=1}^K\sum_{t=1}^H  \sum_{(s,a)\in L_k} w_{tk}(s,a)\abs*{(\hat{p}_{k-1} - p)(\cdot\mid s,a)^T(\bar{V}_{t+1}^{k-1}-V_{t+1}^*)} \\
    &\lesssim  \sum_{k=1}^K\sum_{t=1}^H  \sum_{(s,a)\in L_k} w_{tk}(s,a) \sum_{s'}\sqrt{\frac{p(s'\mid s,a)(1-p(s'\mid s,a))}{n_{k-1}(s,a)\vee 1}}\abs*{\bar{V}_{t+1}^{k-1}(s')-V_{t+1}^*(s')} \\
    &\quad+ \sum_{k=1}^K\sum_{t=1}^H  \sum_{(s,a)\in L_k} w_{tk}(s,a) \sum_{s'}\frac{\abs*{\bar{V}_{t+1}^{k-1}(s')-V_{t+1}^*(s')}}{n_{k-1}(s,a)\vee 1} \\
    &\le  \sum_{k=1}^K\sum_{t=1}^H  \sum_{(s,a)\in L_k} w_{tk}(s,a) \sum_{s'}\sqrt{\frac{p(s'\mid s,a)(1-p(s'\mid s,a))}{n_{k-1}(s,a)\vee 1}}\abs*{\bar{V}_{t+1}^{k-1}(s')-V_{t+1}^*(s')} \\
    &\quad+ \sum_{k=1}^K\sum_{t=1}^H  \sum_{(s,a)\in L_k} w_{tk}(s,a) \frac{HS}{n_{k-1}(s,a)\vee 1} \enspace,
\end{align*}
where in the last inequality we used the fact that $V_{t+1}^*$ and $\bar{V}_{t+1}^{k-1}$ are in $[0,H]$, by Lemma \ref{lemma: 3 properties of decreasing process of EULER}. Next, using the optimism of the value $\underline{V}_{t+1}^{k-1}\le V_{t+1}^*\le\bar{V}_{t+1}^{k-1}$ (Lemma \ref{lemma: optimism EULER}), and since $(1-p)\le1$ for $p\in[0,1]$, we can bound 
{\small
\begin{align*}
    & \le   \sum_{k=1}^K\sum_{t=1}^H  \sum_{(s,a)\in L_k} w_{tk}(s,a) \sum_{s'}\sqrt{\frac{p(s'\mid s,a)}{n_{k-1}(s,a)\vee 1}}\abs*{\bar{V}_{t+1}^{k-1}(s')-\underline{V}_{t+1}^{k-1}(s')} \\
    &\quad+ HS\sum_{k=1}^K\sum_{t=1}^H  \sum_{(s,a)\in L_k}  \frac{w_{tk}(s,a)}{n_{k-1}(s,a)\vee 1} \\
    & \stackrel{(CS)}{\le}   \sum_{k=1}^K\sum_{t=1}^H  \sum_{(s,a)\in L_k} w_{tk}(s,a) \sqrt{\frac{Sp(\cdot\mid s,a)^T\br*{\bar{V}_{t+1}^{k-1}-\underline{V}_{t+1}^{k-1}}^2}{n_{k-1}(s,a)\vee 1}} \\
    &\quad+ HS\sum_{k=1}^K\sum_{t=1}^H  \sum_{(s,a)\in L_k}  \frac{w_{tk}(s,a)}{n_{k-1}(s,a)\vee 1} \\
    & \stackrel{(CS)}{\le}  \sqrt{S} \sqrt{\sum_{k=1}^K\sum_{t=1}^H  \sum_{(s,a)\in L_k}  \frac{w_{tk}(s,a)}{n_{k-1}(s,a)\vee 1} }
    \sqrt{\sum_{k=1}^K\sum_{t=1}^H  \sum_{(s,a)\in L_k} w_{tk}(s,a) p(\cdot\mid s,a)^T\br*{\bar{V}_{t+1}^{k-1}-\underline{V}_{t+1}^{k-1}}^2}\\
    &\quad+ HS\sum_{k=1}^K\sum_{t=1}^H  \sum_{(s,a)\in L_k}  \frac{w_{tk}(s,a)}{n_{k-1}(s,a)\vee 1} \\
    &\stackrel{(*)}{\lesssim} \sqrt{S}\sqrt{SA}\sqrt{SAH^2(F+D)^2 + SAH^5 } + SH\cdot SA \\
    & = \Olog \br*{S^{\frac{3}{2}}AH(F+D+H^{\frac{3}{2}})+S^2AH}
\end{align*}
}
$(CS)$ denotes Cauchy-Schwarz. Specifically, the first inequality uses $\sum_{i=1}^n a_ib_i \le \sqrt{n\sum_{i=1}^n a_i^2b_i^2}$. In $(*)$, we used Lemmas \ref{lemma: good set visitation count ratio} and \ref{lemma: central lemma of EULER analysis}.
\end{proof}

\begin{lemma}[Optimistic Transition Bound] \label{lemma: optimistic transition bound}
Let $F,D$ be the constants defined in Lemma \ref{lemma: next episode value function differences}. Outside the failure event, it holds that 
{\small
\begin{align*}
    &\sum_{k=1}^K\sum_{t=1}^H  \sum_{(s,a)\in L_k} w_{tk}(s,a)(\tilde{p}_{k-1} - \hat{p}_{k-1})(\cdot\mid s,a)^T\bar{V}_{t+1}^{k-1} \\
    & = \Olog\br*{\sqrt{\mathbb{C}^*SAT} + (J+B_p)SA + B_vSAH\br*{F+D + H^{\frac{3}{2}}} + B_vSA\sqrt{S^{\frac{1}{2}}H(F+D+H^{\frac{5}{2}})+SH^2}} 
\end{align*}
}
The following bound also holds:
{\small
\begin{align*}
    &\sum_{k=1}^K\sum_{t=1}^H  \sum_{(s,a)\in L_k} w_{tk}(s,a)(\hat{p}_{k-1} - \hat{p}_{k-1})(\cdot\mid s,a)^TV_{t+1}^* \\
    &= \Olog\br*{\sqrt{\mathbb{C}^{\pi}SAT} + (J+B_p)SA + B_vSAH^{\frac{3}{2}}\br*{F+D + H} + B_vSA\sqrt{S^{\frac{1}{2}}H(F+D+H^{\frac{5}{2}})+SH^2}}
\end{align*}
}
\end{lemma}

\begin{proof}
Similarly to Lemma 10 of \citep{zanette2019tighter}, by the definition of the bonus, 
{\small
\begin{align*}
    &\sum_{k=1}^K\sum_{t=1}^H  \sum_{(s,a)\in L_k} w_{tk}(s,a)(\tilde{p}_{k-1} - \hat{p}_{k-1})(\cdot\mid s,a)^T\bar{V}_{t+1}^{k-1} \\
    &= \sum_{k=1}^K\sum_{t=1}^H  \sum_{(s,a)\in L_k} w_{tk}(s,a)b_k^{pv}(s,a) \\
    & = \sum_{k=1}^K\sum_{t=1}^H  \sum_{(s,a)\in L_k} w_{tk}(s,a)\br*{\phi(\hat{p}_{k-1}(\cdot \mid s,a),\underline{V}^{k-1}_{t+1})  + \frac{B_v\norm{ \bar{V}^{k-1}_{t+1} - \underline{V}^{k-1}_{t+1}  }_{2,\hat{p}}}{\sqrt{n_{k-1}(s,a)\vee 1}}
     + \frac{4J+B_p}{n_{k-1}(s,a)\vee 1}} \\
    & \le \sum_{k=1}^K\sum_{t=1}^H  \sum_{(s,a)\in L_k} w_{tk}(s,a)\br*{\phi(p(\cdot \mid s,a),V^*)  + 2\frac{B_v\norm{ \bar{V}^{k-1}_{t+1} - \underline{V}^{k-1}_{t+1}  }_{2,\hat{p}}}{\sqrt{n_{k-1}(s,a)\vee 1}}
     + 2\frac{4J+B_p}{n_{k-1}(s,a)\vee 1}} \enspace . 
\end{align*}
}

In the last inequality, we applied Lemma \ref{lemma:admiss properties}, Property(\ref{lemma-part:bonus phi bound}), and used Equation \eqref{eq:value norm bound} together with the optimism of the value function, that is $\underline{V}^{k-1}_{t+1}\le V_{t+1}^* \le \bar{V}^{k-1}_{t+1}$ (Lemma \ref{lemma: 3 properties of decreasing process of EULER}). Next, we substitute the definition of $\phi$ (Definition \ref{defn:admissible}), and get 

\begin{align}
    & \lesssim \sum_{k=1}^K\sum_{t=1}^H  \sum_{(s,a)\in L_k} w_{tk}(s,a)\br*{\frac{g(p,V_{t+1}^*)}{n_{k-1}(s,a)\vee 1}  + \frac{J+B_p}{n_{k-1}(s,a)\vee 1}} \label{eq:transition first term}\\
    &\quad + \sum_{k=1}^K\sum_{t=1}^H  \sum_{(s,a)\in L_k} w_{tk}(s,a)\frac{B_v\norm{ \bar{V}^{k-1}_{t+1} - \underline{V}^{k-1}_{t+1}  }_{2,\hat{p}}}{\sqrt{n_{k-1}(s,a)\vee 1}} \label{eq:transition second term} \enspace . 
\end{align}

The term in Equation \eqref{eq:transition first term} is almost identical to Equation \eqref{eq:empirical transition reccurring} of Lemma \ref{lemma: empirical transition bound}, and can be similarly bounded by replacing $J$ with $J+B_p$. This yields a bound of either $\Olog\br*{\sqrt{\mathbb{C}^*SAT}+(J+B_p)SA}$ or $\Olog\br*{\sqrt{\mathbb{C}^{\pi}SAT}+(J+B_p)SA + B_vSAH^{\frac{3}{2}}(F+D)+B_vSAH^{\frac{5}{2}}}$. We now move to bounding the second term. Notice that

\begin{align}
    \norm{ \bar{V}^{k-1}_{t+1} - \underline{V}^{k-1}_{t+1}  }_{2,\hat{p}}^2 
    & = \hat{p}_{k-1}(\cdot\mid s,a)^T\br*{\bar{V}_{t+1}^{k-1}-\underline{V}_{t+1}^{k-1}}^2 \nonumber\\
    & = p(\cdot\mid s,a)^T\br*{\bar{V}_{t+1}^{k-1}-\underline{V}_{t+1}^{k-1}}^2
    + (\hat{p}_{k-1}-p)(\cdot\mid s,a)^T\br*{\bar{V}_{t+1}^{k-1}-\underline{V}_{t+1}^{k-1}}^2 \nonumber\\
    & = \norm{ \bar{V}^{k-1}_{t+1} - \underline{V}^{k-1}_{t+1}  }_{2,p}^2 
    + (\hat{p}_{k-1}-p)(\cdot\mid s,a)^T\br*{\bar{V}_{t+1}^{k-1}-\underline{V}_{t+1}^{k-1}}^2 \label{eq:value diff norm}
\end{align}

Next, applying Cauchy-Schwartz Inequality on \eqref{eq:transition second term}, we get
\begin{align*}
    \eqref{eq:transition second term} 
    &\le B_v\sqrt{\sum_{k=1}^K\sum_{t=1}^H  \sum_{(s,a)\in L_k} \frac{w_{tk}(s,a)}{n_{k-1}(s,a)\vee 1}}
    \sqrt{\sum_{k=1}^K\sum_{t=1}^H  \sum_{(s,a)\in L_k} w_{tk}(s,a)\norm{ \bar{V}^{k-1}_{t+1} - \underline{V}^{k-1}_{t+1}  }_{2,\hat{p}}^2} \\
    & \lesssim B_v\sqrt{SA}\sqrt{\sum_{k=1}^K\sum_{t=1}^H  \sum_{(s,a)\in L_k} w_{tk}(s,a)\norm{ \bar{V}^{k-1}_{t+1} - \underline{V}^{k-1}_{t+1}  }_{2,p}^2} \\
    & \quad + B_v\sqrt{SA}\sqrt{\sum_{k=1}^K\sum_{t=1}^H  \sum_{(s,a)\in L_k} w_{tk}(s,a)\abs*{(\hat{p}_{k-1}-p)(\cdot\mid s,a)^T\br*{\bar{V}_{t+1}^{k-1}-\underline{V}_{t+1}^{k-1}}^2}}\enspace,
\end{align*}
where the last inequality is by Lemma \ref{lemma: good set visitation count ratio}, substituting \eqref{eq:value diff norm} and using the inequality $\sqrt{a+b}\le \sqrt{a}+\sqrt{b}$. The first term can be directly bounded by Lemma \ref{lemma: central lemma of EULER analysis}. The second term can be bounded using Lemma \ref{lemma: lower order} as follows:
\begin{align*}
    &\sum_{k=1}^K\sum_{t=1}^H  \sum_{(s,a)\in L_k} w_{tk}(s,a)\abs*{(\hat{p}_{k-1}-p)(\cdot\mid s,a)^T\br*{\bar{V}_{t+1}^{k-1}-\underline{V}_{t+1}^{k-1}}^2} \\
    & \le H\sum_{k=1}^K\sum_{t=1}^H  \sum_{(s,a)\in L_k} w_{tk}(s,a)\abs*{(\hat{p}_{k-1}-p)(\cdot\mid s,a)^T\br*{\bar{V}_{t+1}^{k-1}-\underline{V}_{t+1}^{k-1}}} \\
    & = \Olog\br*{S^{\frac{3}{2}}AH(F+D+H^{\frac{5}{2}})+S^2AH^2} \enspace,
\end{align*}
where we trivially bounded the value difference by $H$ at the first inequality (due to Lemma \ref{lemma: 3 properties of decreasing process of EULER}) and used Lemma \ref{lemma: lower order} at the second one. Summing both terms yields
\begin{align*}
    \eqref{eq:transition second term} 
    &= \Olog\br*{B_v\sqrt{SA}\sqrt{SAH^2(F+D)^2 + SAH^5} + B_v\sqrt{SA}\sqrt{S^{\frac{3}{2}}AH(F+D+H^{\frac{5}{2}})+S^2AH^2}} \\
    & = \Olog\br*{B_vSAH\br*{F+D + H^{\frac{3}{2}}} + B_vSA\sqrt{S^{\frac{1}{2}}H(F+D+H^{\frac{5}{2}})+SH^2}} \\
\end{align*}

Combining both bounds on \eqref{eq:transition first term} and \eqref{eq:transition second term} concludes the proof.
\end{proof}

\clearpage

\section{General Lemmas} \label{sec: supp general lemmas}

\begin{lemma}[On Trajectory Regret to Sum of Decreasing Bounded Processes Regret]\label{lemma: regret to SH decreasing processes}
For Algorithm~\ref{algo: RTDP} and Algorithm~\ref{alg: general model based RL algorithm} it holds that,
$$\sum_{k=1}^K\sum_{t=1}^{H} \E[ \bar{V}_t^{k-1}(s_t^k) - \bar{V}_t^{k}(s_t^k)\mid \F_{k-1}] = \sum_{t=1}^{H}\sum_s \sum_{k=1}^K \bar{V}_t^{k-1}(s)- \E[\bar{V}_t^{k}(s)  \mid \F_{k-1}] $$
\end{lemma}
\begin{proof}
The following relations hold.
\begin{align}
     &\sum_{k=1}^K\sum_{t=1}^{H} \E[ \bar{V}_t^{k-1}(s_t^k) - \bar{V}_t^{k}(s_t^k)\mid \F_{k-1}]\\ 
     &=\sum_{k=1}^K\sum_{t=1}^{H}\sum_s \E[ \ind\brc*{s =s_t^k}\bar{V}_t^{k-1}(s) - \ind\brc*{s =s_t^k}\bar{V}_t^{k}(s)\mid \F_{k-1}] \nonumber\\
      &\stackrel{(1)}{=}\sum_{t=1}^{H}\sum_s \sum_{k=1}^K \E[\ind\brc*{s =s_t^k}\bar{V}_t^{k-1}(s) + \ind\brc*{s \neq s_t^k}\bar{V}_t^{k-1}(s)\mid \F_{k-1}]   \nonumber\\
      &\quad\quad\quad\quad\quad- \E[\ind\brc*{s =s_t^k}\bar{V}_t^{k}(s) +\ind\brc*{s \neq s_t^k}\bar{V}_t^{k-1}(s)  \mid \F_{k-1}] \nonumber\\
        &\stackrel{(2)}{=}\sum_{t=1}^{H}\sum_s \sum_{k=1}^K \bar{V}_t^{k-1}(s)- \E[\ind\brc*{s =s_t^k}\bar{V}_t^{k}(s) +\ind\brc*{s \neq s_t^k}\bar{V}_t^{k-1}(s)  \mid \F_{k-1}] \nonumber\\
      &\stackrel{(3)}{=}\sum_{t=1}^{H}\sum_s \sum_{k=1}^K \bar{V}_t^{k-1}(s)- \E[\bar{V}_t^{k}(s)  \mid \F_{k-1}].
\end{align}
Relation $(1)$ holds by adding and subtracting $\ind\brc*{s \neq s_t^k}\bar{V}_t^{k-1}(s)$ while using the linearity of expectation. $(2)$ holds since for any event $\ind\brc{A}+\ind\brc{A^c}=1$ and since $\Delta V_t^{k-1}$ is $\F_{k-1}$ measurable. $(3)$ holds by the definition of the update rule. If state $s$ is visited in the $k^{th}$ episode at time-step $t$, then both $\bar{V}^k_t(s),\underline{V}^k_t(s)$ are updated. If not, their value remains as in the $k-1$ iteration.
\end{proof}

\subsection{The Good Set \texorpdfstring{$L_k$}{Lk} and Few Lemmas}\label{sec: Lk definition}
We introduce that set $L_k$. The construction is similar to \citep{dann2017unifying} and we follow the one formulated in \citep{zanette2019tighter}. The idea is to partition the state-action space at each episode to two sets, the set of state-action pairs that have been visited sufficiently often, and the ones that were not. 

\begin{defn} \label{defn:good set}
The set $L_k$ is defined as follows.
\begin{align*}
    L_k \eqdef \brc*{(s,a)\in \mathcal{S}\times\mathcal{A}: \frac{1}{4}\sum_{j<k}w_j(s,a) \geq H\ln \frac{SAH}{\delta'}+H}
\end{align*}
where $w_j(s,a) \eqdef \sum_{t=1}^H w_{tj}(s,a)$
\end{defn}

We now state some useful lemmas. See proofs in  \citep{zanette2019tighter}, Lemma 6, Lemma 7, Lemma 13.
\begin{lemma}\label{lemma: good set visitation ratio 1}
Outside the failure event, it holds that if $(s,a)\in L_k$, then
\begin{align*}
    n_{k-1}(s,a)\geq \frac{1}{4}\sum_{j\geq k}w_j(s,a)\enspace,
\end{align*}
which also implies that $ n_{k-1}(s,a)\ge H\ln\frac{SAH}{\delta'}+H \ge 1$
\end{lemma}
\begin{lemma}\label{lemma: visitation outside good set} 
Outside the failure event, it holds that 
\begin{align*}
    \sum_{k=1}^K\sum_{t=1}^H\sum_{(s,a)\notin L_k}w_{tk}(s,a) \leq \Olog(SAH).
\end{align*}
\end{lemma}
\begin{lemma}\label{lemma: good set visitation count ratio}
Outside the failure event, it holds that 
\begin{align*}
        \sum_{k=1}^K\sum_{t=1}^H\sum_{(s,a)\in L_k} \frac{w_{tk}(s,a)}{n_{k-1}(s,a)} \leq \Olog(SA).
\end{align*}
\end{lemma}

Combining these lemmas we conclude the following one.
\begin{lemma}\label{lemma: supp 1 over sqrt n sum}
Outside the failure event, it holds that
$$\sum_{k=1}^K\sum_{t=1}^H \E\brs*{ \sqrt{\frac{1}{n_{k-1}(s_t^k,\pi_k(s_t^k))\vee 1}} \mid \F_{k-1} }\leq \Olog(\sqrt{SAT} +SAH)$$
\end{lemma}
\begin{proof}
The following holds relations hold.
\begin{align}
    &\sum_{k=1}^K\sum_{t=1}^H \E\brs*{ \sqrt{\frac{1}{n_{k-1}(s_t^k,\pi_k(s_t^k))\vee 1}} \mid \F_{k-1} }\nonumber \\
    &= \sum_{k=1}^K\sum_{t=1}^H\sum_{s,a} w_{tk}(s,a)\sqrt{\frac{1}{n_{k-1}(s,a)\vee 1}}\nonumber\\
    &\leq \sum_{k=1}^K\sum_{t=1}^H\sum_{s,a\in L_{k}} w_{tk}(s,a)\sqrt{\frac{1}{n_{k-1}(s,a)}} +  \sum_{k=1}^K\sum_{t=1}^H\sum_{s,a\notin \in L_{k}} w_{tk}(s,a) \nonumber\\
    &\le \sum_{k=1}^K\sum_{t=1}^H\sum_{s,a\in L_{k}} w_{tk}(s,a)\sqrt{\frac{1}{n_{k-1}(s,a)}} + SAH.\label{eq: supp general results 1 over sqrt n first eq}
    \end{align}
    The first relation holds by definition. The second relation holds by the following argument. For the first term, if $(s,a)\in L_k$ then by Lemma \ref{lemma: good set visitation ratio 1}, $n_{k-1}(s,a)\geq 1$, and thus ${n_{k-1}(s,a)\vee 1 = n_{k-1}(s,a)}$. The second term is bounded by taking the worst case for the fraction, which is $n_{k-1}(s,a)\vee 1\geq  1$. The third relation holds by Lemma \ref{lemma: visitation outside good set}.
    
   Consider the first term in \eqref{eq: supp general results 1 over sqrt n first eq}.
\begin{align*}
        &\sum_{k=1}^K\sum_{t=1}^H\sum_{s,a\in L_{k}} w_{tk}(s,a)\sqrt{\frac{1}{n_{k-1}(s,a)}}\\
        &\leq \sqrt{\sum_{k=1}^K\sum_{t=1}^H\sum_{s,a\in L_{k}} w_{tk}(s,a)}\sqrt{\sum_{k=1}^K\sum_{t=1}^H\sum_{s,a\in L_{k}}
        \frac{w_{tk}(s,a)}{n_{k-1}(s,a)}}\\
        &\leq \sqrt{\sum_{k=1}^K\sum_{t=1}^H\sum_{s,a} w_{tk}(s,a)}\sqrt{\sum_{k=1}^K\sum_{t=1}^H\sum_{s,a\in L_{k}}
        \frac{w_{tk}(s,a)}{n_{k-1}(s,a)}}\\
    &= \sqrt{T}\sqrt{\sum_{k=1}^K\sum_{t=1}^H\sum_{s,a} \frac{w_{tk}(s,a)}{n_{k-1}(s,a)}}\lesssim \Olog(\sqrt{SAT}).
\end{align*}
The first relation holds by Cauchy-Schartz inequality. In the second relation, we replaced the sum in the first term to cover all of the state-action pairs, thus adding positive quantities. The third relation holds since by definition $\sum_{t=1}^H\sum_{s,a} w_{tk}(s,a)=H$ and $T=KH$. The last relation holds by Lemma \ref{lemma: good set visitation count ratio}.

Combining the result in \eqref{eq: supp general results 1 over sqrt n first eq} concludes the proof.
\end{proof}

\clearpage

\begin{lemma}\label{lemma: sum of 1 over n} 
Let $u,v\ge0$ be some non-negative constants. Outside the failure event,
\begin{align*}
    \sum_{k=1}^{K}\sum_{t=1}^{H}\E\brs*{ \min\brc*{ \frac{u}{n_{k-1}(s_{t'}^k,a_{t'}^k)\vee 1},v} \mid \mathcal{F}_{k-1}} \leq \Olog(SAu + SAHv)\enspace,
\end{align*}
and specifically, 
\begin{align*}
    \sum_{k=1}^{K}\sum_{t=1}^{H}\E\brs*{\frac{u}{n_{k-1}(s_{t'}^k,a_{t'}^k)\vee 1} \mid \mathcal{F}_{k-1}} \leq \Olog(SAHu)\enspace.
\end{align*}
\end{lemma}

\begin{proof}
The proof partially follows \citep{zanette2019tighter}, Lemma 12:

\begin{align*}
    &\sum_{k=1}^{K}\sum_{t=1}^{H}\E\brs*{ \min\brc*{ \frac{u}{n_{k-1}(s_{t'}^k,a_{t'}^k)\vee 1},v} \mid \mathcal{F}_{k-1}} \\
    & \stackrel{(1)}{=} \sum_{k=1}^{K}\sum_{t=1}^{H}\sum_{s,a} w_{tk}(s,a)\min\brc*{ \frac{u}{n_{k-1}(s,a)\vee 1},v} \\
    & \stackrel{(2)}{\le} \sum_{k=1}^{K}\sum_{t=1}^{H}\sum_{(s,a)\in L_k} w_{tk}(s,a)\frac{u}{n_{k-1}(s,a)} + 
      v\sum_{k=1}^{K}\sum_{t=1}^{H}\sum_{(s,a)\notin L_k} w_{tk}(s,a)\\
    & \stackrel{(3)}{\lesssim} SAu + SAHv 
\end{align*}

$(1)$ is from the definition of $w_{tk}(s,a)$ and the fact that $n_{k-1}(s,a)$ is $\F_{k-1}$ measurable. In $(2)$ we divided the sum into state-actions in and outside $L_k$. For state-actions in $L_k$, we bounded the minimum by the first term, and otherwise we bounded it by $H^2$. Note that for any $(s,a)\in L_k$, $n_{k-1}(s,a)\ge1$, from Lemma \ref{lemma: good set visitation ratio 1}. $(3)$ is due to Lemmas \ref{lemma: visitation outside good set} and \ref{lemma: good set visitation count ratio}.

The second part of the lemma is a direct result of fixing $v=u$.

\end{proof}

\end{document}